\documentclass[runningheads]{llncs}
\usepackage[T1]{fontenc}
\pdfoutput=1

\usepackage{soul}
\usepackage{url}
\usepackage[utf8]{inputenc}
\usepackage{graphicx}
\usepackage{amsmath}
\usepackage{amssymb}
\usepackage{booktabs}
\usepackage{algorithmic}
\usepackage{bm}
\usepackage{bcprules}
\usepackage{chemarrow}
\usepackage{wrapfig}
\usepackage{amsfonts}
\usepackage{xcolor}
\usepackage{proof}
\usepackage{ifthen}
\usepackage{nicefrac}
\usepackage{thmtools}
\usepackage{thm-restate}

\makeatletter
\let\MYcaption\@makecaption
\makeatother
\usepackage{subcaption}
\usepackage[normalem]{ulem}
\usepackage{ascmac}
\usepackage{here}

\allowdisplaybreaks[1]

\newcommand\Sec[1] {Section~\ref{#1}}

\newcommand\App[1] {Appendix~\ref{#1}}
\newcommand\Apps[1] {Appendices~\ref{#1}}
\newcommand\Thm[1] {Theorem~\ref{#1}}

\newcommand\Propo[1]{Proposition~\ref{#1}}

\newcommand\Def[1] {Definition~\ref{#1}}
\newcommand\Eg[1] {Example~\ref{#1}}

\newcommand\Tbls[1] {Tables~\ref{#1}}
\newcommand\Fig[1] {Fig.~\ref{#1}}
\newcommand\Figs[1] {Figs.~\ref{#1}}

\renewcommand{\sf}[0]{\textsf}
\newcommand\rulespsmall{\vspace{0.1cm}}
\newcommand\rulesp{\vspace{0.25cm}}
\newcommand\reducespace[1]{\vspace{#1}}
\renewcommand{\phi} {\varphi}

\newcommand{\eqdef}{\mathbin{\stackrel{\smash{\tiny \text{def}}}{=}}}

\newcommand{\supp}{\mathit{supp}}
\newcommand{\range}{\mathit{range}}

\newcommand{\Dists}{\mathbb{D}} 

\newcommand{\size}[1]{\mathit{size}(#1)}

\newcommand{\dom}[1]{\mathit{dom}(#1)}

\newcommand{\mydo}{\mathit{do}}
\newcommand{\joint}{\mathbin{::}}
\newcommand{\mrgn}[1]{\downarrow_{#1}}

\newcommand{\ov}[1]{\overline{#1}}
\newcommand{\ud}[1]{\underline{#1}}
\newcommand{\la}{\langle}
\newcommand{\ra}{\rangle}

\newcommand\indep{\protect\mathpalette{\protect\independenT}{\perp}}
\def\independenT#1#2{\mathrel{\rlap{$#1#2$}\mkern2.5mu{#1#2}}}

\newcommand{\cali}{\mathcal{I}}

\newcommand{\calo}{\mathcal{O}}

\newcommand{\calr}{\mathcal{R}}

\newcommand{\calv}{\mathcal{V}}
\newcommand{\calw}{\mathcal{W}}

\newcommand{\nA}{c_0}
\newcommand{\nB}{c_1}

\newcommand{\bmc}{\bm{c}}
\newcommand{\bmcA}{\bm{c_0}}
\newcommand{\bmcB}{\bm{c_1}}
\newcommand{\bmcC}{\bm{c_2}}
\newcommand{\bmd}{\bm{d}}

\newcommand{\bmu}{\bm{u}}

\newcommand{\bmuB}{\bm{u_1}}
\newcommand{\bmuC}{\bm{u_2}}

\newcommand{\bms}{\bm{s}}

\newcommand{\bmv}{\bm{v}}

\newcommand{\bmx}{\bm{x}}

\newcommand{\bmxB}{\bm{x_1}}
\newcommand{\bmxC}{\bm{x_2}}
\newcommand{\bmxD}{\bm{x_3}}
\newcommand{\bmy}{\bm{y}}
\newcommand{\bmyA}{\bm{y_0}}
\newcommand{\bmyB}{\bm{y_1}}

\newcommand{\bmz}{\bm{z}}

\newcommand{\bmo}{\bm{o}}
\newcommand{\bmoA}{\bm{o_1}}
\newcommand{\bmoB}{\bm{o_2}}

\newcommand{\bmbot}{\bm{\bot}}

\newcommand{\M}{\mathfrak{M}}

\newcommand{\relX}[1]{\calr_{#1}}
\newcommand{\relE}[2]{\relX{\intvE{\subst{#2}{#1\!}}}}

\newcommand{\relExc}[0]{\relE{\bmx}{\bmc}}

\newcommand{\memory}{m}
\newcommand{\mem}[1]{m_{#1}}

\newcommand{\dgen}{\mathit{g}}
\newcommand{\datagenerator}{\dgen}
\newcommand{\datagen}[1]{\dgen_{#1}}

\newcommand{\precg}[0]{\mathbin{\prec_{\dgen}}}
\newcommand{\precgw}[0]{\mathbin{\prec_{\datagen{w}}}}

\newcommand{\diag}[1]{\mathit{G}_{#1}}
\newcommand{\semf}[0]{\xi}

\newcommand{\sem}[1]{{[\![ #1 ]\!]}}
\newcommand{\sema}[2]{{[\![ #1 ]\!]_{#2}}}

\newcommand{\rand}[0]{r}
\newcommand{\semr}[2]{{[\![ #1 ]\!]_{#2}^{\rand}}}
\newcommand{\semrxg}[1]{{[\![ #1 ]\!]_{\semf,\dgen}^{\rand}}}
\newcommand{\semrX}[3]{{[\![ #1 ]\!]_{#2}^{#3}}}

\newcommand{\modelsx}[1]{\mathbin{\models_{#1}}}
\newcommand{\modelsg}{\modelsx{\datagenerator}}
\newcommand{\fv}[1]{\sf{fv}(#1)}

\newcommand{\fnc}[1]{\sf{fc}(#1)}
\newcommand{\cdv}[1]{\sf{cdv}(#1)}

\newcommand{\cntv}[1]{\mathit{cnt_{\sf v}}(#1)}

\newcommand{\var}{\sf{Var}}

\newcommand{\cvar}{\sf{CVar}}

\newcommand{\fvar}{\sf{FVar}}

\newcommand{\Const}{\sf{Const}}

\newcommand{\dConst}{\sf{dConst}}
\newcommand{\Func}{\sf{Fsym}}
\newcommand{\pFunc}{\sf{pFsym}}
\newcommand{\dFunc}{\sf{dFsym}}
\newcommand{\Term}{\sf{Term}}
\newcommand{\CTerm}{\sf{CTerm}}

\newcommand{\StaCL}[0]{StaCL}

\newcommand{\Pred}[0]{\sf{Psym}}

\newcommand{\CPred}[0]{\sf{CPsym}}

\newcommand{\Fml}[0]{\sf{Fml}}

\newcommand{\subst}[2]{\nicefrac{#1}{#2}}

\newcommand{\pos}{\mathit{pos}}

\newcommand{\dsep}{\mathit{dsep}}

\newcommand{\pa}{\mathit{pa}}
\newcommand{\npa}{\mathit{npa}}
\newcommand{\anc}{\mathit{anc}}
\newcommand{\nanc}{\mathit{nanc}}
\newcommand{\allnanc}{\mathit{allnanc}}
\newcommand{\Pa}{\mathtt{PA}}
\newcommand{\Anc}{\mathtt{ANC}}
\newcommand{\Anca}{\mathtt{ANC}_{\mathtt{*}}}

\newcommand{\Dec}{\mathtt{DEC}}

\newcommand{\intvE}[1]{\lceil #1 \rceil}
\newcommand{\intvL}[1]{\lfloor #1 \rfloor}

\makeatletter
\providecommand*{\xmapstofill@}{%
  \arrowfill@{\mapstochar\relbar}\relbar\rightarrow
}
\providecommand*{\xmapsto}[2][]{%
  \ext@arrow 0395\xmapstofill@{#1}{#2}%
}

\newcommand{\vdashx}[1]{\mathbin{\vdash_{#1\,}}}
\newcommand{\vdashg}[0]{\vdashx{\datagenerator}}
\newcommand{\vdashgp}[0]{\vdashx{\datagenerator_0}}

\newcommand{\AX}[0]{{\bf A\!X}}
\newcommand{\AXCP}[0]{{\bf A\!X^{CP}}}
\newcommand{\AXwCP}[0]{$\AXCP{}$}

\newcommand{\axPT}{{\sc{PT}}}
\newcommand{\axMP}{{\sc{MP}}}

\newcommand{\axDGEI}{{\sc{DG}$_{\textsc{EI}}$}}
\newcommand{\axEqA}{{\sc{Eq1}}}
\newcommand{\axEqB}{{\sc{Eq2}}}
\newcommand{\axEqName}{{\sc{Eq}$_{\textsc{C}}$}}
\newcommand{\axEqFunc}{{\sc{Eq}$_{\textsc{F}}$}}
\newcommand{\axEffect}{{\sc{Effect}$_{\textsc{EI}}$}}
\newcommand{\axEqEI}{{\sc{Eq}$_{\textsc{EI}}$}}
\newcommand{\axCmpEI}{{\sc{Cmp}$_{\textsc{EI}}$}}

\newcommand{\axSplitE}{{\sc{Split}$_{\textsc{EI}}$}}
\newcommand{\axSimulE}{{\sc{Simul}$_{\textsc{EI}}$}}
\newcommand{\axRptE}{{\sc{Rpt}$_{\textsc{EI}}$}}

\newcommand{\axDistrE}{{\sc{Distr}$_{\textsc{EI}}$}}
\newcommand{\axPD}{{\sc{PD}}}

\newcommand{\axMPD}{{\sc{MPD}}}

\newcommand{\axDGL}{{\sc{DG}$_{\textsc{LI}}$}}

\newcommand{\axCondL}{{\sc{Cond}$_{\textsc{LI}}$}}
\newcommand{\axEqLI}{{\sc{Eq}$_{\textsc{LI}}$}}
\newcommand{\axCmpLI}{{\sc{Cmp}$_{\textsc{LI}}$}}

\newcommand{\axSplitL}{{\sc{Split}$_{\textsc{LI}}$}}
\newcommand{\axSimulL}{{\sc{Simul}$_{\textsc{LI}}$}}
\newcommand{\axRptL}{{\sc{Rpt}$_{\textsc{LI}}$}}
\newcommand{\axDistrL}{{\sc{Distr}$_{\textsc{LI}}$}}
\newcommand{\axXpdEL}{{\sc{Expd}$_{\textsc{EILI}}$}}
\newcommand{\axXcdEL}{{\sc{Excd}$_{\textsc{EILI}}$}}
\newcommand{\axDsepCIndB}{{\sc{DsepCI}}}
\newcommand{\axDsepSm}{{\sc{DsepSm}}}
\newcommand{\axDsepDc}{{\sc{DsepDc}}}
\newcommand{\axDsepWU}{{\sc{DsepWu}}}
\newcommand{\axDsepCn}{{\sc{DsepCn}}}
\newcommand{\axDsepEN}{{\sc{Dsep}$_{\textsc{EI}}$}}
\newcommand{\axDsepENA}{{\sc{Dsep}$_{\textsc{EI1}}$}}
\newcommand{\axDsepENB}{{\sc{Dsep}$_{\textsc{EI2}}$}}
\newcommand{\axDsepLN}{{\sc{Dsep}$_{\textsc{LI}}$}}
\newcommand{\axDsepLNA}{{\sc{Dsep}$_{\textsc{LI1}}$}}
\newcommand{\axDsepLNB}{{\sc{Dsep}$_{\textsc{LI2}}$}}
\newcommand{\axDsepLNC}{{\sc{Dsep}$_{\textsc{LIC}}$}}
\newcommand{\axNancA}{{\sc{Nanc0}}}
\newcommand{\axNancAB}{{\sc{Nanc1}}}
\newcommand{\axNancB}{{\sc{Nanc2}}}
\newcommand{\axNancC}{{\sc{Nanc3}}}
\newcommand{\axNancD}{{\sc{Nanc4}}}
\newcommand{\axNancAll}{{\sc{AllNanc}}}
\newcommand{\axPaToNanc}{{\sc{PaNanc}}}
\newcommand{\axPaToDsep}{{\sc{PaDsep}}}

\newcommand{\axUniq}{{\sc{Unq}}}
\newcommand{\axDoA}{{\sc{Do1}}}
\newcommand{\axDoB}{{\sc{Do2}}}
\newcommand{\axDoC}{{\sc{Do3}}}

\newcommand{\mytrue}{\mathtt{true}}
\newcommand{\myfalse}{\mathtt{false}}

\newcommand{\garrow}[0]{\mathbin{\chemarrow}}
\newcommand{\leftgarrow}[0]{\mathbin{\rotatebox[origin=c]{180}{\chemarrow}}}

\newcommand{\Leaf}[0]{\mathsf{Leaf}}

\newcommand{\LeafF}[0]{\mathsf{LF}}

\newcommand{\phiRCT}{\phi_{\textsc{RCT}}}
\newcommand{\phiBDA}{\phi_{\textsc{BDA}}}

\newcommand{\psiPre}[0]{\psi_{\rm pre}}
\newcommand{\psiDP}[0]{\psi_{\rm dp}}
\newcommand{\psiDoB}[0]{\psi_{\rm do2}}

\newcommand{\psiDA}[0]{\psi_{\rm d1}}
\newcommand{\psiDB}[0]{\psi_{\rm d2}}
\newcommand{\psiNanc}[0]{\psi_{\rm nanc}}
\newcommand{\psiPos}[0]{\psi_{\rm pos}}

\makeatletter
\providecommand{\leftsquigarrow}{%
  \mathrel{\mathpalette\reflect@squig\relax}%
}
\newcommand{\reflect@squig}[2]{%
  \reflectbox{$\m@th#1\rightsquigarrow$}%
}
\makeatother

\newif\ifcommentson\commentsonfalse

\newif\ifconferenceon\conferenceonfalse
\ifconferenceon
\newcommand{\arxiv}[1]{}
\newcommand{\conference}[1]{#1}
\newcommand{\conferenceShort}[1]{}
\else
\newcommand{\arxiv}[1]{#1}
\newcommand{\conference}[1]{}
\newcommand{\conferenceShort}[1]{}
\fi
\ifcommentson
\newcommand{\commentsize}[0]{.90\textwidth}
\newcommand{\commentYK}[1]{\begin{center} \parbox{\commentsize}{\textbf{\textcolor{black}{Comment Y.}} \textcolor{red}{#1} }\end{center}}
\newcommand{\commentTS}[1]{\begin{center} \parbox{\commentsize}{\textbf{\textcolor{black}{Comment T.}} \textcolor{red}{#1 }}\end{center}}
\newcommand{\commentKS}[1]{\begin{center} \parbox{\commentsize}{\textbf{\textcolor{black}{Comment K.}} \textcolor{red}{#1} }\end{center}}
\newcommand{\replyYK}[1]{\begin{center} \parbox{\commentsize}{\textbf{Reply Y.} \textcolor{blue}{#1} }\end{center}}
\newcommand{\replyTS}[1]{\begin{center} \parbox{\commentsize}{\textbf{Reply T.} \textcolor{blue}{#1} }\end{center}}
\newcommand{\replyKS}[1]{\begin{center} \parbox{\commentsize}{\textbf{Reply K.} \textcolor{blue}{#1} }\end{center}}
\newcommand{\commentY}[1]{\marginpar{\footnotesize \color{red} {\bf Y:} \textsf{\scriptsize #1}}}
\newcommand{\commentT}[1]{\marginpar{\footnotesize \color{red} {\bf T:} \textsf{\scriptsize #1}}}
\newcommand{\commentK}[1]{\marginpar{\footnotesize \color{red} {\bf K:} \textsf{\scriptsize #1}}}
\newcommand{\replyY}[1]{\marginpar{\footnotesize \color{red} {\bf Y:} \textsf{\scriptsize #1}}}
\newcommand{\replyT}[1]{\marginpar{\footnotesize \color{red} {\bf T:} \textsf{\scriptsize #1}}}
\newcommand{\replyK}[1]{\marginpar{\footnotesize \color{red} {\bf K:} \textsf{\scriptsize #1}}}
\else
\newcommand{\commentYK}[1]{}
\newcommand{\commentTS}[1]{}
\newcommand{\commentKS}[1]{}
\newcommand{\replyYK}[1]{}
\newcommand{\replyTS}[1]{}
\newcommand{\replyKS}[1]{}
\newcommand{\commentY}[1]{}
\newcommand{\commentT}[1]{}
\newcommand{\commentK}[1]{}
\newcommand{\replyY}[1]{}
\newcommand{\replyT}[1]{}
\newcommand{\replyK}[1]{}
\fi

\definecolor{DarkGreen}{rgb}{0,.6,0}

\newcommand{\colorR}[1]{\textcolor{red}{#1}}

\newcommand{\pagelimitmarker}[1]{~\\ {\colorR{\ifthenelse{\thepage>#1}{\Huge Exceeding the page limit}{\huge Within the page limit}}}~\\ {\huge{\colorR{~~Page Limit\,\,\,\,\, = #1}}}~\\ {\huge{\colorR{~~Current Page = $\thepage$}}}}

\newcommand{\myqed}[0]{\hfill $\Box$}

\newcommand{\PARAGRAPH}[1]{\rulespsmall \noindent \textbf{#1.} }
\newcommand{\PARAGRAPHws}[1]{\noindent \textbf{#1.} }
\makeatletter
\newcommand{\tblcaption}[1]{\def\@captype{table}\caption{#1}}
\newcommand{\figsubcaption}[1]{\def\@captype{figure}\subcaption{#1}}
\newcommand{\figcaption}[1]{\def\@captype{figure}\caption{#1}}
\makeatother

\begin{document}
\title{Formalizing Statistical Causality via Modal Logic
}
\author{Yusuke Kawamoto\inst{1,2}
\and
Tetsuya Sato\inst{3}
\and
Kohei Suenaga\inst{4}
}
\authorrunning{Y. Kawamoto et al.}
\institute{
AIST, Tokyo, Japan
\and
PRESTO, JST, Tokyo, Japan
\and
Tokyo Institute of Technology, Tokyo, Japan
 \and
Kyoto University, Kyoto, Japan
}
\maketitle              
\begin{abstract}
We propose a formal language for describing and explaining statistical causality.
Concretely, we define \emph{Statistical Causality Language} (StaCL) for expressing causal effects and specifying the requirements for causal inference.
StaCL incorporates modal operators for interventions to express causal properties between probability distributions in different possible worlds in a Kripke model.
We formalize axioms for probability distributions, interventions, and causal predicates using StaCL formulas.
These axioms are expressive enough to derive the rules of Pearl's do-calculus.
Finally, we demonstrate by examples that StaCL can be used to specify and explain the correctness of statistical causal inference.

\end{abstract}

\section{Introduction}
\label{sec:intro}

\emph{Statistical causality} has been gaining significant importance in a variety of research fields.
In particular, in life sciences,
 more and more researchers
have been using statistical techniques to discover \emph{causal relationships} from experiments and observations.
However, these statistical methods can easily be misused or misinterpreted.
In fact, 
it is reported that many research articles have serious errors in the applications and interpretations of statistical methods~\smash{\cite{Fernandes-Taylor:11:BMCRN,Makin:19:elife}}. 

A common mistake is to misinterpret 
statistical \emph{correlation} as statistical \emph{causality}.
Notably, when we analyze observational data without experimental interventions, we may overlook some requirements for causal inference and make wrong calculations,
leading to incorrect conclusions about the causality.

For this reason, the scientific community has developed guidelines on many requirements for statistical analyses \smash{\cite{von:07:STROBE,Moher:12:CONSORT}}.
However, since there is no formal language to describe the entire procedures and their requirements, 
we refer to guidelines manually and 
cannot formally guarantee the correctness of analyses.

To address these problems, we propose a logic-based approach to formalizing and explaining the correctness of statistical causal inference.
Specifically, we introduce a formal language called \emph{statistical causality language} (\StaCL{}) to formally describe and check the requirements for statistical causal inference.
We consider this work as the first step to building a framework for formally guaranteeing and explaining the reliability of 
scientific research.

\PARAGRAPH{Contributions}
Our main contributions are as follows:
\begin{itemize}
\item
We propose \emph{statistical causality language} (\StaCL{}) for formalizing and explaining statistical causality by using modal operators for interventions.
\item
We define a \emph{Kripke model for statistical causality}.
To formalize not only statistical correlation but also statistical causality, we introduce a \emph{data generator} in a possible world
to model a causal diagram in a Kripke model.
\item 
We introduce the notion of \emph{causal predicates} to express statistical causality
and interpret them using a data generator instead of a valuation in a Kripke model.
In contrast, \emph{(classical) predicates} are interpreted using a valuation in a Kripke model to express only \emph{statistical correlations}.
\item
We introduce a sound deductive system \AXwCP{} for \StaCL{} with axioms for probability distributions, interventions, and causal predicates.
These axioms are expressive enough to reason about all causal effects identifiable by Pearl's \emph{do-calculus}~\cite{Pearl:95:biometrika}.
We show that \AXwCP{} can reason about the correctness of causal inference methods (e.g., backdoor adjustment).
Unlike prior work, \AXwCP{} does not aim to conduct causal inference about a specific causal diagram; rather, it concerns the correctness of the inference methods
for any diagram.
To the best of our knowledge, 
ours appears to be the first modal logic that can specify and reason about the requirements for causal inference.
\end{itemize}

\PARAGRAPHws{Related Work}
Many studies on causal reasoning rely on causal diagrams~\cite{Pearl:09:causality}.
Whereas 
they aim to reason about a specific diagram,
our logic-based approach aims to specify and reason about the requirements for causal inference methods.

Logic-based approaches for formalizing causal reasoning have been proposed.
To name a few, 
Halpern and Pearl provide logic-based definitions of actual causes where logical formulas with events formalize counterfactuals \cite{Halpern:01:IJCAI,Halpern:01:UAI,Halpern:15:IJCAI}.
Probabilistic logical languages~\cite{Ibeling:20:AAAI} are proposed to axiomatize causal reasoning 
with observation, intervention, and counterfactual inference.
Unlike our logic, however, their framework does not aim to syntactically derive the correctness of statistical causal inference.
The causal calculus~\cite{McCain:97:AAAI} is used to provide a logical representation~\cite{Bochman:15:AAAI,Bochman:21:book} of Pearl~\cite{Pearl:09:causality}'s structural causal model.
The counterfactual-observational language~\cite{Barbero:21:jphil} can reason about interventionist counterfactuals and has an axiomatization that is complete w.r.t. a causal team semantics.
A modal logic in~\cite{Barbero:20:Dali} integrates causal and epistemic reasoning.
While these works deal with deterministic cases only, our \StaCL{} can reason about statistical causality in probabilistic settings.

There have been studies on incorporating probabilities into team semantics~\cite{Hodges:97:IGPL}.
For example, team semantics is used to deal with the dependence and independence among random variables~\cite{Durand:16:Foiks,Corander:19:APAL}.
A probabilistic team semantics is provided for a first-order logic that can deal with conditional independence~\cite{Durand:18:foiks}.
A team semantics is also introduced for logic with exact/approximate dependence and independence atoms~\cite{Hirvonen:19:TAMC}.
Unlike our \StaCL{}, however, these works do not allow for deriving the do-calculus or the correctness of causal inference methods.

Concerning the axiomatic characterization of causality,
Galles and Pearl~\cite{Galles:98:FoS} prove that the axioms of composition, effectiveness, and reversibility are sound and complete with respect to the structural causal models.
They also show that the reversibility axiom can be derived from the composition axiom if the causal diagram is acyclic (i.e., has no feedback loop).
Halpern~\cite{Halpern:00:JAIR} provides axiomatizations for more general classes of causal models with feedback and with equations that may have no solutions.
In contrast, our deductive system \AXwCP{} has axioms for causal predicates and two forms of interventions that can derive the rules of Pearl's do-calculus~\cite{Pearl:95:biometrika},
while being equipped with axioms corresponding to the composition and effectiveness axioms mentioned above only for acyclic diagrams.%

For the efficient computation of causal reasoning, 
constraint solving 
is applied
~\cite{Hyttinen:14:UAI,Hyttinen:15:UAI,Triantafillou:15:JMLR}.
Probabilistic logic programming is used to encode and reason about a specific causal diagram~\cite{Ruckschloss:22a:ICLP}.
These are 
orthogonal to the goal of our work.

Finally, a few studies propose modal logic for statistical methods.
Statistical epistemic logic~\cite{Kawamoto:19:FC,Kawamoto:19:SEFM,Kawamoto:20:SoSyM} 
specifies various properties of machine learning.
Belief Hoare logic~\cite{Kawamoto:21:KR,Kawamoto:22:arxiv} can reason about 
statistical hypothesis testing programs.
However, unlike our \StaCL{}, these cannot reason about statistical causality.

\section{Illustrating Example}
\label{sec:overview}
We first present a simple example to explain our framework.

\begin{example}[Drug's efficacy]\rm \label{eg:illustrate}
We attempt to check a drug's efficacy for a disease
by observing a situation where some patients take a drug and the others do not.

Table~\ref{tab:example:simpson} shows the recovery rates and the numbers of patients treated with\slash without the drug.
For both males and females, \emph{more} patients recover by taking the drug.
However, for the combined population, the recovery rate with the drug (0.73) is \emph{less} than that without it (0.80).
This inconsistency is called \emph{Simpson's paradox}~\cite{Simposon:51:RSS}, showing the difficulty of identifying causality from observed data.

To model this, we define three variables: a \emph{treatment} $x$ ($1$ for drug, $0$ for no-drug), an \emph{outcome} $y$ ($1$ for recovery, $0$ for non-recovery), and a gender $z$.
\Fig{fig:example:confound} depicts their causal dependency; the arrow $x \garrow y$ denotes that $y$ depends on $x$. 
The \emph{causal effect} $p(y | \mydo(x\,{=}\,c))$ of a treatment $x\,{=}\,c$ on an outcome $y$~\cite{Pearl:09:causality} is defined as the distribution of $y$ in case $y$ were generated from $x=c$
(\Fig{fig:example:RCT}).

However, since the gender $z$ influences 
the choice of the treatment $x$ in reality (\Fig{fig:example:confound}),
the causal effect $p(y | \mydo(x\,{=}\,c))$
depends on the common cause $z$ of $x$ and $y$
and differs from the correlation $p(y |x\,{=}\,c)$.
Indeed, in Table~\ref{tab:example:simpson},  80 \% of females chose to take the drug ($x = 1$) while only 20 \% of males did so;
this dependency of $x$ on the gender $z$ leads to Simpson's paradox in  Table~\ref{tab:example:simpson}.
Thus, calculating the causal effect requires an ``adjustment'' for $z$, as explained below.
\end{example}

\begin{figure}[t]
\begin{tabular}{cc}
\begin{minipage}[h]{0.43\hsize}
  \def\@captype{table}
  \centering
  \begin{footnotesize}
  \tblcaption{Recovery rates of patients with/without taking a drug.
  \label{tab:example:simpson}}
  \begin{tabular}{lcc}
  \hline
  \!&\!\hspace{-0ex} Drug  &\!\hspace{-0ex} No-drug \\[-0.3ex]
  \!&\!\hspace{-0ex} $x=1$&\!\hspace{-0ex} $x=0$ \\
  \hline
  \hspace{-0ex}Male	&\!\hspace{-0ex}{\bf 0.90}	&\!\hspace{-0ex}0.85 \\[-0.5ex]
  					&\!\hspace{-0ex}(18/20)	&\!\hspace{-0ex}(68/80) \\[0.2ex]
  \hline
  \hspace{-0ex}Female	&\!\hspace{-0ex}{\bf 0.69}	&\!\hspace{-1ex}0.60 \\[-0.5ex]
  					&\!\hspace{-0ex}(55/80)	&\!\hspace{-0ex}(12/20) \\[0.2ex]
  \hline
  \hspace{-0ex}Total &\!\hspace{-0ex}0.73	&\!\hspace{-0ex}{\bf 0.80} \\[-0.5ex]
  					&\!\hspace{-0ex}(73/100)	&\!\hspace{-0ex}(80/100) \\
  \hline
  \end{tabular}
  \end{footnotesize}
\end{minipage}
~&
\begin{minipage}[h]{0.57\hsize}
  \def\@captype{figure}
\begin{minipage}[t]{1.00\hsize} 
  \def\@captype{subfigure}
  \setcounter{figure}{1}
  \setlength\unitlength{1pt}
  \centering
  \begin{picture}(50,26)(0,0)
    \put(10,0){\vector(1,0){30}}
    \put(21,17){\vector(-1,-1){13}}
    \put(29,17){\vector(1,-1){13}}
    \put(2, -3){$x$}
    \put(43, -3){$y$}
    \put(23, 18){$z$}
  \end{picture}
  \figsubcaption{The actual diagram $G$ with a gender (confounder) $z$, a treatment $x$, and an outcome $y$.\!
  \label{fig:example:confound}}
\end{minipage}
\\[0.0ex]
\begin{minipage}[t]{1.00\hsize} 
  \def\@captype{subfigure}
  \setlength\unitlength{1pt}
  \centering
  \begin{picture}(50,26)(0,0)
    \put(10,0){\vector(1,0){30}}
    \put(5,15){\vector(0, -1){11}}
    \put(29,17){\vector(1,-1){13}}
    \put(2, -3){$x$}
    \put(43, -3){$y$}
    \put(23, 18){$z$}
    \put(3, 18){$c$}
  \end{picture}
  \figsubcaption{The diagram $G \intvE{\subst{c}{x}}$ with an intervention to~$x$.\!
  \label{fig:example:RCT}\!}
\end{minipage}
\setcounter{figure}{0}
\figcaption{Causal diagrams in \Eg{eg:illustrate}.
\label{fig:eg:diagram}}
\end{minipage}
\end{tabular}
\end{figure}

\PARAGRAPH{Overview of the Framework}
We describe reasoning about the causal effect in \Eg{eg:illustrate}
using logical formulas in our formal language \StaCL{} (\Sec{sec:assertion-logic}).

We define $\phiRCT \eqdef \intvE{\subst{c}{x}} (\nA = y)$ to express 
a \emph{randomized controlled trial (RCT)},
where we 
randomly divide the patients into two groups: one taking the drug ($x=1$) and the other not ($x=0$).
This random choice of the treatment $x$ is expressed by the intervention $\intvE{\subst{c}{x}}$ 
for $c=0,1$
in the diagram $G\intvE{\subst{c}{x}}$ 
(\Fig{fig:example:RCT}).
Since $x$ is independent of $z$ in $G\intvE{\subst{c}{x}}$,
the causal effect $p(y | \mydo(x\,{=}\,c))$ of $x$ on the outcome $y$ 
is given as $y$'s distribution $\nA$ observed in the experiment in $G\intvE{\subst{c}{x}}$.

In contrast, $\phiBDA \eqdef (f = y|_{z,x=c} \,\land~\nB = z \,\land~\nA = f(\nB)\!\mrgn{y})$ describes the inference about the causal effect from observation \emph{without} intervention to $x$ 
(\Fig{fig:example:confound}).
This saves the cost of the experiment and avoids ethical issues in random treatments.
Instead, to avoid Simpson's paradox, 
the inference $\phiBDA$ conducts 
a \emph{backdoor adjustment} 
(\Sec{sec:reasoning-StaCL})
to cope with the confounder $z$.

Concretely, the backdoor adjustment $\phiBDA$ computes $x$'s causal effect on $y$ as follows.
We first obtain the conditional distribution $f \eqdef y|_{z,x=c}$ and the prior $\nB \eqdef z$.
Then we conduct the adjustment by calculating the joint distribution $f(\nB)$ from $f$ and $\nB$ and then taking the marginal distribution $\nA \eqdef f(\nB)\!\mrgn{y}$.
The resulting $\nA$ is the same as the $\nA$ in the RCT experiment $\phiRCT$;
that is, the backdoor adjustment $\phiBDA$ can compute the causal effect obtained by $\phiRCT$.

For this adjustment, we need to check the requirement 
$\pa(z, x) \land \pos(x \joint z)$,
that is,
$z$ is $x$'s parent 
in the diagram $\diag{}$ and
the joint distribution $x \joint z$ satisfies the positivity (i.e., it takes each value with a non-zero probability).

Now we formalize the \emph{correctness} of this causal inference method (for any diagram $\diag{}$)
as the judgment expressing that under the above requirements, the backdoor adjustment computes the same causal effect as the RCT experiment:
\begin{align}
\label{eq:correctness} 
\pa(z, x) 
\land \pos(x \joint z) \vdashg \phiRCT \leftrightarrow \phiBDA
{.}
\end{align}
By deriving this judgment in a deductive system called \AXwCP{} (\Sec{sec:axioms}), we show the correctness of this causal inference method for any diagram (\Sec{sec:reasoning-StaCL}).
\conference{We show all proofs of the technical results in this paper's full version~\cite{Kawamoto:22:JELIA:arxiv}.}
\arxiv{We show all proofs of the technical results in Appendix.}

\section{Language for Data Generation}
\label{sec:data}
In this section, we introduce a language for describing data generation.

\PARAGRAPH{Constants and Causal Variables}
We introduce a set $\Const$ of \emph{constants} to denote probability distributions of data values and a set $\dConst \subseteq \Const$ of \emph{deterministic constants}, each denoting a single data value (strictly speaking, denoting a distribution having a single data value with probability $1$).

We introduce a finite set $\cvar$ of \emph{causal variables}.
A tuple $\la x_1, \ldots, x_k \ra$ of causal variables represents the joint distribution
of $k$ variables $x_1, \ldots, x_k$.
We denote the set of all non-empty (resp. possibly empty) tuples of variables by $\cvar^+$ (resp. $\cvar^*$).
We use the bold font for a \emph{tuple}; e.g., $\bm{x} = \la x_1, \ldots, x_k \ra$.
We write $\size{\bm{x}}$ for the \emph{dimension} $k$ of a tuple $\bm{x}$.
We assume that the variables in a tuple $\bm{x}$ are sorted lexicographically.

For disjoint tuples $\bmx$ and $\bmy$,
$\bmx \joint \bmy$ denotes the \emph{joint distribution} of $\bmx$ and $\bmy$.
Formally, `$\joint$' is \emph{not} a function symbol, but a meta-operator on $\cvar^*$;
$\bmx \joint \bmy$ is the tuple obtained by merging $\bmx$ and $\bmy$ and sorting the variables lexicographically.

We use \emph{conditional causal variables} 
$\bmy|_{\bmz,\bmx=\bmc}$ to denote the conditional distribution of $\bmy$ given $\bmz$ and $\bmx=\bmc$.
We write $\fvar$ for the set of all conditional causal variables.
For a conditional distribution $\bmy|_{\bmx}$ and a prior distribution $\bmx$,
we write $\bmy|_{\bmx}(\bmx)$ for the joint distribution $\bmx\joint\bmy$.

\PARAGRAPH{Terms}
We define \emph{terms} to express how data are generated.
Let 
$\Func$ be a set of \emph{function symbols} denoting algorithms. 
We define the set $\CTerm$ of \emph{causal terms} as the terms of depth at most $1$; i.e.,
$u \mathbin{::=} c \mid f(v, \ldots, v)$
where $c\in\Const$, $f\in\Func$, and $v\in\cvar\cup\Const$.
For example, $f(c)$ denotes a data generated by an algorithm $f$ with input $c$.
We denote the set of variables (resp. the set of constants) occurring in a term $u$ by $\fv{u}$ (resp. $\fnc{u}$).

We also define the set $\Term$ 
of \emph{terms} by the BNF:
$u \mathbin{::=} \bmx \mid c \mid f(u, \ldots, u),$
where 
$\bmx\in\cvar^+$, 
$c\in\Const$,
and $f\in \Func \cup \fvar$.
Unlike $\CTerm$, terms in $\Term$ may repeatedly apply functions 
to describe multiple steps of data generation.

We introduce the special function symbol $\mrgn{\bmx}$ for marginalization.
$\bmy\!\mrgn{\bmx}$ denotes the \emph{marginal distribution} of $\bmx$ given a joint distribution $\bmy$;
e.g., for a joint distribution $\bmx = \la x_0, x_1 \ra$,\, $\bmx\!\mrgn{x_0}$ expresses the marginal distribution $x_0$.
We also introduce the special constant $\bmbot$ for \emph{undefined values}.

\PARAGRAPH{Data Generators}
To describe how data are generated, we introduce the notion of a \emph{data generator} as a function $\dgen: \cvar\rightarrow\CTerm \cup \{\bmbot\}$ that maps a causal variable $x$ to a causal term representing how the data assigned to $x$ is generated.
If $\dgen(y) = u$ for $u \in \CTerm$ and $y \in \cvar$, we write $u \garrow_{\dgen} y$.
For instance, the data generator $\dgen$ in \Fig{fig:data-generator} models the situation in  \Eg{eg:illustrate}.
To express that a variable $x$'s value is generated by an algorithm $f_1$ with an input $z$, the data generator $\dgen$ maps $x$ to $f_1(z)$, i.e., $f_1(z) \garrow_{\dgen} x$.
Since the causal term $f_1(z)$'s depth is at most $1$, $z$ represents the \emph{direct cause} of $x$.
We denote the set of all variables $x$ satisfying $\dgen(x) \neq \bmbot$ by $\dom{\dgen}$, 
and the range of $\dgen$ by $\range(\dgen)$.

\begin{wrapfigure}[9]{l}[0pt]{0.42\textwidth}
\centering
\vspace*{-1.7em}
\begin{small}
\begin{tabular}{@{}l|c}
\hline
\begin{tabular}[c]{l}
Data generator $\dgen$
\end{tabular}
&
\begin{tabular}[c]{l}
Causal diagram
\\[-0.4ex]
$\diag{}$ given from $\dgen$
\end{tabular}
\\\hline
\begin{tabular}{c}
$
\begin{aligned}
\dom{\dgen}&\,{=}\, \{x,y,z\} \\[-0.4ex]
f_1(z) &\garrow_{\dgen} x \\[-0.4ex]
f_2(z,x)&\garrow_{\dgen} y
\end{aligned}
$
\end{tabular}
&
\begin{tabular}[c]{c}
  \setlength\unitlength{1pt}
  \begin{picture}(50,31)(0,0)
    \put(10, 3){\vector(1,0){30}}
    \put(21, 21){\vector(-1,-1){14}}
    \put(29, 21){\vector(1,-1){14}}
    \put(2, 0){$x$}
    \put(43, 0){$y$}
    \put(23, 23){$z$}
  \end{picture}
  \end{tabular}
\\ \hline
\end{tabular}
\end{small}
\caption{The data generator and causal diagram for Example \ref{eg:illustrate}.}
\label{fig:data-generator}
\end{wrapfigure}

We assume the following \emph{at-most-once} condition: 
Each function symbol and constant can be used at most once in a single data generator.
This ensures that different sampling uses different randomness and is denoted by different symbols.

We say that a data generator $\dgen$ is \emph{finite} if $\dom{\dgen}$ is a finite set.
We say that a data generator $\dgen$ is \emph{closed} if no undefined variable occurs in the terms that $\dgen$ assigns to variables, namely, $\fv{\range(\dgen)} \subseteq \dom{\dgen}$.

We write $x \precg y$ iff 
$y$'s value depends on $x$'s, 
i.e., there are variables $z_1, \ldots, z_i$ ($i \ge 2$) such that 
$z_1 = x$, $z_i = y$, and
$z_{j} \in \fv{\dgen(z_{j+1})}$ for $1 \leq j \leq i-1$.
A data generator $\dgen$ is \emph{acyclic} if $\precg$ is a strict partial order over $\dom{\dgen}$.
Then we can avoid the cyclic definitions of $\dgen$.
E.g., the data generator $\dgen_1$ defined by $f(z) \garrow_{\dgen_1} x$ and $f(c) \garrow_{\dgen_1} z$ is acyclic, whereas $\dgen_2$ by $f(z) \garrow_{\dgen_2} x$ and $f(x) \garrow_{\dgen_2} z$ is cyclic.

\section{Kripke Model for Statistical Causality}
\label{sec:model}
In this section, 
we introduce a Kripke model for statistical causality. 

We write $\calo$ for the set of all data values we deal with, such as the Boolean values, integers, real numbers, and lists of data values.
We write $\bot$ for the undefined value.
For a set $S$, we denote the set of all probability distributions over~$S$ by $\Dists S$.
For a probability distribution $\mem{} \in \Dists S$, we write $\supp(\mem{})$ for the set of $\mem{}$'s non-zero probability elements.

\PARAGRAPH{Causal Diagrams}
To model causal relations corresponding to a given data generator $\datagen{}$,
we consider a \emph{causal diagram} $\diag{} = (U, V, E)$ \cite{Pearl:09:causality} 
where $U\cup V$ is the set of all nodes and $E$ is the set of all edges such that:
\begin{itemize}
\item 
$U \,{\eqdef} \fnc{\range(\datagen{})} \subseteq \Const$ is a set of symbols called \emph{exogenous variables} that denote distributions of data;
\item 
$V \,{\eqdef} \dom{\datagen{}} \subseteq \cvar$ is a set of symbols called \emph{endogenous variables} that may depend on other variables;
\item 
$E \eqdef \{ x \rightarrow y \,{\in}\, V \,{\times}\, V \,|\, x \,{\in}\, \fv{\datagen{}(y)} \} \cup
\{ c \rightarrow y \allowbreak \,{\in}\, \allowbreak U \,{\times}\, \allowbreak V \,|\, c \,{\in}\, \fnc{\datagen{}(y)} \}$ 
is the set of all \emph{structural equations}, i.e., directed edges (arrows) denoting the direct causal relations between variables defined by the data generator $\datagen{}$.
\end{itemize}
For instance, in \Fig{fig:data-generator}, \Eg{eg:illustrate} is modeled as the causal diagram $\diag{}$.

Since a causal term's depth is at most $1$, $\datagen{}$ specifies all information for defining $\diag{}$.
By $\datagen{}$'s acyclicity, $\diag{}$ is a directed acyclic graph (DAG)
\conference{(See 
Proposition 4
in the full version~\cite{Kawamoto:22:JELIA:arxiv} for details).
}%
\arxiv{(See \Propo{prop:acyclic:DG}
in \App{sub:app:relation:DG:CD} for details).
}%

\PARAGRAPH{Pre-/Post-Intervention Distributions}
For a causal diagram $\diag{} = (U, V, E)$ and a tuple $\bmy \subseteq V$,
we write $P_{\diag{}}(\bmy)$ for the joint distribution of $\bmy$ over $ \calo^{\size{\bm{y}}}$ generated according to  $\diag{}$.
As shown in the standard textbooks (e.g.,~\cite{Pearl:09:causality}),
$P_{\diag{}}(V)$ is factorized into conditional distributions according to $\diag{}$ as follows:
\begin{align}
\label{eq:Pg1}
P_{\diag{}}(V) \eqdef {\textstyle \prod_{y_i \in V}} P_{\diag{}}(y_i \mid \pa_{\diag{}}(y_i)),
\end{align}
where $\pa_{\diag{}}(y_i)$ is the set of parent variables of $y_i$ in $\diag{}$.
For example, in \Fig{fig:data-generator}, 
for $V = \{ x, y, z \}$,
$P_{\diag{}}( V ) = P_{\diag{}}(y \,|\, x, z)\, P_{\diag{}}(x \,|\, z)\, P_{\diag{}}(z)$.

For tuples $\bmx \subseteq V$ and $\bmo \subseteq \calo$ with $\size{\bm{x}} = \size{\bm{o}}$, the \emph{post-intervention distribution} $P_{\diag{}}(V \,|\, do(\bmx{=}\bmo))$ is the joint distribution of $V$
after $\bmx$ is assigned $\bmo$ and all the variables dependent on $\bmx$ in $\diag{}$ are updated by $\bmx := \bmo$ as follows:
\begin{align*}
P_{\diag{}}(V \,|\, do(\bmx\,{=}\,\bmo)) \eqdef 
\begin{cases}
\prod_{y_i \in V\setminus\bmx} P_{\diag{}}(y_i \,|\, \pa_{\diag{}}(y_i)) \\[-0.2ex]
\hspace{11ex}
\mbox{ for values of $V$ consistent}
\mbox{ with $\bmx = \bmo$ } \\[-0.3ex]
0 
\hspace{10ex} \mbox{ otherwise.}
\end{cases}
\end{align*}
For instance, in \Fig{fig:data-generator}, 
$P_{\diag{}}( y, z | do(x=o) ) = P_{\diag{}}(y | x=o, \allowbreak z)\, P_{\diag{}}(z)$
for any $o \,{\in}\, \calo$.

\PARAGRAPH{Possible Worlds}
We introduce the notion of a \emph{possible world} to define the probability distribution of causal variables from a data generator.
Formally, a possible world is a tuple $(\dgen, \semf, \memory)$ of 
(i) a finite and acyclic data generator $\dgen: \cvar \rightarrow \CTerm\cup\{\bmbot\}$, 
(ii) an interpretation $\semf$ that maps a function symbol in $\Func$ with arity $k\ge 0$ 
to a function from $\calo^{k}$ to $\Dists \calo$, and
(iii) a memory 
$\memory$ that maps a tuple of variables to a joint distribution of data values,
which is determined by $\dgen$ and $\semf$.
We denote these components of a world $w$ by $\datagen{w}$, $\semf_{w}$, and $\mem{w}$,
and the set of all defined variables in $w$ by 
$\var(w) = \dom{\mem{w}}$.

The interpretation $\semf$ can be constructed using 
a probability distribution $I$ over an index set $\cali$ and 
a family $\{ \semf^\rand \}_{\rand \in \cali}$ of interpretations 
each mapping a function symbol $f$ with arity $k\ge 0$ to a deterministic function $\semf^\rand(f)$ from $\calo^{k}$ to $\calo$.
Then $\semf(f)$ 
maps data values $\bm{o}$ to the probability distribution over $\calo$
obtained by randomly drawing an index $\rand$ from $I$ and then computing $\semf^\rand(f)(\bm{o})$.

If $k = 0$, $f$ is a constant and
$\semf^\rand(f) \in \calo$,
hence
$\semf(f) \in \Dists\calo$ is a distribution of data values. 
For the undefined constant, we assume $\semf^{\rand}(\bmbot) = \bot$.

\PARAGRAPH{Interpretation of Terms}
Terms are interpreted in a possible world $w = (\semf, \datagen{}, \mem{})$ as follows.
First, for each index $r \in \cali$, we define the \emph{interpretation} $\semrxg{\_}$ 
that maps a tuple of $k$ terms
to $k$ data values in $\calo$ or~$\bot$ by:
{\footnotesize
\begin{align*}
\semrxg{\bmx} &= 
\semrxg{\datagen{}(\bmx)} 
&&&
\semrxg{\la u_1, \ldots, u_k \ra} &= 
(\semrxg{u_1},\, \ldots,\, \semrxg{u_k})
\\
\semrxg{c} &= 
\semf^{\rand}(c)
&&&
\semrxg{f(u_1, \ldots, u_k)} &
= \semf^{\rand}(f)
(\semrxg{\la u_1, \ldots, u_k \ra})
.
\end{align*}
}%
For instance, in \Fig{fig:data-generator}, 
we have
$\semrxg{x} = \semrxg{\datagen{}(x)} = \semrxg{f_1(z)} = \semf^{\rand}(f_1)(\semrxg{z})$,
where the interpretation of $z$ does not depend on that of $x$ due to $\datagen{}$'s acyclicity.
We define the probability distribution $\sema{u}{w}$ over $\calo$ by
randomly drawing $\rand$ and then computing $\semrxg{u}$.
Similarly, we define $\sema{\la u_1, \ldots, u_k \ra}{w}$ via 
$\semrxg{\la u_1, \ldots, u_k \ra}$.

We remark that the interpretation $\sema{\_}{w}$ defines the joint distribution $P_{\diag{w}}$ of all variables in the causal diagram $\diag{w}$;
e.g.,  
$\sema{\bmy|_{\bmz}}{w} = P_{\diag{w}}(\bmy \,|\, \bmz)$
\conference{(See 
Proposition 5
in the full version~\cite{Kawamoto:22:JELIA:arxiv} for details).
}%
\arxiv{(See \Propo{prop:G:equiv:DG}
in \App{sub:app:relation:DG:CD} for details).
}%
A function symbol $f$ is interpreted as the function $\semf(f)$ that maps data values in $\calo$ to the distribution over $\calo$.
We define the memory $\mem{}$ by 
$\mem{}(\bmx) = \sema{\bmx}{w}$ for all $\bmx\in\cvar^{+}$.
Notice that $\sema{\_}{w}$ is defined using $\datagen{}$ and~$\semf$ without using $\mem{}$.

We expand the interpretation $\sema{\_}{w}$ to a conditional causal variable $\bmy|_{\bmz,\bmx=\bmc} \in \fvar$ 
to interpret it as 
a function that maps a value $\bmc'$ of $\bmz$ to the distribution $\sema{(\bmx\joint\bmy\joint\bmz)|_{\bmz=\bmc', \bmx=\bmc}}{w}$.
We then have 
$\sema{\bmy|_{\bmz,\bmx=\bmc}(\bmz|_{\bmx=\bmc})}{w} \,{=}\, \allowbreak
\sema{\bmy|_{\bmz,\bmx=\bmc}}{w}(\sema{\bmz|_{\bmx=\bmc}}{w})$.

For the sake of reasoning in \Sec{sec:axioms},
for each data generator $\datagenerator$, $\bmx \in \cvar^+$, and $\bmy|_{\bmz,\bmx=\bmc} \in \fvar$,
we introduce 
a constant $c^{(\datagenerator,\bmx)}$ 
and a function symbol $f^{(\datagenerator,\bmy|_{\bmz\!,\bmx=\bmc})}$.
For brevity, we often omit the superscripts of these symbols.

\PARAGRAPH{Eager/Lazy Interventions}
We introduce two forms of \emph{interventions} and their corresponding \emph{intervened worlds}.
Intuitively, in a causal diagram, an \emph{eager intervention} $\intvE{\subst{c}{x}}$ expresses the removal of all arrows pointing \emph{to} a variable $x$ by replacing $x$'s value with $c$.

In contrast, a \emph{lazy intervention} $\intvL{\subst{c}{x}}$ expresses the removal of all arrows emerging \emph{from} $x$, 
which does not change the value of $x$ itself but affects the values of the variables dependent on $x$, computed using $\sem{c}$ (instead of $\sem{x}$) as the value of $x$.

For instance, \Fig{Fig:difference_interventions} shows how two interventions $\intvE{\subst{c}{x}}$ and $\intvL{\subst{c}{x}}$ change the data generator and the causal diagram in a world $w$ that models \Eg{eg:illustrate}.

\begin{wrapfigure}[12]{l}[-1pt]{0.47\textwidth}
\centering
\vspace*{-2em}
{\small
\begin{tabular}{@{}ccc}
\hline
\!World
&\begin{tabular}[c]{c}Data\,generator \end{tabular}
&\hspace{-0.0ex}\begin{tabular}[c]{@{}c} Causal\,diagram \end{tabular}  \\
\hline
$w$
&
$
\begin{aligned}
f_1(z) &\garrow x;\\
f_2(z,x)&\garrow y
\end{aligned}
$
& 
  \begin{tabular}[c]{c}
  \setlength\unitlength{1pt}
  \begin{picture}(50,28)(0,0)
    \put(10,2){\vector(1,0){30}}
    \put(21,19){\vector(-1,-1){12}}
    \put(29,19){\vector(1,-1){12}}
    \put(2, 0){$x$}
    \put(43, 0){$y$}
    \put(23, 20){$z$}
  \end{picture}
  \end{tabular}
\\
\hline
$w\intvE{\subst{c}{x}}$
&
$
\begin{aligned}
c &\garrow x;\\
f_2(z,x)&\garrow y
\end{aligned}
$
& 
  \begin{tabular}[c]{c}
  \setlength\unitlength{1pt}
   \begin{picture}(50,28)(0,0)
    \put(10,2){\vector(1,0){30}}
    \put(5,18){\vector(0, -1){11}}
    \put(29,19){\vector(1,-1){12}}
    \put(2, 0){$x$}
    \put(43, 0){$y$}
    \put(23, 20){$z$}
    \put(3, 20){$c$}
  \end{picture}
  \end{tabular}
\\
\hline
$w\intvL{\subst{c}{x}}$
&
$
\begin{aligned}
f_1(z) &\garrow x;\\
f_2(z,c)&\garrow y
\end{aligned}
$
&   \begin{tabular}[c]{c}
  \setlength\unitlength{1pt}
   \begin{picture}(50,28)(0,0)
    \put(23,2){\vector(1,0){17}}
    \put(21,19){\vector(-1,-1){12}}
    \put(29,19){\vector(1,-1){12}}
    \put(2, 0){$x$}
    \put(43, 0){$y$}
    \put(23, 20){$z$}
    \put(16, 0){$c$}
  \end{picture}
  \end{tabular}
\\ \hline
\end{tabular}
}
\caption{Eager/lazy interventions. 
}
\label{Fig:difference_interventions}
\end{wrapfigure}

For a world $w$ and 
a $c \in \dConst$,
we define an \emph{eagerly intervened world} $w\intvE{\subst{c}{x}}$ as the world where $\sema{c}{w}$ is assigned to $x$ and is used to compute the other variables dependent on $x$.
Formally, 
$w\intvE{\subst{c}{x}}$ 
is defined by
$\semf_{w\intvE{\subst{c}{x}}} = \semf_{w}$,
$\datagen{w\intvE{\subst{c}{x}}}(y) = c$ if $y = x$, and
$\datagen{w\intvE{\subst{c}{x}}}(y) = \datagen{w}(y)$ if $y \neq x$.
For instance, in \Fig{Fig:difference_interventions}, in the world $w\intvE{\subst{c}{x}}$, we use the value of $c$ to compute 
$\sema{x}{w\intvE{\subst{c}{x}}} = \semf_{w}(c)$ and 
$\sema{y}{w\intvE{\subst{c}{x}}} = \sema{f_2(z,x)}{w\intvE{\subst{c}{x}}} = \sema{f_2(z,c)}{w}$.

Then the interpretation $\sema{\_}{w\intvE{\subst{c}{x}}}$ defines the joint distribution of all variables in the causal diagram $\diag{w}$ after the intervention $\bmx \,{:=}\, \sema{\bmc}{w}$;
e.g.,  
$\sema{\bmy|_{\bmz}}{w \intvE{\subst{\bmc}{\bmx}}} = P_{\diag{w}}(\bmy \,|\,  do(\bmx \,{=}\, \sema{\bmc}{w}),\, \bmz)$
\conference{(See 
Proposition 5
in the full version~\cite{Kawamoto:22:JELIA:arxiv} for details).
}%
\arxiv{(See \Propo{prop:G:equiv:DG} in \App{sub:app:relation:DG:CD} for details).}

We next define a \emph{lazily intervened world} $w\intvL{\subst{c}{x}}$ as the world where $x$'s value is unchanged but the other variables dependent on $x$ are computed using $\sema{c}{w}$ instead of $\sema{x}{w}$.
Formally, 
$w\intvL{\subst{c}{x}}$ is defined by
$\semf_{w\intvL{\subst{c}{x}}} = \semf_{w}$,
$\datagen{w\intvL{\subst{c}{x}}}(y)= x$ if $y = x$, and
$\datagen{w\intvL{\subst{c}{x}}}(y)= \datagen{w}(y)[x \mapsto c]$ if $y \neq x$.
E.g., in \Fig{Fig:difference_interventions}, 
$\sema{x}{w\intvL{\subst{c}{x}}} = \sema{f_1(z)}{w}$.

For $\bm{x} \,{=}\, \la x_1, \ldots, x_k \ra$ and $\bm{c} \,{=}\, \la c_1, \ldots, c_k \ra$, 
we define $\intvE{\subst{\bm{c}}{\bm{x}}}$ from the simultaneous replacement $\datagen{w\intvE{\subst{c_1}{x_1}, \ldots, \subst{c_k}{x_k}}}$.
We also define $\intvL{\subst{\bmc}{\bmx}}$ analogously.

\rulespsmall
\noindent
\textbf{Kripke Model.}
Let $\Pred$ be a set of predicate symbols.
For a variable tuple $\bmx$ and a deterministic constant tuple $\bmc$,
we introduce an \emph{intervention relation} $w \relExc w'$ that expresses a transition from a world $w$ to another $w'$ by the intervention $\intvE{\subst{\bmc}{\bmx}}$;
namely, 
${\relExc} = 
\{ (w, w') \in \calw \times \calw \mid
w' \,{=}\, w\intvE{\subst{\bmc}{\bmx}} \}$.

Then we define a \emph{Kripke model for statistical causality} as
a tuple 
$\M = (\calw, \allowbreak (\relExc\!)_{\bmx\in\cvar^+\!,\bmc\in\dConst^+}, \allowbreak \calv)$ 
consisting of:
(1) a set $\calw$ of all possible worlds 
over the set $\cvar$ of causal variables;
(2) for each $\bmx\in\cvar^+$ and 
$\bmc\in\dConst^+$, 
an \emph{intervention relation} 
$\relExc$;
(3)
a valuation $\calv$
that maps a $k$-ary predicate symbol $\eta\in\Pred$ to a set $\calv(\eta)$ of $k$-tuples of distributions.

Notice that different worlds $w$ and $w'$ in $\calw$ may have 
different data generators $\datagen{w}$ and $\datagen{w'}$ corresponding to different causal diagrams;
that is, $\calw$ specifies all possible causal diagrams.
Furthermore, different worlds $w$ and $w'$ may also have different interpretations $\semf_{w}$ and $\semf_{w'}$ of function symbols if we do not have the knowledge of functions~\cite{Kawamoto:07:JSIAM}.

\section{Statistical Causality Language}
\label{sec:assertion-logic}

\PARAGRAPHws{Predicates and Causal Predicates}
Classical predicates in $\Pred$ 
describe \emph{statistical correlation} among the distributions of variables, and are interpreted using a valuation $\calv$.
For example, $\pos(x)$ expresses that $x$ takes each value in the domain $\calo$ with a non-zero probability.
However, predicates 
cannot express the \emph{statistical causality} among variables, whose interpretation relies on a causal diagram.
Thus, we introduce a set $\CPred$ of \emph{causal predicates} 
(e.g., $\dsep$, $\nanc$, $\allnanc$)
and interpret them using a data generator~$\datagenerator$ instead of a valuation~$\calv$.

\PARAGRAPH{Syntax and Semantics of \StaCL{}}
We define the set $\Fml$ of \emph{formulas}:
For $\eta\in\Pred$, $\chi \in \CPred$, $\bmx \,{\in}\, \var^+$, $\bmu \,{\in}\, \Term^+$, $\bmc \,{\in}\, \Const^+$, and $f \,{\in}\, \Func \cup \fvar$,
\begin{align*}
\phi \,\mathbin{::=} &\,
\eta(\bmx, \ldots, \bmx) \,|\,
\chi(\bmx, \ldots, \bmx) \,|\,
 \bmu \,{=}\, \bmu \mid
 f \,{=}\, f 
 \,|\,
 \mytrue \,|\,
 \neg \phi \,|\,
 \phi \land \phi 
 \mid
 \intvE{\subst{\bmc}{\bmx}} \phi \mid
 \intvL{\subst{\bmc}{\bmx}} \phi 
{.}
\end{align*}
Intuitively, 
$\intvE{\subst{\bmc}{\bmx}} \phi$ (resp. $\intvL{\subst{\bmc}{\bmx}} \phi$) expresses that $\phi$ is satisfied in the eager (resp. lazy) intervened world.
We assume that each variable appears at most once in $\bmx$ in $\intvE{\subst{\bmc}{\bmx}}$ 
and $\intvL{\subst{\bmc}{\bmx}}$. 
We use syntax sugar 
$\myfalse$,
$\lor$, $\rightarrow$, and $\leftrightarrow$ as usual.
Note that the formulas have no quantifiers over variables.

We interpret  
a formula in a world $w$ in a Kripke model $\M$ 
by:
{\footnotesize
\begin{align*}
\span\span\span
\M, w \models \eta(\bmx_1, \ldots, \bmx_k)
 ~\mbox{ iff }~
(\sema{x_1}{w}, \ldots, \sema{x_k}{w}) \in \calv(\eta)
\span\span\span
\\
\M, w \models \bmu = \bmu'
& ~\mbox{ iff }~
\sema{\bmu}{w} = \sema{\bmu'}{w}
&&&
\M, w \models f = f'
&~\mbox{ iff }~
\sema{f}{w} = \sema{f'}{w}
\\
\M, w \models \neg \phi
 &~\mbox{ iff }~
\M, w \not\models \phi
&&&
\M, w \models \phi \land \phi'
 &~\mbox{ iff }~
\M, w \models \phi
\mbox{ and }
\M, w \models \phi'
\\
\M, w \models \intvE{\subst{\bmc}{\bmx}} \phi
&~\mbox{ iff }~
\M, w\intvE{\subst{\bmc}{\bmx}} \models \phi
&&&
\M, w \models \intvL{\subst{\bmc}{\bmx}} \phi
&~\mbox{ iff }~
\M, w\intvL{\subst{\bmc}{\bmx}} \models \phi
{,}
\end{align*}
}%
where
$w\intvE{\subst{\bmc}{\bmx}}$ 
and $w\intvL{\subst{\bmu}{\bmx}}$ are 
intervened worlds
and the interpretation of atomic formulas with causal predicates $\chi$ is given below. 
For brevity, we often omit $\M$.

Note that $\eta(x_1, \ldots, x_k)$ represents a property of $k$ independent distributions $\sema{x_1}{w}, \ldots, \sema{x_k}{w}$, 
where the randomness $\rand_i$ in each $\semrX{x_i}{w}{\rand_i}$ is chosen independently.
In contrast, $\eta(\la x_1, \ldots, x_k \ra)$ expresses a property of a single joint distribution,
since the same $\rand$ is used in all of $\semr{x_1}{w}$, $\ldots$, $\semr{x_k}{w}$.

Atomic formulas with causal predicates $\chi$ are interpreted using a causal diagram $\diag{w}$ corresponding to $\datagen{w}$.
Let $\Anc(\bmy)$ is the set of all ancestors of $\bmy$ in $\diag{w}$,
and
$\Pa(\bmy)$ be the set of all parent variables of $\bmy$ in $\diag{w}$.
Then:
{\footnotesize
\begin{align*}
w \models \dsep(\bmx, \bmy, \bmz)
& \mbox{ iff }
\mbox{$\bmx$ and $\bmy$ are $d$-separated by $\bmz$ in $\diag{w}$}
\\[-0.3ex]
w \models \nanc(\bmx, \bmy)
& \mbox{ iff }
\bmx \,{\cap}\, \Anc(\bmy) = \emptyset
\mbox{ and } \bmx \,{\cap}\, \bmy = \emptyset
\\[-0.3ex]
w \models \allnanc(\bmx, \bmy, \bmz)
& \mbox{ iff }
\bmx = \bmy \setminus \Anc(\bmz)
\\[-0.3ex]
w \models \pa(\bmx, \bmy)
& \mbox{ iff }
\bmx = \Pa(\bmy)
\mbox{ and } \bmx\cap\bmy=\emptyset
{,}
\end{align*}
}%
where the $d$-separation 
\footnote{%
An undirected path in a causal diagram $\diag{w}$ is said to be \emph{$d$-separated} by $\bmz$ 
if it has either
(a) a chain $v' \garrow v \garrow v''$ s.t. $v\in\bmz$,
(b) a fork $v' \leftgarrow v \garrow v''$ s.t. $v\in\bmz$,
or 
(c) a collider $v' \garrow v \leftgarrow v''$ s.t. $v\not\in\bmz \cup \Anc(\bmz)$.
$\bmx$ and $\bmy$ are said to be \emph{$d$-separated} by $\bmz$ if all undirected paths between variables in $\bmx$ and in $\bmz$ are $d$-separated by $\bmz$.
}
of $\bmx$ and $\bmy$ by $\bmz$~\cite{Verma:88:UAI} 
is a sufficient condition for the conditional independence of $\bmx$ and $\bmy$ given $\bmz$
\conference{(See Appendix A
in the full version~\cite{Kawamoto:22:JELIA:arxiv}).%
}%
\arxiv{(See \App{sec:appendix:models} for details).}%

\rulespsmall
\noindent
\textbf{Formalization of Causal Effect.}
Conventionally, the conditional probability of $\bmy$ given $\bmz=\bmoB$ after an intervention $\bmx=\bmoA$ is expressed using the $\mydo$-operator by
$P( \bmy \,|\, do(\bmx=\bmoA), \bmz=\bmoB )$.
This causal effect can be expressed using \StaCL{}:
\begin{restatable}[Causal effect]{prop}{PropCausal}
\label{prop:causal}
Let $w$ be a world, $\bmx,\allowbreak\bmy,\bmz\in\var(w)^+$ be disjoint, 
$\bmc \,{\in}\, \dConst^+$, 
$\bmc' \,{\in}\, \Const^+$, and $f \,{\in}\, \Func$.
Then: 
\begin{enumerate}\renewcommand{\labelenumi}{(\roman{enumi})}
\item
$w \models \intvE{\subst{\bmc}{\bmx}} (\bmc' \,{=}\, \bmy)$
iff 
there is a distribution $P_{\diag{w}}$ that is factorized according to $\diag{w}$ and satisfies 
$P_{\diag{w}}( \bmy \,|\, do(\bmx \,{=}\, \bmc) ) \,{=}\, \sema{\bmc'}{w}$.
\item
$w \models \intvE{\subst{\bmc}{\bmx}} (f \,{=}\, \bmy|_{\bmz})$
iff 
there is a distribution $P_{\diag{w}}$ that is factorized according to $\diag{w}$ and satisfies 
$P_{\diag{w}}( \bmy \,|\, do(\bmx \,{=}\, \bmc), \bmz ) \,{=}\, \sema{f}{w}$.
\end{enumerate}
\end{restatable}%

If $\bmx$ and $\bmy$ are $d$-separated by $\bmz$, they are conditionally independent given $\bmz$~\cite{Verma:88:UAI} (but not vice versa).
\StaCL{} can express this by
$
\modelsg ( \dsep(\bmx, \bmy, \bmz) \land \pos(\bmz)\,\!\rightarrow
\bmy|_{\bmz,\bmx=\bmc} = \bmy|_{\bmz},
$
where $\pos(\bmz)$ means that $\bmz$ takes each value with a positive probability,
and $\modelsg\, \phi$ is defined as $w \modelsg \phi$ for all world $w$ having the data generator $\dgen{}$.
Furthermore,
if $\sema{\bmx}{w}$ and $\sema{\bmy}{w}$ are conditionally independent given $\sema{\bmz}{w}$ 
for any world $w$ with the data generator $\datagen{w}$, 
then they are $d$-separated by~$\bmz$:
$
\modelsg\, (\pos(\bmz) \rightarrow \bmy|_{\bmz,\bmx=\bmc} = \bmy|_{\bmz})
\mbox{ implies }
\modelsg\, \dsep(\bmx, \bmy, \bmz)
$
\conference{(See Proposition 15
in the full version~\cite{Kawamoto:22:JELIA:arxiv} for details).}
\arxiv{(See \Propo{prop:dsep:CInd} in Appendix~\ref{sub:StaCL:ax:dsep}).}

\begin{figure}[t]
\centering

\begin{screen}[7]
\begin{footnotesize}
\hspace{1ex}{\bf Axioms for probability distributions}
\begin{flushleft}
\[
\renewcommand{\arraystretch}{1.10}
\begin{array}{@{\hspace{-1.0ex}}l@{\hspace{-0.5ex}}l}
\phantom{\mbox{\axXpdEL{}}}~&~~
\\[-9.5ex]
\mbox{\axEqName}
&~~~ \vdashg c^{(\datagenerator, \bmx)} = \bmx
\\
\mbox{\axEqFunc}
&~~~ \vdashg f^{(\datagenerator, \bmy|_{\bmz\!,\bmx=\bmc})} = \bmy|_{\bmz\!,\bmx=\bmc}
\\
\mbox{\axPD}
&~~~\vdashg
(
	\pos(\bmx)
	\land
	\nA = \bmx
	\land
	f = \bmy|_{\bmx}
	\land
	\nB = \bmx \,{\joint}\, \bmy
)
\rightarrow
\nB {=} f(\nA)
~\hspace{5.5ex}~
\\
\mbox{\axMPD{}}
&~~~\vdashg
\bmxB\mrgn{\bmxC} = \bmxC
~~\mbox{ if } \bmxC \subseteq \bmxB
\end{array}
\]
\end{flushleft}
\end{footnotesize}
\end{screen}
\caption{The axioms of \AX{} for probability distributions, where
$\bmx, \bmxB, \bmxC, \bmy \in \cvar^+$ are disjoint,
$\nA, \nB, \allowbreak c^{(\datagenerator, \bmx)} \in \Const$,
$f, f^{(\datagenerator, \bmy|_{\bmz\!,\bmx=\bmc})} \in \Func$.}
\label{fig:AX1}
\end{figure}

\begin{figure}[t]
\centering
\begin{screen}[7]
\begin{footnotesize}
\hspace{1ex}{\bf Axioms for eager interventions}
\begin{flushleft}
\[
\renewcommand{\arraystretch}{1.10}
\begin{array}{@{\hspace{-1.0ex}}l@{\hspace{-0.5ex}}l}
\phantom{\mbox{\axXpdEL{}}}~&~~
\\[-9.5ex]
\mbox{\axDGEI}
&~~~ \vdashg \intvE{\subst{\bmc}{\bmx}} \phi
~\,\mbox{ iff } \vdashx{\datagenerator\intvE{\subst{\bmc}{\bmx}}}\phi
\\
\mbox{\axEffect~~}
&~~~ \vdashg 
\intvE{\subst{\bmc}{\bmx}} (\bmx = \bmc)
\\
\mbox{\axEqEI}
&~~~ \vdashg \bmuB = \bmuC \leftrightarrow \intvE{\subst{\bmc}{\bmx}} (\bmuB = \bmuC)
\mbox{ if } \fv{\bmuB} = \fv{\bmuC} = \emptyset
\\
\mbox{\axSplitE}
&~~~\vdashg
\intvE{\subst{\bmcB}{\bmxB},\, \subst{\bmcC}{\bmxC}} \phi
\rightarrow
\intvE{\subst{\bmcB}{\bmxB}} \intvE{\subst{\bmcC}{\bmxC}} \phi
\\
\mbox{\axSimulE{}}
&~~~\vdashg
\intvE{\subst{\bmcB}{\bmxB}} \intvE{\subst{\bmcC}{\bmxC}} \phi
\rightarrow
\intvE{\subst{\bmcB'}{\bmxB'},\, \subst{\bmcC}{\bmxC}} \phi
~\mbox{ if }
\bmxB' = \bmxB {\setminus} \bmxC,~
\bmcB' = \bmcB {\setminus} \bmcC
\hspace{-3.5ex}
\\
\mbox{\axRptE{}}
&~~~\vdashg
\intvE{\subst{\bmc}{\bmx}} \phi
\rightarrow
\intvE{\subst{\bmc}{\bmx}} \intvE{\subst{\bmc}{\bmx}} \phi
\\
\mbox{\axCmpEI}
&~~~ \vdashg 
\big(
\intvE{\subst{\bmcB}{\bmxB}} (\bmxC = \bmcC) \land
\intvE{\subst{\bmcB}{\bmxB}} (\bmxD = \bmu)
\big)
\rightarrow
\intvE{\subst{\bmcB}{\bmxB}, \subst{\bmcC}{\bmxC}} (\bmxD = \bmu)
\hspace{-5ex}~
\\
\mbox{\axDistrE}^{\neg}
&~~~\vdashg
(\intvE{\subst{\bmc}{\bmx}} \neg \phi)
\leftrightarrow
(\neg \intvE{\subst{\bmc}{\bmx}} \phi)
\\
\mbox{\axDistrE}^{\land}
&~~~\vdashg
(\intvE{\subst{\bmc}{\bmx}} (\phi_1 \land \phi_2))
\leftrightarrow
(\intvE{\subst{\bmc}{\bmx}} \phi_1 \land \intvE{\subst{\bmc}{\bmx}} \phi_2)
~\hspace{18ex}~
\end{array}
\reducespace{-0.5ex}
\]
\end{flushleft}
\end{footnotesize}
\end{screen}

\begin{screen}[7]
\begin{footnotesize}
\hspace{1ex}{\bf Axioms for lazy interventions}
\begin{flushleft}
\[
\renewcommand{\arraystretch}{1.10}
\begin{array}{@{\hspace{-1.0ex}}l@{\hspace{-0.5ex}}l}
\phantom{\mbox{\axXpdEL{}}}~&~~
\\[-9.5ex]
\!\mbox{\axCondL{}}
&~~~\vdashg
( f = \bmy|_{\bmx = \bmc} ) \leftrightarrow
 \intvL{\subst{\bmc}{\bmx}} ( f = \bmy|_{\bmx = \bmc} )
\\
& \hspace{-11ex}\mbox{Other axioms are analogous to eager interventions except for \axEffect{}.\hspace{2.5ex}~}
\end{array}
\reducespace{-1.4ex}
\]
\end{flushleft}
\end{footnotesize}
\end{screen}

\begin{screen}[7]
\begin{footnotesize}
\hspace{1ex}{\bf Axioms for the exchanges of eager and lazy interventions}
\begin{flushleft}
\[
\renewcommand{\arraystretch}{1.10}
\begin{array}{@{\hspace{-1.0ex}}l@{\hspace{-0.5ex}}l}
\phantom{\mbox{\axXpdEL{}}}~&~~
\\[-9.5ex]
\mbox{\axXpdEL{}}
&~~~\vdashg
 (\intvE{\subst{\bmc}{\bmx}} \bmc' = \bmy)
 \leftrightarrow
 (\intvL{\subst{\bmc}{\bmx}} \bmc' = \bmy)
\\
\mbox{\axXcdEL{}}
&~~~\vdashg
 \pos(\bmz) \,{\rightarrow}\,
 \big(
 (\intvE{\subst{\bmc}{\bmx}} f \,{=}\, \bmy|_{\bmz})
 {\leftrightarrow}
 (\intvL{\subst{\bmc}{\bmx}} f \,{=}\, \bmy|_{\bmz})
 \big)
~\hspace{17ex}~
\end{array}
\reducespace{-1.2ex}
\]
\end{flushleft}
\end{footnotesize}
\end{screen}
\caption{The axioms of \AX{}, where
$\bmx, \bmxB, \bmxC, \bmxD, \bmy, \bmz \in \cvar^+$ are disjoint,
$f \in \Func$,
$\bmc, \bmcB, \bmcC \in \dConst^+$,
$\bmc' \in \Const^+$,
$\bmu,\bmuB,\bmuC \in \Term^+$,
and
$\phi, \phi_1, \phi_2 \in \Fml$.}
\label{fig:AX2}
\end{figure}

\section{Axioms for $\text{\StaCL{}}$}
\label{sec:axioms}

We present a sound deductive system for \StaCL{} in the Hilbert style.
Our system consists of axioms and rules for the judgments of the form $\Gamma \vdashg \varphi$.

The deductive system is stratified into two groups.
The system \AX{}, determined by the axioms in \Figs{fig:AX1} and~\ref{fig:AX2}, concerns the derivation of $\Gamma \vdashg \varphi$ that 
does not involve causal predicates (e.g., $\pa$, $\nanc$, $\dsep$).
The system \AXwCP{}, determined by the axioms in \Fig{fig:AXCP}, concerns the derivation of a formula $\varphi$ possibly equipped with causal predicates in a judgment $\Gamma \vdashg \varphi$.
\reducespace{-0.1ex}

In these systems, we deal only with the reasoning that is independent of a causal diagram.
Indeed, in \Sec{sec:reasoning-StaCL}, we will present examples of reasoning using the deductive system \AXwCP{} that 
do not refer to a specific causal diagram.

\PARAGRAPH{Axioms of \AX{}}
\Fig{fig:AX1} shows the axioms of the deductive system \AX{},
where we omitted the axioms for propositional logic and equations
(\axPT{} for the propositional tautologies, \axMP{} for the modus ponens, \axEqA{} for the reflexivity, and \axEqB{} for the substitutions for formulas).
\axEqName{} and \axEqFunc{} represent the definitions of constants and function symbols corresponding to causal variables.
\axPD{} describes the relationships among the prior distribution $\bmx$, the conditional distribution $\bmy|_{\bmx}$ of $\bmy$ given $\bmx$, and the joint distribution $\bmx\joint\bmy$.
\axMPD{} represents the computation $\mrgn{\bmxC}$ of the marginal distribution $\bmxC$ from a joint distribution $\bmxB$.
\reducespace{-0.1ex}

The axioms named with the subscript \textsc{EI} deal with eager intervention.
Remarkably, \axDGEI{} reduces the derivation of $\vdashg \intvE{\subst{\bmc}{\bmx}} \phi$, which involves an intervention modality $\intvE{\subst{\bmc}{\bmx}}$, to the derivation of $\vdashx{\datagenerator\intvE{\subst{\bmc}{\bmx}}}\phi$, which does not involve the modality under the modified data generator $\datagenerator\intvE{\subst{\bmc}{\bmx}}$.
The axioms $\mbox{\axDistrE}^{\neg}$ and $\mbox{\axDistrE}^{\land}$ allow for pushing intervention operators outside logical connectives.
\reducespace{-0.1ex}

The axioms with the subscript \textsc{LI} deal with lazy intervention; 
they are analogous to the corresponding \textsc{EI}-rules.
The axioms with the subscript \textsc{EILI} describe when an eager intervention can be exchanged with a lazy intervention.

\PARAGRAPH{Axioms of \AXwCP{}}
\Fig{fig:AXCP} shows the axioms for \AXwCP{}.
\axDsepCIndB{} represents that $d$-separation implies conditional independence.
\axDsepSm{}, \axDsepDc{}, \axDsepWU{}, and \axDsepCn{} are the \emph{semi-graphoid} axioms~\cite{Verma:88:UAI}, characterizing the $d$-separation.
However, these well-known axioms are not sufficient to derive the relationships between $d$-separation and interventions.
Therefore, we introduce two axioms 
\axDsepEN{} and \axDsepLN{}
in \Fig{fig:AXCP} for the $d$-separation before/after interventions, 
and four axioms to reason about the relationships between the causal predicate $\nanc$ and the interventions/$d$-separation (named $\textsc{Nanc}_{\{\textsc{1,2,3,4}\}}$ in \Fig{fig:AXCP}).
By \axNancAll{}, \axPaToNanc{}, and \axPaToDsep{}, we transform the formulas using $\allnanc$ and $\pa$ into those with $\nanc$ or $\dsep$.

\PARAGRAPH{Properties of Axiomatization}
For a data generator $\datagenerator$, a set $\Gamma \eqdef \{\psi_1,\dots,\psi_n\}$ of formulas, and a formula $\phi$, we write $\Gamma \vdashg \phi$ if there is a derivation of $\vdashg ( \psi_1 \land \dots \land \psi_n ) \rightarrow \phi$ using axioms of \AX{} or \AXwCP{}. 
We write $\Gamma \modelsg \phi$ if for all model $\M$ and all world $w$ having the data generator $\datagenerator$,\, 
$\M, w \models \phi$.
Then we obtain the \emph{soundness} of  \AX{} and \AXwCP{}.
\begin{restatable}[Soundness]{thm}{ThmSoundComplete}
\label{thm:sound:complete}
Let $\datagenerator$ be a finite, closed, and acyclic data generator.
$\Gamma \subseteq \Fml$, and $\phi\in\Fml$.
If $\Gamma \vdashg \phi$ then $\Gamma \modelsg \phi$.
\end{restatable}%
\conference{We show the proof in Appendices B and C
in the full version~\cite{Kawamoto:22:JELIA:arxiv}.
}%
\arxiv{See \Apps{sec:app:StaCL:sound} and~\ref{sec:StaCL:ax:pred} for the proof.}
As shown in \Sec{sec:reasoning-StaCL},
\AXwCP{} is expressive enough to derive the rules of Pearl's do-calculus~\cite{Pearl:95:biometrika};
it can reason about all causal effects identifiable by the do-calculus
(without referring to a specific causal diagram).
Furthermore, \AX{} includes/derives the axioms used in the previous work~\cite{Barbero:21:jphil} that are complete w.r.t. a different semantics without dealing with probability distributions.
We leave investigating whether \AX{} is complete w.r.t. our Kripke model for future work.
We also remark that
\AXwCP{} has axioms corresponding to the composition and effectiveness axioms introduced by Galles and Pearl~\cite{Galles:98:FoS}.

\begin{figure}[t]
\centering
\begin{screen}[7]
\begin{small}
\hspace{1ex}{\bf Axioms for $d$-separation}
\begin{flushleft}
\[
\renewcommand{\arraystretch}{1.10}
\begin{array}{@{\hspace{-1.0ex}}l@{\hspace{-0.5ex}}l}
\phantom{\mbox{\axXpdEL{}}}~&~~
\\[-7.5ex]
\!\mbox{\axDsepCIndB{}}
&~~~ \vdashg ( \dsep(\bmx, \bmy, \bmz) \land \pos(\bmz)) \rightarrow
\bmy|_{\bmz,\bmx=\bmc} \,{=}\, \bmy|_{\bmz}
\\
\!\mbox{\axDsepSm{}}
&~~~ \vdashg \dsep(\bmx, \bmy, \bmz) \leftrightarrow \dsep(\bmy, \bmx, \bmz)
\\
\!\mbox{\axDsepDc{}}
&~~~ \vdashg \dsep(\bmx, \bmy \cup \bmy', \bmz) 
\rightarrow (\dsep(\bmx, \bmy, \allowbreak \bmz) \land  \dsep(\bmx, \bmy', \bmz))
\\
\!\mbox{\axDsepWU{}}
&~~~ \vdashg \dsep(\bmx, \bmy \cup \bmv, \bmz) \rightarrow \dsep(\bmx, \bmy, \bmz \cup \bmv)
\\
\!\mbox{\axDsepCn{}}
&~~~ \vdashg (\dsep(\bmx, \bmy, \bmz) {\land} \dsep(\bmx, \bmv, \bmz \cup \bmy)) 
\rightarrow \dsep(\bmx, \bmy \cup \bmv, \bmz)
~\hspace{2.5ex}~
\end{array}
\]
\end{flushleft}
\end{small}
\end{screen}

\begin{screen}[7]
\begin{small}
\hspace{1ex}\mbox{\bf Axioms for $d$-separation with interventions}
\begin{flushleft}
\[
\renewcommand{\arraystretch}{1.10}
\begin{array}{@{\hspace{-2.0ex}}l@{\hspace{-0.5ex}}l}
\phantom{\mbox{\axXpdEL{}}}~&~~
\\[-9.5ex]
\!\mbox{\axDsepEN{}}
&~~~ \vdashg (\intvE{\subst{\bmc}{\bmz}} \dsep(\bmx, \bmy, \bmz)) \leftrightarrow
  \dsep(\bmx, \bmy, \bmz)
\\
\!\mbox{\axDsepLN{}}
&~~~ \vdashg (\intvL{\subst{\bmc}{\bmz}} \dsep(\bmx, \bmy, \bmz)) \leftrightarrow
  \dsep(\bmx, \bmy, \bmz)
~\hspace{19.5ex}~
\end{array}
\reducespace{-0.9ex}
\]
\end{flushleft}
\end{small}
\end{screen}

\begin{screen}[7]
\begin{small}
\begin{flushleft}
\hspace{1ex}{\bf Axioms with other causal predicates}
\[
\renewcommand{\arraystretch}{1.10}
\begin{array}{@{\hspace{-1.0ex}}l@{\hspace{-0.8ex}}l}
\phantom{\mbox{\axXpdEL{}}}~&~~
\\[-5.5ex]
\!\mbox{\axNancAB{}}
&~~~ \vdashg ( \nanc(\bmx, \bmy) {\land} \nanc(\bmx, \bmz) ) 
\rightarrow 
(f = \bmy|_{\bmz} \leftrightarrow \intvE{\subst{\bmc}{\bmx}} (f = \bmy|_{\bmz}))
\\
\!\mbox{\axNancB{}}
&~~~ \vdashg \nanc(\bmx, \bmy) \leftrightarrow  \intvE{\subst{\bmc}{\bmx}} \nanc(\bmx, \bmy)
\\
\!\mbox{\axNancC{}}
&~~~ \vdashg \nanc(\bmx, \bmy) \rightarrow \intvE{\subst{\bmc}{\bmx}} \dsep(\bmx, \bmy, \emptyset)
\\
\!\mbox{\axNancD{}}
&~~~ \vdashg (\nanc(\bmx, \bmz) \,{\land}\, \dsep(\bmx, \bmy, \bmz)) \,{\rightarrow}\, \nanc(\bmx, \bmy)
\hspace{-4ex}~
\\
\!\mbox{\axNancAll{}}
&~~~ \vdashg \allnanc(\bmx, \bmy, \bmz) \rightarrow \nanc(\bmx, \bmz)
\hspace{27.5ex}~
\\
\!\mbox{\axPaToNanc{}}
&~~~ \vdashg \pa(\bmx, \bmy) \rightarrow \nanc(\bmy, \bmx)
\\
\!\mbox{\axPaToDsep{}}
&~~~ \vdashg \pa(\bmz, \bmx) \rightarrow
\intvL{\subst{\bmc}{\bmx}} \dsep(\bmx, \bmy, \bmz)
\end{array}
\reducespace{-0.2ex}
\]
\end{flushleft}
\end{small}
\end{screen}
\caption{The additional axioms for $\AXCP{}\!$ where
$\bmx, \bmy, \bmy'\!, \bmz, \bmv \,{\in} \cvar^+$ are \allowbreak disjoint,
$\bmc \,{\in}\, \dConst^+$, 
and 
$f {\in}\, \Func$.
}
\label{fig:AXCP}
\end{figure}

\section{Reasoning About Statistical Causality}
\label{sec:reasoning-StaCL}

\begin{figure*}[t]
\begin{footnotesize}
\hspace{-4ex}
\begin{minipage}[b]{\linewidth}
\centering
\infer [\!\mbox{\axMPD, \axEqB}]
{
  \psiPre \vdashg ( \intvE{\subst{\bmc}{\bmx}} \nA = \bmy )
  \leftrightarrow ( \psi_1 \land \psi_2 \land \psi_3 )
}
{
  \infer [\!\mbox{\axPD, \axEqB}]
  {
    \psiPre \vdashg ( \intvE{\subst{\bmc}{\bmx}} \nA = ( \bmy\joint\bmz)\!\mrgn{\bmy} )
    \leftrightarrow ( \psi_1 \land \psi_2 \land \psi_3 )
  }
  {
    \infer[\!\mbox{\axEqName, \axEqFunc, \axEqB}]
    {
      \psiPre \vdashg ( \intvE{\subst{\bmc}{\bmx}} \nA = ( \bmy|_{\bmz}(\bmz))\!\mrgn{\bmy} )
      \leftrightarrow ( \psi_1 \land \psi_2 \land \psi_3 )
    }
    {
      \infer[\!\mbox{\axDistrE}^{\land}]
      {
        \psiPre \vdashg ( \intvE{\subst{\bmc}{\bmx}} ( \psi_0 \land \psi_2 \land \psi_3 ) )
        \leftrightarrow ( \psi_1 \land \psi_2 \land \psi_3 )
      }
      {
        \infer[]
        {
          \psiPre \vdashg ( \intvE{\subst{\bmc}{\bmx}} \psi_0 \land \intvE{\subst{\bmc}{\bmx}} \psi_2 \land \intvE{\subst{\bmc}{\bmx}} \psi_3 )
          \leftrightarrow ( \psi_1 \land \psi_2 \land \psi_3 )
        }
        {
          \infer[\!\mbox{\axDoB}]
          {
            \vdashg \psiDA \rightarrow ( ( \intvE{\subst{\bmc}{\bmx}} \psi_0 ) {\leftrightarrow} \psi_1 )
          }
          {
          }
          \hspace{-1ex}
          &
          \infer[\!\mbox{\axDoC}]
          {
            \vdashg \psiDB \rightarrow ( ( \intvE{\subst{\bmc}{\bmx}} \psi_2 ) {\leftrightarrow} \psi_2 )
          }
          {
            \infer[\!\mbox{\axNancC}]
            {
              \vdashg \psiNanc \rightarrow ( ( \intvE{\subst{\bmc}{\bmx}} \psi_2 ) {\leftrightarrow} \psi_2 )
            }
            {
            }
          }
          \hspace{-3.5ex}
          &
          \infer[\!\mbox{\axEqEI}~]
          {
            \vdashg ( \intvE{\subst{\bmc}{\bmx}} \psi_3 ) {\leftrightarrow} \psi_3
          }
          {
          }
        }
      }
    }
  }
}
\end{minipage}
\end{footnotesize}
\vspace{-1ex}
\caption{Sketch of a derivation tree for the correctness of the backdoor adjustment (\Sec{sec:overview}) using \AXwCP{}
where
$\psiPos \eqdef \pos(\bmz\joint\bmx)$,\,
$\psiDA \eqdef \intvL{\subst{\bmc}{\bmx}} \dsep(\bmx, \bmy, \bmz) \land \psiPos$,\,
$\psiDB \eqdef \intvE{\subst{\bmc}{\bmx}} \dsep(\bmx, \bmz, \emptyset) \land \psiPos$,\,
$\psiNanc \eqdef \nanc(\bmx, \bmz) \land \psiPos$,\,
$\psiPre \eqdef \psiDA \land \psiNanc$,\,
$\psi_0 \eqdef ( f = \bmy|_{\bmz} )$,\,
$\psi_1 \eqdef ( f = \bmy|_{\bmz,\bmx=\bmc} )$,\,
$\psi_2 \eqdef ( \nB = \bmz )$,\, and
$\psi_3 \eqdef ( \nA = f(\nB)\!\mrgn{\bmy} )$.
}
\label{fig:overview:proof2:drug}
\end{figure*}

\PARAGRAPHws{Deriving the Rules of the Do-Calculus}
Using \StaCL{}, we express the \emph{do-calculus}'s rules~\cite{Pearl:95:biometrika},
which are sufficient to compute all identifiable causal effects from observable quantities~\cite{Huang:06:UAI,Shpitser:06:UAI}.
Let $\fv{\phi}$ be
the set of all variables occurring in a formula $\phi$,
and $\cdv{\phi}$ be the set of all \emph{conditioning variables} in~$\phi$.

\begin{restatable}[Do-calculus rules]{prop}{PropDoCalculus}
\label{prop:do-calculus}
Let $\bmv, \bmx, \bmy,\bmz\in\cvar^+$ be disjoint,
$\bmxB, \bmxC \in \cvar^+$,
and $\bmcA,\bmcB,\bmcC\in\dConst^+$.
Let $S = \cdv{\phi_0} \cup \cdv{\phi_1}$.
\begin{enumerate}
\item 
\axDoA{}.
~\hspace{0ex}~
Introduction/elimination of conditioning:
\begin{align*}
\hspace{-3ex}
\vdashg
& \intvE{\subst{\bmcA}{\bmv}} 
 ( \dsep(\bmx, \bmy, \bmz) \land {\textstyle \bigwedge_{\bms \in S}}\, \pos(\bms) ) 
\rightarrow 
( (\intvE{\subst{\bmcA}{\bmv}} \phi_0 )
  \leftrightarrow
  \intvE{\subst{\bmcA}{\bmv}} \phi_1)
\end{align*} 
where $\phi_1$ is obtained by replacing some occurrences of $\bmy|_{\bmz}$ in $\phi_0$ with $\bmy|_{\bmz,\bmx=\bmcB}$;
\item 
\axDoB{}.
~\hspace{0ex}~
Exchange between intervention and conditioning:
\begin{align*}
\hspace{-3ex}
\vdashg
& \intvE{\subst{\bmcA}{\bmv}} \intvL{\subst{\bmcB}{\bmx}} 
 ( \dsep(\bmx, \bmy, \bmz) \land {\textstyle \bigwedge_{\bms \in S}}\, \pos(\bms) )
\rightarrow\!%
( ( \intvE{\subst{\bmcA}{\bmv},\subst{\bmcB}{\bmx}} \phi_0 )
  \leftrightarrow
  \intvE{\subst{\bmcA}{\bmv}} \phi_1 )
\end{align*}
where $\phi_1$ is obtained by replacing every occurrence of $\bmy|_{\bmz}$ in $\phi_0$ with $\bmy|_{\bmz,\bmx=\bmcB}$;
\item 
\axDoC{}
~\hspace{0ex}~
Introduction/elimination of intervention:
\begin{align*}
\hspace{-3ex}
\vdashg
& 
\intvE{\subst{\bmcA}{\bmv}} 
(\allnanc(\bmxB, \bmx, \bmy) \land
\intvE{\subst{\bmcB}{\bmxB}} 
 ( \dsep(\bmx, \bmy, \bmz) \land \pos(\bmz) ) )
\\
& \hspace{0ex}\!%
\rightarrow 
( ( \intvE{\subst{\bmcA}{\bmv}} \phi)
  \leftrightarrow
  \intvE{\subst{\bmcA}{\bmv},\subst{\bmcB}{\bmxB},\subst{\bmcC}{\bmxC}} \phi )
\end{align*}
where 
$\fv{\phi} = \{ \bmy|_{\bmz} \}$
and $\bmx \eqdef \bmxB\joint\bmxC$.
\end{enumerate}
\end{restatable}

By using the deductive system \AXwCP{}, we can derive those rules.
Thanks to the modal operators for lazy interventions, our derivation of those rules is partly different from Pearl's~\cite{Pearl:95:biometrika} in that it does not use diagrams augmented with the intervention arc of the form $F_{x} \garrow x$
\conference{(See Appendix D
in the full version~\cite{Kawamoto:22:JELIA:arxiv}).
}%
\arxiv{(See \App{sec:appendix:StaCL:reasoning} for details).}

\PARAGRAPH{Reasoning About Statistical Adjustment}
We present how \AXwCP{} can be used to reason about the correctness of the backdoor adjustment discussed in \Sec{sec:overview}
\conference{(See Appendix A.6 
in the full version~\cite{Kawamoto:22:JELIA:arxiv}
for the details of the backdoor adjustment).}%
\arxiv{(See \App{sub:appendix:def:diagrams} for the details of the backdoor adjustment).}
\Fig{fig:overview:proof2:drug} shows the derivation of the judgment:
\begin{align} \label{eq:general:backdoor}
\psiPre \vdashg ( \intvE{\subst{\bmc}{\bmx}} \nA = \bmy) \leftrightarrow ( \psi_1 \land \psi_2 \land \psi_3 ).
\end{align}

This judgment asserts the correctness of the backdoor adjustment in any causal diagram.
Recall that $\phiRCT \eqdef (\intvE{\subst{\bmc}{\bmx}} \nA = \bmy)$ expresses the RCT and $\phiBDA \eqdef ( \psi_1 \allowbreak \land \psi_2 \land \psi_3 )$ expresses the backdoor adjustment.
The correctness of the backdoor adjustment ($\phiRCT \leftrightarrow \phiBDA$)
depends on the precondition $\psiPre$.

By reading the derivation tree in a bottom-up manner, we observe that the proof first converts $(\intvE{\subst{\bmc}{\bmx}} \nA = \bmy)$ to a formula to which \axEqName{} and \axEqFunc{} are applicable.
Then, the derived axioms \axDoB{} and \axDoC{} in Proposition~\ref{prop:do-calculus} are used to complete the proof at the leaves of the derivation.

In \Sec{sec:overview}, we stated the correctness of the backdoor adjustment
in \eqref{eq:correctness}
 using a simpler requirement $\pa(z, x)$ instead of $\psiDA$ and $\psiNanc$.
We can derive the judgment \eqref{eq:correctness} from \eqref{eq:general:backdoor},
thanks to the axioms \axPaToDsep{} and \axPaToNanc{}.

The derivation does not mention the data generator $\dgen{}$ representing the causal diagram~$\diag{}$.
This exhibits that our logic successfully separates the reasoning about the properties of arbitrary causal diagrams
from those 
depending on a specific causal diagram.
Once we prove $\psiPre \vdashg \phiRCT \leftrightarrow \phiBDA$ using $\AXCP{}$, 
one can claim the correctness of the causal inference $( \phiRCT \leftrightarrow \phiBDA )$ by checking that the requirement $\psiPre$ indeed holds for a specific causal diagram $\diag{}$.

\section{Conclusion}
\label{sec:conclude}

We proposed statistical causality language (\StaCL{}) to formally describe and explain the correctness of statistical causal inference.
We introduced the notion of causal predicates 
and Kripke models equipped with data generators.
We defined a sound deductive system \AXwCP{}
that can deduce all causal effects derived using Pearl's do-calculus.
In ongoing and future work, we study the completeness of \AX{} and \AXwCP{}
and develop a decision procedure for \AXwCP{} 
for automated reasoning.

\rulesp
\PARAGRAPH{Acknowledgements}\,\!%
We thank Kenji Fukumizu for providing helpful information on the literature on causal inference.
The authors are supported by ERATO HASUO Metamathematics for Systems Design Project (No. JPMJER1603), JST.
Yusuke Kawamoto is supported by JST, PRESTO Grant Number JPMJPR2022, Japan, and by JSPS KAKENHI Grant Number 21K12028, Japan.
Tetsuya Sato is supported by JSPS KAKENHI Grant Number 20K19775, Japan.
Kohei Suenaga is supported by JST CREST Grant Number JPMJCR2012, Japan.

\bibliographystyle{splncs04}
\bibliography{short,short-causal,short-stat}

\begin{thebibliography}{10}
\providecommand{\url}[1]{\texttt{#1}}
\providecommand{\urlprefix}{URL }
\providecommand{\doi}[1]{https://doi.org/#1}

\bibitem{Barbero:21:jphil}
Barbero, F., Sandu, G.: Team semantics for interventionist counterfactuals:
  Observations vs. interventions. J. Philos. Log.  \textbf{50}(3),  471--521
  (2021). \doi{10.1007/s10992-020-09573-6}

\bibitem{Barbero:20:Dali}
Barbero, F., Schulz, K., Smets, S., Vel{\'{a}}zquez{-}Quesada, F.R., Xie, K.:
  Thinking about causation: {A} causal language with epistemic operators. In:
  Proc. the third International Workshop on Dynamic Logic (DaL{\'{\i}}'20).
  LNCS, vol. 12569, pp. 17--32. Springer (2020).
  \doi{10.1007/978-3-030-65840-3\_2}

\bibitem{Bochman:21:book}
Bochman, A.: A logical theory of causality. MIT Press (2021)

\bibitem{Bochman:15:AAAI}
Bochman, A., Lifschitz, V.: {Pearl}'s causality in a logical setting. In: Proc.
  the Twenty-Ninth {AAAI} Conference on Artificial Intelligence. pp.
  1446--1452. {AAAI} Press (2015),
  \url{http://www.aaai.org/ocs/index.php/AAAI/AAAI15/paper/view/9686}

\bibitem{Corander:19:APAL}
Corander, J., Hyttinen, A., Kontinen, J., Pensar, J.,
  V{\"{a}}{\"{a}}n{\"{a}}nen, J.: A logical approach to context-specific
  independence. Ann. Pure Appl. Log.  \textbf{170}(9),  975--992 (2019).
  \doi{10.1016/j.apal.2019.04.004}

\bibitem{Durand:16:Foiks}
Durand, A., Hannula, M., Kontinen, J., Meier, A., Virtema, J.: Approximation
  and dependence via multiteam semantics. In: Gyssens, M., Simari, G.R. (eds.)
  Proc. the 9th International Symposium on the Foundations of Information and
  Knowledge Systems ({FoIKS}'16). LNCS, vol.~9616, pp. 271--291. Springer
  (2016). \doi{10.1007/978-3-319-30024-5\_15}

\bibitem{Durand:18:foiks}
Durand, A., Hannula, M., Kontinen, J., Meier, A., Virtema, J.: Probabilistic
  team semantics. In: Proc. the 10th International Symposium on the Foundations
  of Information and Knowledge Systems ({FoIKS}'18). LNCS, vol. 10833, pp.
  186--206. Springer (2018). \doi{10.1007/978-3-319-90050-6\_11}

\bibitem{Fernandes-Taylor:11:BMCRN}
Fernandes-Taylor, S., Hyun, J.K., Reeder, R.N., Harris, A.H.: Common
  statistical and research design problems in manuscripts submitted to
  high-impact medical journals. BMC Research Notes  \textbf{4}(1), ~304 (2011).
  \doi{10.1186/1756-0500-4-304}, \url{https://doi.org/10.1186/1756-0500-4-304}

\bibitem{Galles:98:FoS}
Galles, D., Pearl, J.: An axiomatic characterization of causal counterfactuals.
  Foundations of Science  \textbf{3},  151--182 (1998)

\bibitem{Halpern:00:JAIR}
Halpern, J.Y.: Axiomatizing causal reasoning. J. Artif. Intell. Res.
  \textbf{12},  317--337 (2000). \doi{10.1613/jair.648},
  \url{https://doi.org/10.1613/jair.648}

\bibitem{Halpern:15:IJCAI}
Halpern, J.Y.: A modification of the {H}alpern-{P}earl definition of causality.
  In: Proc. {IJCAI}'15. pp. 3022--3033. {AAAI} Press (2015)

\bibitem{Halpern:01:IJCAI}
Halpern, J.Y., Pearl, J.: Causes and explanations: {A} structural-model
  approach - part {II:} explanations. In: Proc. {IJCAI}'01. pp. 27--34. Morgan
  Kaufmann (2001)

\bibitem{Halpern:01:UAI}
Halpern, J.Y., Pearl, J.: Causes and explanations: {A} structural-model
  approach: Part 1: Causes. In: Proc. {UAI}'01. pp. 194--202. Morgan Kaufmann
  (2001)

\bibitem{Hirvonen:19:TAMC}
Hirvonen, {\AA}., Kontinen, J., Pauly, A.: Continuous team semantics. In: Proc.
  the 15th Annual Conference on Theory and Applications of Models of
  Computation ({TAMC}'19). LNCS, vol. 11436, pp. 262--278. Springer (2019).
  \doi{10.1007/978-3-030-14812-6\_16}

\bibitem{Hodges:97:IGPL}
Hodges, W.: Compositional semantics for a language of imperfect information.
  Log. J. {IGPL}  \textbf{5}(4),  539--563 (1997). \doi{10.1093/jigpal/5.4.539}

\bibitem{Huang:06:UAI}
Huang, Y., Valtorta, M.: {Pearl}'s calculus of intervention is complete. In:
  Proc. {UAI}'06. p. 217^^e2^^80^^93224. AUAI Press (2006)

\bibitem{Hyttinen:14:UAI}
Hyttinen, A., Eberhardt, F., J{\"{a}}rvisalo, M.: Constraint-based causal
  discovery: Conflict resolution with answer set programming. In: Proc. the
  Thirtieth Conference on Uncertainty in Artificial Intelligence ({UAI}'14).
  pp. 340--349. {AUAI} Press (2014)

\bibitem{Hyttinen:15:UAI}
Hyttinen, A., Eberhardt, F., J{\"{a}}rvisalo, M.: Do-calculus when the true
  graph is unknown. In: Proc. the Thirty-First Conference on Uncertainty in
  Artificial Intelligence ({UAI}'15). pp. 395--404. {AUAI} Press (2015)

\bibitem{Ibeling:20:AAAI}
Ibeling, D., Icard, T.: Probabilistic reasoning across the causal hierarchy.
  In: Proc. the Thirty-Fourth {AAAI} Conference on Artificial Intelligence
  ({AAAI}'20). pp. 10170--10177. {AAAI} Press (2020),
  \url{https://aaai.org/ojs/index.php/AAAI/article/view/6577}

\bibitem{Kawamoto:19:FC}
Kawamoto, Y.: Statistical epistemic logic. In: The Art of Modelling
  Computational Systems: {A} Journey from Logic and Concurrency to Security and
  Privacy. LNCS, vol. 11760, pp. 344--362. Springer (2019).
  \doi{10.1007/978-3-030-31175-9\_20}

\bibitem{Kawamoto:19:SEFM}
Kawamoto, Y.: Towards logical specification of statistical machine learning.
  In: Proc. {SEFM}. pp. 293--311 (2019). \doi{10.1007/978-3-030-30446-1\_16}

\bibitem{Kawamoto:20:SoSyM}
Kawamoto, Y.: An epistemic approach to the formal specification of statistical
  machine learning. Software and Systems Modeling  \textbf{20}(2),  293--310
  (2020). \doi{10.1007/s10270-020-00825-2}

\bibitem{Kawamoto:07:JSIAM}
Kawamoto, Y., Mano, K., Sakurada, H., Hagiya, M.: Partial knowledge of
  functions and verification of anonymity. Transactions of the Japan Society
  for Industrial and Applied Mathematics  \textbf{17}(4),  559--576 (2007).
  \doi{10.11540/jsiamt.17.4\_559}

\bibitem{Kawamoto:21:KR}
Kawamoto, Y., Sato, T., Suenaga, K.: Formalizing statistical beliefs in
  hypothesis testing using program logic. In: Proc. {KR'21}. pp. 411--421
  (2021). \doi{10.24963/kr.2021/39}

\bibitem{Kawamoto:22:arxiv}
Kawamoto, Y., Sato, T., Suenaga, K.: Sound and relatively complete belief
  {H}oare logic for statistical hypothesis testing programs. CoRR
  \textbf{abs/2208.07074} (2022)

\bibitem{Makin:19:elife}
Makin, T.R., de~Xivry, J.J.O.: Science forum: Ten common statistical mistakes
  to watch out for when writing or reviewing a manuscript. Elife  \textbf{8},
  e48175 (2019)

\bibitem{McCain:97:AAAI}
McCain, N., Turner, H.: Causal theories of action and change. In: Proc. the
  Fourteenth National Conference on Artificial Intelligence and Ninth
  Innovative Applications of Artificial Intelligence Conference
  ({AAAI}'97/{IAAI}'97). pp. 460--465. {AAAI} Press / The {MIT} Press (1997)

\bibitem{Moher:12:CONSORT}
Moher, D., Hopewell, S., Schulz, K.F., Montori, V., G{\o}tzsche, P.C.,
  Devereaux, P., Elbourne, D., Egger, M., Altman, D.G.: Consort 2010
  explanation and elaboration: updated guidelines for reporting parallel group
  randomised trials. International journal of surgery  \textbf{10}(1),  28--55
  (2012)

\bibitem{Pearl:95:biometrika}
Pearl, J.: Causal diagrams for empirical research. Biometrika  \textbf{82}(4),
  669--688 (1995), \url{http://www.jstor.org/stable/2337329}

\bibitem{Pearl:09:causality}
Pearl, J.: Causality. Cambridge university press (2009)

\bibitem{Ruckschloss:22a:ICLP}
R{\"{u}}ckschlo{\ss}, K., Weitk{\"{a}}mper, F.: Exploiting the full power of
  {Pearl}'s causality in probabilistic logic programming. In: Proc. the 9th
  Workshop on Probabilistic Logic Programming ({PLP}'22). {CEUR} Workshop
  Proceedings, vol.~3193. CEUR-WS.org (2022),
  \url{http://ceur-ws.org/Vol-3193/paper1PLP.pdf}

\bibitem{Shpitser:06:UAI}
Shpitser, I., Pearl, J.: Identification of conditional interventional
  distributions. In: Proc. {UAI}'06. p. 437^^e2^^80^^93444. AUAI Press (2006)

\bibitem{Simposon:51:RSS}
Simpson, E.H.: The interpretation of interaction in contingency tables. Journal
  of the Royal Statistical Society. Series B (Methodological)  \textbf{13}(2),
  238--241 (1951), \url{http://www.jstor.org/stable/2984065}

\bibitem{Triantafillou:15:JMLR}
Triantafillou, S., Tsamardinos, I.: Constraint-based causal discovery from
  multiple interventions over overlapping variable sets. J. Mach. Learn. Res.
  \textbf{16},  2147--2205 (2015)

\bibitem{Verma:88:UAI}
Verma, T., Pearl, J.: Causal networks: semantics and expressiveness. In: Proc.
  {UAI}'88. pp. 69--78. North-Holland (1988)

\bibitem{von:07:STROBE}
Von~Elm, E., Altman, D.G., Egger, M., Pocock, S.J., G{\o}tzsche, P.C.,
  Vandenbroucke, J.P.: The strengthening the reporting of observational studies
  in epidemiology (strobe) statement: guidelines for reporting observational
  studies. Bulletin of the World Health Organization  \textbf{85},  867--872
  (2007)

\end{thebibliography}

\arxiv{%
\appendix
\conference{%
We present the technical details omitted in the paper:
}%
\arxiv{%
\section*{Appendix}
We present the following technical details:
}%

\begin{itemize}
\item \App{sec:appendix:models} presents the details of our models and causality.
\item \App{sec:app:StaCL:sound} proves the soundness of the deductive system \AX{} for \StaCL{}.
\item \App{sec:StaCL:ax:pred} shows the soundness of the deductive system \AXwCP{} for \StaCL{} with causal predicates.
\item \App{sec:appendix:StaCL:reasoning} presents the details of the reasoning and the explanation about statistical causality using \StaCL{}.
\end{itemize}

We first introduce notations.
We present key notations in \Tbls{tbl:key-symbols:syn} and~\ref{tbl:key-symbols:sem}.

For tuples $\bm{x}$ and $\bm{x'}$ of variables, we write $\bm{x} \subseteq \bm{x'}$ iff every variable in $\bm{x}$ appears in $\bm{x'}$.
$\bmx \setminus \bmy$ is the tuple of variables obtained by removing all variables in $\bmy$ from $\bmx$.
As with $\joint$, the symbol `$\setminus$' is a meta-operator on sets of variables, and \emph{not} a function symbol.
For brevity, we identify a singleton tuple $\la x \ra$ as its element $x$.

For $u,u'\in\Term$ and $x\in\cvar$, the \emph{substitution} $u[x \mapsto u']$ is the term obtained by replacing every occurrence of $x$ in $u$ with $u'$.

We recall that a \emph{memory} is a joint probability distribution $\memory \in \Dists(\var \rightarrow \calo\cup\{\bot\})$ of data values of all variables in $\var$.
We write 
$\memory(\bmx)$ 
for the joint distribution of all variables 
in $\bmx$.

\section{Details on Models and Causality}
\label{sec:appendix:models}

In this section, we show 
a couple of remarks on the interpretation of terms (\App{sub:app:remark:interpretation}),
the relationships between data generators and causal diagrams (\App{sub:app:relation:DG:CD}),
and properties on memories (\App{sub:appendix:basic}).
Then we present more details on
causal predicates (\App{sub:appendix:def:CP}), causal effects (\App{sub:app:causal:effect}), and causal diagrams (\App{sub:appendix:def:diagrams}).

\begin{table}[!h]
\caption{Notations in syntax.}
\label{tbl:key-symbols:syn}
\centering
\begin{tabular}{@{}ll@{}} 
\toprule
Symbol & Description
\\ \midrule
$\cvar$ & Set of causal variable \\
$\fvar$ & Set of conditional causal variables \\
$\fv{\phi}$ & Set of all free variables in a formula $\phi$ \\
$\cdv{\phi}$ & Set of all conditioning variables in a formula $\phi$ \\
\hline
$\Func$ & Set of function symbols \\
$\pFunc$ & Set of probabilistic function symbols \\
$\dFunc$ & Set of deterministic function symbols \\
\hline
$\Const$ & Set of constants \\
$\dConst$ & Set of deterministic constants \\
$\bmbot$ & Constant denoting the undefined value \\
\hline
$\CTerm$ & Set of causal terms \\
$\Term$ & Set of terms \\
\hline
$\Pred$ & Set of predicates \\
$\CPred$ & Set of causal predicates \\
$\Fml$ & Set of (causality) formulas \\
\bottomrule
\end{tabular}
\end{table}

\begin{table}[!h]
\caption{Notations in semantics.}
\label{tbl:key-symbols:sem}
\centering
\begin{tabular}{@{}l@{}l@{}} 
\toprule
Symbol & Description
\\ \midrule
$\M$ & Kripke model \\
$\calw$ & Set of all possible worlds \\
$\relExc$ & Intervention relation \\
$\calv$ & Valuation \\
$\calo$ & Domain of data values \\
$\Dists \calo$ & Set of all probability distributions over $\calo$ \\
\hline
$w$ & Possible world \\
$w\intvE{\subst{c}{x}}$ & Eagerly intervened world \\
$w\intvL{\subst{c}{x}}$ & Lazily intervened world \\
$\datagen{w}$ & Data generator in a world $w$ \\
$x \precg y$ & $y$'s value depends on $x$'s in a data generator $\dgen$ \\
$\semf_{w} \,{=}\, \{ \semf_{w}^{\rand} \}_{\rand \sim I}$ & Interpretation of function symbols in a world $w$ \\
$\memory_{w}$ & Memory on variables in a world $w$ \\
\hline
$\diag{w} \,{=}\, (U, V, E)$ ~& Causal diagram in a world $w$ \\
$\pa_{\diag{w}}(v)$ & All parent variables of $v$ in $\diag{w}$ \\
$P_{\diag{w}}(V)$ & Joint distribution of all variables $V$ in $\diag{w}$ \\
\bottomrule
\end{tabular}
\end{table}

\subsection{Remarks on the Interpretation of Terms}
\label{sub:app:remark:interpretation}

We present remarks on the \emph{at-most-once condition}, on deterministic functions, and on the well-definedness of the interpretation of terms.

\subsubsection{Remark on At-Most-Once Condition}

We remark on the at-most-once condition.
In \Sec{sec:data},
we assumed that a data generator satisfies the following \emph{at-most-once} condition: 
Each function symbol $f$ and each constant $c$ 
can be used at most once in a single data generator.
For example, we may consider the data generator $\dgen_3$ defined by
$c \garrow_{\dgen_3} z$ and $f(z,z) \garrow_{\dgen_1} y$ as it is.
Then $\dgen_3$ rewrites $y$ into $f(c, c)$ after substitutions.
In contrast, the data generator $\dgen_4$ defined by
\begin{align}
\label{eq:at-most-once-no}
c \garrow_{\dgen_4} z_1,~c \garrow_{\dgen_4} z_2,~f(z_1,z_2) \garrow_{\dgen_4} y
\end{align}
also rewrites $y$ into $f(c, c)$, but does not satisfy the at-most-once condition.
Thus, the two calls of the constant $c$ should be distinguished and replaced with two symbols $c_1$ and $c_2$
(denoting the same distribution of data values as $c$):
\begin{align}
\label{eq:at-most-once-yes}
c_1 \garrow_{\dgen_4} z_1,~c_2 \garrow_{\dgen_4} z_2,~f(z_1,z_2) \garrow_{\dgen_4} y.
\end{align}
Then $\dgen_4$ rewrites $y$ into $f(c_1, c_2)$.
This at-most-once condition clarifies that an occurrence of a probabilistic constant $c_i$ represents a single independent sampling.
In the former definition~\eqref{eq:at-most-once-no} of $g_4$, which does not satisfy the at-most-once condition, it is not clear whether (i) $c$ is sampled once and its (single) value is assigned to both $y_1$ and $y_2$, or (ii) it is sampled twice and the drawn 
 (two) values are assigned to $y_1$ and $y_2$.
By imposing the at-most-once condition as in the latter definition~\eqref{eq:at-most-once-yes}, we can clarify that there are two occurrences of sampling that may use different randomness.

\subsubsection{Remark on Interpretation of Terms}
We remark that two copies of the same data value are obtained from the distribution $\sema{\la c, c \ra}{w}$, whereas two different values may be drawn from $\sema{\la c_1, c_2 \ra}{w}$ for constants $c_1, c_2$ denoting the same distribution, i.e., $\sema{c_1}{w} = \sema{c_2}{w}$.
Indeed, for a randomly chosen $\rand$,
the former results in $\semrxg{\la c, c \ra} = (\semf^\rand(c), \allowbreak \semf^\rand(c))$.
In contrast, the latter results in $\semrxg{\la c_1, c_2 \ra} = (\semf^\rand(c_1), \allowbreak \semf^\rand(c_2))$,
where two data values are sampled independently.
Notice that this definition is consistent with the at-most-once condition on a data generator $\dgen$.

\subsubsection{Remark on Deterministic Functions}
As a remark, we consider probabilistic and deterministic functions.
Let $\pFunc \subseteq \Func$ be the set of all \emph{probabilistic function symbols}, 
each denoting a randomized algorithm that produces a data value using a randomness drawn internally
(e.g., a function that returns a number obtained by adding a random number to an input).
Let $\dFunc \eqdef \Func \setminus \pFunc$ be the set of all \emph{deterministic function symbols}, denoting deterministic algorithms (e.g., $+$ and $-$).

If $f \in \dFunc$, the interpretation of $f$ is independent of the randomness;
i.e., $\semf^{\rand}(f) = \semf^{\rand'}(f)$ for all $\rand, \rand' \in \cali$,
hence $\semf(f)$ maps $k$ data values $\bm{o}$ to the Dirac distribution $\delta_{\semf^{\rand}(f)(\bm{o})}$, having a single value $\semf^{\rand}(f)(\bm{o})$ with probability $1$.

We may relax the at-most-once condition in \Sec{sec:data} so that deterministic function symbols do not have to satisfy the condition.
This is because the interpretation of a deterministic function symbol $f$ is the same in every occurrence of~$f$.

\subsubsection{Well-definedness of the Interpretation}

We next show that the interpretation $\sema{\_}{w}$ of terms in a world $w$ is unique thanks to the assumption on the strict partial order $\prec$ over the defined causal variables $\dom{\datagen{w}}$ as follows.

\setcounter{prop}{2} 
\begin{restatable}[Well-definedness of $\sema{\_}{w}$]{prop}{PropUnique}
\label{prop:unique:interpretation}
Let $u$ be a term and $w = (\semf, \datagen{}, \mem{})$ be a world such that $\datagen{}$ is closed.
Then we have $\sema{u}{w} \in \Dists\calo^k$.
\end{restatable}

\begin{proof}
Since $\datagen{}$ is closed, we have
$\fv{\range(\dgen)} \subseteq \dom{\dgen}$.
By the definition of $\sema{u}{w}$ 
in \Sec{sec:model},
it suffices to show $\semrxg{u} \in \calo^k$ for each $r \in \cali$.

By the definition of the possible world $w$,\, $\datagen{}$ is finite and acyclic.
Then, 
$\precg$ is the strict partial order over $\dom{\datagen{}}$ defined in 
\Sec{sec:data}.
For a tuple $\bmx = \la x_1, \ldots, x_k \ra$ of variables, let $\cntv{\bmx}$ be the number of variables $z$ such that $z \precg x_i$ for some $i = 1, \ldots, k$.
Since $\dom{\datagen{}}$ is a finite set, $\cntv{\bmx}$ is finite.

Let $r \in \cali$.
We first show that for any $\bmx = \la x_1, \ldots, x_k \ra \in \dom{\datagen{}}^{k}$, 
we have $\semrxg{\bmx} \,{\in}\, \calo^k$ by induction on $\cntv{\bmx}$.
\begin{itemize}
\item Case $\cntv{\bmx} = 0$.
By definition, there is no variable $z$ such that $z \precgw x_i$ for any $i = 1, \ldots, k$;
hence $\fv{\datagen{}(\bmx)} = \emptyset$.
Since $\datagen{}$ is closed,
for each $i = 1, \ldots, k$,
$\sema{\datagen{}(x_i)}{w}$ is defined and represented as
$\sema{c_i}{w}$ or 
$\sema{f_i(c_{i1}, \ldots, c_{il})}{w}$ 
for constants $c_i, c_{i1}, \ldots, c_{il}$.
Therefore,
\[
\semrxg{\bmx} = \semrxg{\la \datagen{}(x_1), \ldots, \datagen{}(x_k) \ra} 
\in \calo^k
\]
follows immediately from the definition of $\semf^{\rand}$.

\item Case $\cntv{\bmx} > 0$.
Let $\bmz \eqdef \la z_1,\ldots, z_l\ra \subseteq \fv{\datagen{}(\bmx)}$ for $l \ge 1$.
From the strict partial order structure of $\precg$, we have $\cntv{\bmz} < \cntv{\bmx}$.
By induction hypothesis, $\semrxg{\bmz} \in \calo^l$.
Then, by the definition of $\fv{\datagen{}(\bmx)}$,
for each $i = 1,2,\ldots, l$, 
$\datagen{}(x_i)$ can be represented as
a constant $c_i$ or a causal term $f_i(v_1,\ldots v_l)$ where $v_j \in \cvar \cup \Const$ for each $j = 1,2,\ldots,l$.
In the former case, 
$\semrxg{\datagen{}(x_i)} = \semrxg{c_i} = \semf^{\rand}(c_i) \in \calo$ by definition.
In the latter case, $\semrxg{\datagen{}(x_i)} = \semrxg{f_i(v_1,\ldots v_l)}$.
If $v_i \in \Const$ then $\semrxg{v_i} \in \calo$ is immediate; 
if $v_j \in \cvar$ then we obtain
$v_j \in \fv{\datagen{}(\bmx)}$, and hence 
 $\semrxg{v_j} \in \calo$.
Therefore, we conclude:
\begin{align*}
\semrxg{f_i(v_1, \ldots, v_l)} 
&= \semf^{\rand}(f_i)(\semrxg{\la v_1, \ldots, v_l \ra})\\
&= \semf^{\rand}(f_i)(\semrxg{v_1}, \cdots, \semrxg{v_l}) \in \calo.
\end{align*}
\end{itemize}

The rest of the proof is immediately by induction on the structures of tuples of terms.
\myqed
\end{proof}

\subsection{Relationships Between Data Generators and Causal Diagrams}
\label{sub:app:relation:DG:CD}

We show that each data generator corresponds to a DAG (directed acyclic graph) of the causal model as follows.

\begin{restatable}[Acyclicity of $\diag{}$]{prop}{PropAcyclicDG}
\label{prop:acyclic:DG}
Let $\diag{}$ be the causal diagram corresponding to a finite and acyclic data generator $\datagen{}$.
Then $\diag{}$ is a finite directed acyclic graph.
\end{restatable}%
\begin{proof}
Let $\diag{} = (U, V, E)$.
Since $\datagen{}$ is finite, $\dom{\datagen{}}$ and $\range(\datagen{})$ are finite.
By definition,
for each $x \garrow y \in E$, we have $x \precg y$.
Since $\precg$ is a strict partial order, so is $\garrow$.
Therefore, $\diag{}$ is a finite directed acyclic graph.
\myqed
\end{proof}

Next, we show that the interpretation $\sema{\_}{w}$ defines the joint distribution  $P_{\diag{w}}$ of all variables in the causal diagram $\diag{w}$.

\begin{restatable}[Relationship between $\sema{\_}{w}$ and $P_{\diag{w}}$]{prop}{PropGequivDG}
\label{prop:G:equiv:DG}
Let $w$ be a world, $\diag{w}$ be the causal diagram corresponding to the data generator $\datagen{w}$, $\bmx, \bmy \in \cvar^+$, $\bmz \in \cvar^*$, and $\bmc\in\dConst^+$.
Then there is a joint distribution $P_{\diag{w}}$ that factorizes according to $\diag{w}$, and that satisfies:
\begin{enumerate}
\item $\sema{\bmy|_{\bmz}}{w} = P_{\diag{w}}(\bmy \,|\, \bmz)$
\item $\sema{\bmy|_{\bmz}}{w \intvE{\subst{\bmc}{\bmx}}} = P_{\diag{w}}(\bmy \,|\, do(\bmx \,{=}\, \sema{\bmc}{w}),\, \bmz)$.
\end{enumerate}
\end{restatable}%

\begin{proof}
We fix a world $w$.
For brevity, we write $\diag{} \eqdef \diag{w}$.
We recall that $\diag{}$ is the triple $(U, V, E)$ consisting of
$U \eqdef \fnc{\range(\datagen{w})}$, 
$V \eqdef \dom{\datagen{w}}$, and 
the set $E$ of structural equations defined by $\datagen{w}$.
Let $P_{\diag{}}(V)$ be the probability distribution defined in \eqref{eq:Pg1}.

We inductively define 
the set $\Leaf^k$ of variables of depth $k$ from leaves and 
its subset $\LeafF^k$ having a parent variable by:
\begin{align*}
V^0 &= V,\\
E^0 &= E,\\
\Leaf^{k} &= \{v \in V^{k} ~|~ \mbox{ for all } v' \in V^{k},~v \garrow v' \notin E^{k} \},\\
\LeafF^{k} &= \{ v \in \Leaf^{k} \mid \mbox{ there is a } v'' \in V^{k},~v'' \garrow v \notin E^{k} \}, \\
V^{k+1} &= V^{k} \setminus \LeafF^k,\\
E^{k+1} &= E^{k} \cap ((U \cup V^{k+1}) \times V^{k+1}).
\end{align*}
This procedure terminates when $\LeafF^{k} = \emptyset$.
Since $\datagen{w}$ is finite and acyclic, the above procedure terminates in a finite number $N$ of steps.

Notice that, by the definition of causal terms,
for any $k < N$ and
and any $v \in V^{k}$, $\datagen{w}(v)$ is of the form $f(z_1, \ldots, z_l)$ for some $f \in \Func$ and $z_1, \ldots, z_l \in \cvar$.

For each $k=0, 1. \ldots, N$, 
we consider the restricted diagram $\diag{}^{k} = (U, V^{k}, E^{k})$.

We first claim that $\sema{V^{k}}{w} = P_{\diag{}^{k}}(V^{k})$ for any $k=0, 1, \ldots, N$.
We prove this by induction on $k$ as follows.

The base case is $k = N$.
Since no variable in $V^{N}$ has a parent variable,
there is a $\bmc \in \Const^+$ such that $\datagen{w}(V^{N}) = \bmc$.
Hence,
\[
P_{\diag{}}(V^{N}) = \sema{\bmc}{w} = \sema{\datagen{w}(V^{N})}{w} = \sema{V^{N}}{w}.
\]

Next, we consider the case $k < N$.
Since the above procedure terminates exactly at $N$ steps, we obtain $\LeafF^{k} \neq \emptyset$.
By applying the induction hypothesis (the case of $k + 1$), we obtain 
$P_{\diag{}^{k+1}}(V^{k} \setminus \LeafF^{k})
= P_{\diag{}^{k+1}}(V^{k + 1})
= \sema{V^{k + 1}}{w}
= \sema{V^{k} \setminus \LeafF^{k}}{w}$.
Since each causal term is of depth at most $1$ by definition,
the set $\pa_{\diag{}^{k}}(v)$ of all parents of a variable $v$ in the causal diagram $\diag{}^{k}$ is given by $\fv{\datagen{w}(v)}$.
Hence, 
\begin{align*}
P_{\diag{}^{k}}(V^{k})
&= 
P_{\diag{}^{k+1}}( \LeafF^{k} ) \cdot
P_{\diag{}^{k+1}}( V^{k} \setminus \LeafF^{k} )
\\
&= 
\Big(\! \prod_{v \in \LeafF^{k}}\! P_{\diag{}^{k}}(v \,|\, \pa_{\diag{}^{k}}(v)) \Big) \cdot 
P_{\diag{}^{k+1}}( V^{k} \setminus \LeafF^{k} )
\\
&= 
\Big(\! \prod_{v \in \LeafF^{k}}\! P_{\diag{}^{k}}(v \,|\, \pa_{\diag{}^{k}}(v)) \Big) \cdot 
\sema{ V^{k} \setminus \LeafF^{k} }{w}
\\
&\qquad \text{ (by induction hypothesis) }
\\
&=
\sema{V^{k}}{w}
{.}
\\
&\qquad \text{ (since each causal term is of depth at most $1$) }
\end{align*}
Therefore, we conclude:
\[
P_{\diag{}}(V) = P_{\diag{}^{0}}(V^{0}) = \sema{V^{0}}{w} = \sema{V}{w}.
\]

Now the first equation in the proposition is obtained as follows.
\begin{align*}
\sema{\bmy|_{\bmz}}{w} 
&= {\textstyle\frac{\sema{\bmy\joint\bmz}{w}}{\sema{\bmz}{w}}}
\\
&= {\textstyle \frac{ ( \sema{V}{w} )|_{\bmy\joint\bmz} }{\, ( \sema{V}{w} )|_{\bmz}~~~ }}
\\
&= {\textstyle \frac{ ( P_{\diag{w}}(V)) |_{\bmy\joint\bmz} }{\, ( P_{\diag{w}}(V) )|_{\bmz}~~~ }}
\\
&= {\textstyle\frac{P_{\diag{w}}(\bmy\joint\bmz)}{P_{\diag{w}}(\bmz)}}
\\
&= P_{\diag{w}}(\bmy \,|\, \bmz).
\end{align*}

Similarly, the second equation in the proposition is obtained as follows.
Let $\diag{w}'$ be the causal diagram obtained by an intervention $\bmx :=
 \sema{\bmc}{w}$ in $\diag{w}$.
\begin{align*}
\sema{\bmy|_{\bmz}}{w \intvE{\subst{\bmc}{\bmx}}}
&= {\textstyle\frac{\sema{\bmy\joint\bmz}{w \intvE{\subst{\bmc}{\bmx}}}}{\sema{\bmz}{w \intvE{\subst{\bmc}{\bmx}}}}}
\\
&= {\textstyle \frac{ ( \sema{V}{w \intvE{\subst{\bmc}{\bmx}}} )|_{\bmy\joint\bmz} }{\, ( \sema{V}{w \intvE{\subst{\bmc}{\bmx}}} )|_{\bmz}~~~ }}
\\
&= {\textstyle \frac{ ( P_{\diag{w}'}(V)) |_{\bmy\joint\bmz} }{\, ( P_{\diag{w}'}(V) )|_{\bmz}~~~ }}
\\
&= {\textstyle\frac{P_{\diag{w}}(\bmy\joint\bmz \,|\, do(\bmx \,{=}\, \sema{\bmc}{w}))}{P_{\diag{w}}(\bmz \,|\, do(\bmx \,{=}\, \sema{\bmc}{w}))}}
\\
&= P_{\diag{w}}(\bmy \,|\, do(\bmx \,{=}\, \sema{\bmc}{w}),\, \bmz).
\end{align*}
\myqed
\end{proof}

\subsection{Properties on Memories}
\label{sub:appendix:basic}

We present basic properties on memories.
Intuitively, we show that:
\begin{enumerate}\renewcommand{\labelenumi}{(\roman{enumi})}
\item
the same formulas $\phi_{\bmy}$ with free variables $\bmy$ are satisfied in any worlds $w$ and $w'$ having the same memory on $\bmy$;
\item
an eager and a lazy interventions to $\bmx$ result in the same distribution of the variables $\bmy$ disjoint from~$\bmx$;
\item
disjoint worlds $w$ and $w'$ have  data generators $\datagen{w}$ and $\datagen{w'}$ with disjoint domains.
\end{enumerate}

\begin{restatable}[Properties on $\mem{w}$]{prop}{PropDistFml}
\label{prop:dist-formula}
Let $w$ and $w'$ be worlds, $\bmx,\bmy \,{\in}\, \var^+$, $\bmc \,{\in}\, \dConst^+$, and 
$\phi_{\bmy} \,{\in}\, \Fml$
with $\fv{\phi_{\bmy}} = \bmy$.
\begin{enumerate}\renewcommand{\labelenumi}{(\roman{enumi})}
\item
If $\mem{w}(\bmy) = \mem{w'}(\bmy) \neq \bot$, then
$w\models \phi_{\bmy}$ iff $w'\models \phi_{\bmy}$.
\item
If $\bmx\cap\bmy = \emptyset$, then
$\mem{w\intvE{\subst{\bmc}{\bmx}}}(\bmy)
 = \mem{w\intvL{\subst{\bmc}{\bmx}}}(\bmy)$.
\item
If $\var(w) \cap \var(w') = \emptyset$ (namely, $\dom{\mem{w}} \allowbreak \cap \dom{\mem{w'}} \allowbreak = \emptyset$), then
$\dom{\datagen{w}} \cap \dom{\datagen{w'}} = \emptyset$.
\end{enumerate}
\end{restatable}

\begin{proof}
\begin{enumerate}\renewcommand{\labelenumi}{(\roman{enumi})}
\item
Assume that $\mem{w}(\bmy) = \mem{w'}(\bmy) \neq \bot$.
Then we prove $w\models \phi_{\bmy}$ iff $w'\models \phi_{\bmy}$ by induction on $\phi_{\bmy}$.

If $ \phi_{\bmy} \eqdef \eta(\bmy)$ for $\eta\in\Pred$, then:
\begin{align*}
&\mbox{ \phantom{iff} }~
w\models \eta(\bmy)
\\ &\mbox{ iff }~
\mem{w}(\bmy) \in \calv(\eta)
\\ &\mbox{ iff }~
\mem{w'}(\bmy) \in \calv(\eta)
& \mbox{(by $\mem{w}(\bmy) = \mem{w'}(\bmy) \neq \bot$)}
\\ &\mbox{ iff }~
w'\models \eta(\bmy).
\end{align*}
The other cases are straightforward by definitions.
\item
Assume $\bmx\cap\bmy = \emptyset$.
By definition, we obtain:
\begin{align*}
\mem{w\intvE{\subst{\bmc}{\bmx}}}(\bmy)
&=
\sema{\datagen{w\intvE{\subst{\bmc}{\bmx}}}(\bmy)}{w\intvE{\subst{\bmc}{\bmx}}}
\\&=
\sema{\datagen{w}(\bmy)}{w\intvE{\subst{\bmc}{\bmx}}}
& \!\text{(by $\bmx\cap\bmy = \emptyset$)}
\\&=
\sema{\datagen{w}(\bmy)[\bmx \mapsto \bmc]}{w\intvE{\subst{\bmc}{\bmx}}}
\\&=
\sema{\datagen{w}(\bmy)[\bmx \mapsto \bmc]}{w\intvL{\subst{\bmc}{\bmx}}}
&
\text{(by $\bmx \cap \fv{\datagen{w}(\bmy)[\bmx \mapsto \bmc]} = \emptyset$)}
\\&=
\sema{\datagen{w\intvL{\subst{\bmc}{\bmx}}}(\bmy)}{w\intvL{\subst{\bmc}{\bmx}}}
\\&=
\mem{w\intvL{\subst{\bmc}{\bmx}}}(\bmy).
\end{align*}
\item
By the definition of possible worlds, we have
$\dom{\datagen{w}} \subseteq \dom{\mem{w}}$ and
$\dom{\datagen{w'}} \subseteq \dom{\mem{w'}}$.
Therefore, we obtain
$\dom{\datagen{w}} \cap \dom{\datagen{w'}} \subseteq 
\dom{\mem{w}} \allowbreak \cap \dom{\mem{w'}} =~\emptyset$.
\end{enumerate}
\myqed
\end{proof}

\subsection{Causal Predicates}
\label{sub:appendix:def:CP}

Next, we present more details on causal predicates.
Among the causal predicates listed below, our deduction system \AXwCP{} requires only $\dsep$, $\nanc$, and $\allnanc$.
For the sake of convenience, we can use $\pa$, but it is sufficient for us to derive the formulas equipped with $\pa$ from those with $\nanc$.
Thus, we do not deal with axioms of the other predicates in this paper.

We show a list of causal predicates as follows.
\begin{itemize}
\item\!%
$\dsep(\bmx, \bmy, \bmz)$ ~~
$\bmx$ and $\bmy$ are $d$-separated by $\bmz$;
\item\!%
$\pa(\bmx, \bmy)$ ~~
$\bmx$ is the set of all parents of variables in $\bmy$;
\item\!%
$\npa(\bmx, \bmy)$ ~~
$\bmx$ is a set of non-parents of variables in $\bmy$;
\item\!%
$\anc(\bmx, \bmy)$ ~
$\bmx$ is the set of all ancestors of variables in $\bmy$;
\item\!%
$\nanc(\bmx, \bmy)$\,
$\bmx$ is a set of non-ancestors of variables in $\bmy$;
\item\!%
$\allnanc(\bmx, \bmy, \bmz)$\,
$\bmx$ is the set of all variables in $\bmy$ that are not ancestors of any variables in $\bmz$.
\end{itemize}
These causal predicates are interpreted using a data generator $\datagen{w}$ in a world $w$ as follows.
\begin{definition}[Semantics of $\dsep$, $\pa$, $\npa$, $\anc$, $\nanc$, \allowbreak $\allnanc$] \rm \label{def:sem:pa-npa}
Let $w$ be a world, and $\diag{w}$ be the causal diagram corresponding to $\datagen{w}$.
Let $\Pa(\bmy)$ be the set of all parent variables of $\bmy$:
\begin{align*}
\Pa(\bmy) = 
\{ x \in \var(w) \mid x \garrow y',~ y' \in \bmy \}.
\end{align*}
Let $\Anc(\bmy)$ is the set of all ancestors variables of $\bmy$:
\begin{align*}
\Anc(\bmy) = 
\{ x \in \var(w) \mid x \garrow^{\!+} y',~ y' \in \bmy \}
\end{align*}
where $\garrow^{\!+}$ is the transitive closure of $\garrow$.
The interpretations of the causal predicates $\dsep, \pa, \npa, \anc, \nanc, \allnanc$ are given as follows:
\begin{align*}
w \models \dsep(\bmx, \bmy, \bmz)
& \mbox{ iff }
\mbox{ $\bmx$ and $\bmy$ are $d$-separated by $\bmz$ in $\diag{w}$}
\\
w \models \pa(\bmx, \bmy)
& \mbox{ iff }
\bmx = \Pa(\bmy)
\mbox{ and } \bmx\cap\bmy=\emptyset
\\
w \models \npa(\bmx, \bmy)
& \mbox{ iff }
\bmx \,{\cap}\, \Pa(\bmy) = \emptyset
\mbox{ and } \bmx \,{\cap}\, \bmy=\emptyset
\\
w \models \anc(\bmx, \bmy)
& \mbox{ iff }
\bmx = \Anc(\bmy)
\mbox{ and } \bmx\cap\bmy=\emptyset
\\
w \models \nanc(\bmx, \bmy)
& \mbox{ iff }
\bmx \,{\cap}\, \Anc(\bmy) = \emptyset
\mbox{ and } \bmx \,{\cap}\, \bmy = \emptyset
\\
w \models \allnanc(\bmx, \bmy, \bmz)
& \mbox{ iff }
\bmx = \bmy \setminus \Anc(\bmz)
{,}
\end{align*}
where we recall the notion of $d$-separation in \App{sub:appendix:def:diagrams}.
\end{definition}

\begin{restatable}[Relationships among causal predicates]{prop}{PropGraph}
\label{prop:relation:c-predicates}
The causal predicates $\pa$, $\npa$, $\anc$, and $\nanc$ 
satisfy the relationships:
\begin{enumerate}
\item
$\models
\pa(\bmx, \bmy) \rightarrow \anc(\bmx, \bmy)$.
\item
$\models
\anc(\bmx, \bmy) \rightarrow \nanc(\bmy, \bmx)$.
\item
$\models
\nanc(\bmy, \bmx) \rightarrow \npa(\bmy, \bmx)$.
\end{enumerate}
\end{restatable}

\begin{proof}
These claim are straightforward from \Def{def:sem:pa-npa}.
\myqed
\end{proof}

\subsection{Causal Effect}
\label{sub:app:causal:effect}

We show that the causal effect can be expressed using a \StaCL{} formula as follows.

\arxiv{%
\PropCausal*
\vspace{1ex}
}%
\conference{%
\setcounter{prop}{0}
\begin{restatable}[Causal effect]{prop}{PropCausalAPP}
Let $w$ be a world, $\bmx,\allowbreak\bmy,\bmz\in\var(w)^+$ be disjoint, 
$\bmc \,{\in}\, \dConst^+$, 
$\bmc' \,{\in}\, \Const^+$, and $f \,{\in}\, \Func$.
Then: 
\begin{enumerate}\renewcommand{\labelenumi}{(\roman{enumi})}
\item
$w \models \intvE{\subst{\bmc}{\bmx}} (\bmc' \,{=}\, \bmy)$
iff 
there is a distribution $P_{\diag{w}}$ that is factorized according to $\diag{w}$ and satisfies 
$P_{\diag{w}}( \bmy \,|\, do(\bmx \,{=}\, \bmc) ) \,{=}\, \sema{\bmc'}{w}$.
\item
$w \models \intvE{\subst{\bmc}{\bmx}} (f \,{=}\, \bmy|_{\bmz})$
iff 
there is a distribution $P_{\diag{w}}$ that is factorized according to $\diag{w}$ and satisfies 
$P_{\diag{w}}( \bmy \,|\, do(\bmx \,{=}\, \bmc), \bmz ) \,{=}\, \sema{f}{w}$.
\end{enumerate}
\end{restatable}%
\setcounter{prop}{7}
}%

\begin{proof}
We show the first claim as follows.
By \Propo{prop:G:equiv:DG}, 
there is a joint distribution $P_{\diag{w}}$ that is factorized according to $\diag{w}$ and that satisfies 
$\sema{\bmy}{w \intvE{\subst{\bmc}{\bmx}}} = P_{\diag{w}}(\bmy \,|\, do(\bmx\,{=}\,\sema{\bmc}{w}) )$.
Thus, we obtain:
\begin{align*}
& \phantom{\mbox{ iff }~}
w \models \intvE{\subst{\bmc}{\bmx}} (\bmc' \,{=}\, \bmy)
\\ & \mbox{ iff }~
w \intvE{\subst{\bmc}{\bmx}} \models \bmc' \,{=}\, \bmy
\\ & \mbox{ iff }~
\sema{\bmy}{w \intvE{\subst{\bmc}{\bmx}}} = \sema{\bmc'}{w}
\\ & \mbox{ iff }~
P_{\diag{w}}( \bmy \,|\, do(\bmx\,{=}\,\sema{\bmc}{w}) ) \,{=}\, \sema{\bmc'}{w}.
\end{align*}

Analogously, the second claim is obtained as follows.
By \Propo{prop:G:equiv:DG}, 
there is a joint distribution $P_{\diag{w}}$ that is factorized according to $\diag{w}$ and that satisfies 
$\sema{\bmy|_{\bmz}}{w \intvE{\subst{\bmc}{\bmx}}} = P_{\diag{w}}(\bmy \,|\, do(\bmx\,{=}\,\sema{\bmc}{w}),\, \bmz)$.
Thus, we obtain:
\begin{align*}
& \phantom{\mbox{ iff }~}
w \models \intvE{\subst{\bmc}{\bmx}} (\bmc' \,{=}\, \bmy|_{\bmz})
\\ & \mbox{ iff }~
w \intvE{\subst{\bmc}{\bmx}} \models \bmc' \,{=}\, \bmy|_{\bmz}
\\ & \mbox{ iff }~
\sema{\bmy|_{\bmz}}{w \intvE{\subst{\bmc}{\bmx}}} = \sema{\bmc'}{w}
\\ & \mbox{ iff }~
P_{\diag{w}}( \bmy \,|\, do(\bmx\,{=}\,\bmc), \bmz ) \,{=}\, \sema{\bmc'}{w}.
\end{align*}
\myqed
\end{proof}

\subsection{Details on Causal Diagrams}
\label{sub:appendix:def:diagrams}

Next, we recall the notion of $d$-separation~\cite{Verma:88:UAI} as follows.

\begin{definition}[$d$-separation]\label{def:d-separate}\rm
Let $\bmx, \bmy, \bmz$ be disjoint sets of variables,
and $\Anca(\bmz) \eqdef \bmz \cup \Anc(\bmz)$ be the union of $\bmz$ and the set of $\bmz$'s all ancestors.
An undirected path $p$ is said to be \emph{$d$-separated} by $\bmz$ if it satisfies one of the following conditions:
\begin{enumerate}\renewcommand{\labelenumi}{(\alph{enumi})}
\item
$p$ has a chain $v' \garrow v \garrow v''$ s.t. $v\in\bmz$.
\item 
$p$ has a fork $v' \leftgarrow v \garrow v''$ s.t. $v\in\bmz$.
\item 
$p$ has a collider $v' \garrow v \leftgarrow v''$ s.t. $v\not\in\Anca(\bmz)$.
\end{enumerate}
$\bmx$ and $\bmy$ are \emph{$d$-separated} by $\bmz$ if all undirected paths between variables in $\bmx$ and in $\bmz$ are $d$-separated by $\bmz$.
\end{definition}

We also recall the notion of back-door path as follows.

\begin{definition}[Back-door path]\label{def:back-path}\rm
For variables $x$ and $y$,
a \emph{back-door path} from $x$ to $y$ in a causal diagram $G$ is an arbitrary undirected path between $x$ and $y$ in $G$ that starts with an arrow pointing to $x$
(i.e., an undirected path of the form $x \leftgarrow v \cdots y$).
For tuples of variables $\bmx$ and $\bmy$, a \emph{back-door path} from $\bmx$ to $\bmy$ is an arbitrary back-door path from $x\in\bmx$ to $y\in\bmy$.
\end{definition}

We remark on the relationships between back-door paths and two kinds of interventions as follows.

\begin{remark}\label{rem:backdoor-path}
An eager intervention $\intvE{\subst{\bmc}{\bmx}}$ can remove \emph{all back-door paths from $\bmx$ to $\bmy$}, because all of these paths have arrows pointing to $\bmx$ and $\intvE{\subst{\bmc}{\bmx}}$ removes all such arrows.

In contrast, a lazy intervention $\intvL{\subst{\bmc}{\bmx}}$ can remove all undirected paths between $x$ and $y$ \emph{except for all back-door paths from $x$ to $y$},
because $\intvL{\subst{\bmc}{\bmx}}$ removes all arrows emerging from $\bmx$ while keeping all arrows pointing to $\bmx$.
Thus, $\intvL{\subst{\bmc}{\bmx}}\dsep(\bmx, \bmy, \bmz)$ represents that
all back-door paths from $\bmx$ to $\bmy$ are $d$-separated by $\bmz$.
\end{remark}

These relationships are used to reason about the causality, e.g., when we derive the second rule of Pearl's do-calculus using our \StaCL{} (\Propo{prop:do-calculus}).

Now we recall the back-door criteria and the back-door adjustment in Pearl's causal model.

\begin{definition}[Back-door criterion]\label{def:back-criterion}\rm
For two sets $\bmx$ and $\bmy$ of variables, a set $\bmz$ of variables satisfies the \emph{back-door criterion} in a causal diagram $\diag{}$ if (i) no variable in $\bmz$ is a descendent of an element of $\bmx$ in $\diag{}$ and (ii) all back-door paths from $\bmx$ to $\bmy$ are $d$-separated by $\bmz$ in~$\diag{}$.
\end{definition}

The back-door criterion is expressed as the following \StaCL{} formula:
\begin{align*}
\nanc(\bmx,\bmz) \land \intvL{\subst{\bmc}{\bmx}}\dsep(\bmx, \bmy, \bmz).
\end{align*}

When $\bmz$ satisfies the back-door criterion in a causal diagram $\diag{}$, then the causal effect of $\bmx$ on $\bmy$ is given by:
\begin{align*}
P_{\diag{}}(\bmy \,|\, \mydo(\bmx)) =
\sum_{\bmz} P_{\diag{}}(\bmy \,|\, \bmz, \bmx)\, P_{\diag{}}(\bmz).
\end{align*}

In \Fig{fig:overview:proof2:drug} in \Sec{sec:reasoning-StaCL},
we show a derivation tree for the correctness of computing the causal effect using the backdoor adjustment.

\section{Proof for the Soundness of \AX{}}
\label{sec:app:StaCL:sound}

We show that the deductive system \AX{} of \StaCL{} is sound w.r.t. the Kripke semantics for statistical causality (\Thm{thm:sound:complete}).

We first remark that \AX{} satisfies the deduction theorem.

\begin{restatable}[Deduction]{prop}{PropDeductionThm}
\label{prop:deduction}
Let $\Gamma \subseteq \Fml$, and $\phi_1, \phi_2 \in \Fml$.
Then $\Gamma \vdashg \phi_1 \rightarrow \phi_2$ iff 
$\Gamma, \phi_1 \vdashg \phi_2$.
\end{restatable}%
\begin{proof}
The direction from left to right is straightforward by the application of \axMP{}.
The other direction is shown as usual by induction on the derivation.
\myqed
\end{proof}

We prove the soundness of \AX{} as follows.
We show the validity of the axioms for basic constructs (\App{sub:app:ax:basic}),
for eager interventions (\App{sub:app:intvE}),
for lazy interventions (\App{sub:app:lazy:intv}),
and for the exchanges of eager/lazy interventions (\App{sub:app:two:intv}).

\subsection{Validity of the Basic Axioms}
\label{sub:app:ax:basic}

Here are the basic axioms of \AX{} without interventions.
\[
\renewcommand{\arraystretch}{1.3}
\begin{array}{l@{\hspace{-0.3ex}}l}
\mbox{\axPT}
&~~\vdashg
\phi ~~\mbox{ for a propositional tautology $\phi$}
\\
\mbox{\axMP}
&~~~ \mbox{$\phi_1,\, \phi_1 \rightarrow \phi_2 \vdashg\phi_2$}
\\
\mbox{\axEqA}
&~~~ \vdashg \bmx = \bmx
\\
\mbox{\axEqB}
&~~~ \vdashg \bmx = \bmy \rightarrow (\phi_1 \rightarrow \phi_2)
\\[-0.4ex]
&~~~ \mbox{ where $\phi_2$ is the formula obtained by replacing}
\\[-0.8ex]
&~~~ \mbox{ any number of occurrences of $\bmx$ in $\phi_1$ with $\bmy$}
\\
\mbox{\axEqName}
&~~~ \vdashg c^{(\datagenerator, \bmx)} = \bmx
\\
\mbox{\axEqFunc}
&~~~ \vdashg f^{(\datagenerator, \bmy|_{\bmz\!,\bmx=\bmc})} = \bmy|_{\bmz\!,\bmx=\bmc}
\\
\mbox{\axPD}
&~~\vdashg
\!(
	\pos(\bmx)
	\land
	\nA {=} \bmx
	\land
	f {=} \bmy|_{\bmx}
	\land
	\nB {=} \bmx \,{\joint}\, \bmy
)
\rightarrow
\nB {=} f(\nA)
\\
\mbox{\axMPD{}}
&~~\vdashg
\bmxB\mrgn{\bmxC} = \bmxC
~~\mbox{ if } \bmxC \subseteq \bmxB
\end{array}
\]

The validity of the rules
\axPT{}, \axMP{}, \axEqA{}, \axEqB{}, 
is straightforward.
The validity of \axEqName{} and \axEqFunc{} is by the definition of the interpretation of the constants and function symbols introduced for the purpose of reasoning (\Sec{sec:model}):
\begin{align*}
\semf^{\rand}(c^{(\datagenerator,\bmx)}) &= \semrxg{\bmx}
\\
\semf^{\rand}(f^{(\datagenerator,\bmy|_{\bmz\!,\bmx=\bmc})}) &= \semrxg{\bmy|_{\bmz\!,\bmx=\bmc}}.
\end{align*}

We show the validity of \axPD{} and \axMPD{} as follows.

\begin{restatable}[Probability distributions]{prop}{PropBasicPDMPD}
\label{prop:basic:PD:MPD}
Let $\bmx, \bmxB, \bmxC$, $\bmy \in\cvar^+$, $\nA, \nB\in\Const$, and $f \in \Func$.
\begin{enumerate}\renewcommand{\labelenumi}{(\roman{enumi})}
\item\!\label{item:basic:axPD}%
\axPD{}
~\hspace{0ex}~
$\vdashg (
	\pos(\bmx)
	\land
	\nA {=} \bmx
	\land
	f {=} \bmy|_{\bmx}
	\land
	\nB {=} \bmx \,{\joint}\, \bmy
) $
\\ \phantom{\axPD{}~$\vdashg$} \hspace{0ex}~
$\rightarrow \nB {=} f(\nA)$.
\item\!\label{item:basic:axMPD}%
\axMPD{}
~\hspace{0ex}~
$\vdashg \bmxB\mrgn{\bmxC} = \bmxC
~~\mbox{ if } \bmxC \subseteq \bmxB$.
\end{enumerate}
\end{restatable}

\begin{proof}
Let $w = (\datagen{w}, \semf_{w}, \mem{w})$
be a world such that $\bmx, \bmxB, \allowbreak \bmxC, \bmy \in \var(w)^+$.

\begin{enumerate}\renewcommand{\labelenumi}{(\roman{enumi})}
\item
We show the validity of \axPD{} as follows.
Suppose that
$w \models \pos(\bmx)\land\nA {=} \bmx\land
f {=} \bmy|_{\bmx}\land\nB {=} \bmx \,{\joint}\, \bmy$.
Then we have
$\sema{\bmx}{w}(o'_{\bmx}) > 0$ for all $o'_{\bmx} \in \calo^{|\bmx|}$,
$\sema{\nA}{w} = \sema{\bmx}{w}$,
$\sema{\nB}{w} = \sema{\bmx \,{\joint}\, \bmy}{w}$,
and
$\sema{f}{w} = \sema{\bmy|_{\bmx}}{w}$.
Since $\sema{\bmx}{w}(o'_{\bmx}) > 0$, we have:
\[
(\sema{\bmy|_{\bmx}}{w}(o'_{\bmx}))
=
{\textstyle\sum_{o'_{\bmy}}}\,
{\textstyle\frac{\sema{ \bmx \,{\joint}\, \bmy}{w}(o'_{\bmx},o'_{\bmy})}{\sema{\bmx}{w}(o'_{\bmx})}  \cdot \delta_{(o'_{\bmx},o'_{\bmy})}}.
\]

Thus, for each $o_{\bmx} \in \calo^{|\bmx|}$ and $o_{\bmy}\in \calo^{|\bmy|}$, we have:
\begin{align*}
&\sema{f(\nA)}{w}(o_{\bmx},o_{\bmy})\\
&= (\sema{f}{w} \sema{\nA}{w})(o_{\bmx},o_{\bmy})\\
&= (\sema{\bmy|_{\bmx}}{w} \sema{\bmx}{w})(o_{\bmx},o_{\bmy})\\
&=\Big( \sum_{o'_{\bmx}}\sum_{o'_{\bmy}} 
{\textstyle \frac{\sema{ \bmx \,{\joint}\, \bmy}{w}(o'_{\bmx},o'_{\bmy})}{\sema{\bmx}{w}(o'_{\bmx})} } \sema{\bmx}{w}(o'_{\bmx})   \cdot \delta_{(o'_{\bmx},o'_{\bmy})}  \Big)(o_{\bmx},o_{\bmy})
\\
&=\sema{ \bmx \,{\joint}\, \bmy}{w}(o_{\bmx},o_{\bmy})
\\
&=\sema{ c_1 }{w}(o_{\bmx},o_{\bmy}).
\end{align*}
Therefore, we obtain $w \models \nB {=} f(\nA)$.

\item
We show the validity of \axMPD{} as follows.
Let $\otimes$ be the product of probability distributions of data values.
Let $\bmxB = \{x_1,\ldots,x_k\}$.
Assume that $\emptyset \neq \bmxC \subseteq \bmxB$.
Then we may write 
$\bmxC = \{x_{l(1)},\ldots,x_{l(k')}\}$ for some $1 \leq k' \leq k$ and monotone increasing function $l \colon \{1,\ldots,k'\} \to  \{1,\ldots,k\}$.
Using this, we obtain:
\begin{align*}
&\sema{\bmxB\mrgn{\bmxC}}{w}\\
&= \sema{\mrgn{\bmxC}}{w} \sema{\bmxB}{w}\\
&= ((o_1,\ldots,o_k)\mapsto (o_{l(1)},\ldots,o_{l(k')}) ) \sema{\la x_1,\ldots,x_k \ra}{w}\\
&= ((o_1,\ldots,o_k)\mapsto (o_{l(1)},\ldots,o_{l(k')}) ) \sema{x_1}{w}\otimes\cdots \otimes\sema{x_k}{w}\\
&= \sema{x_{l(1)}}{w}\otimes\cdots \otimes\sema{x_{l(k')}}{w}\\
&= \sema{\la x_{l(1)},\ldots,x_{l(k')} \ra}{w}
= \sema{\bmxC}{w}.
\end{align*}
Therefore, we obtain $w \models \bmxB\mrgn{\bmxC} = \bmxC$ if $\bmxC \subseteq \bmxB$.
\myqed
\end{enumerate}
\end{proof}

\subsection{Validity of the Axioms for Eager Interventions}
\label{sub:app:intvE}

Here are the axioms of \AX{} with the eager interventions $\intvE{\cdot}$.
\[
\renewcommand{\arraystretch}{1.3}
\begin{array}{l@{\hspace{0ex}}l}
\mbox{\axDGEI}
&~~~ \vdashg \intvE{\subst{\bmc}{\bmx}} \phi
~\,\mbox{ iff } \vdashx{\datagenerator\intvE{\subst{\bmc}{\bmx}}}\phi
\\
\mbox{\axEffect~~}
&~~~ \vdashg 
\intvE{\subst{\bmc}{\bmx}} (\bmx = \bmc)
\\
\mbox{\axEqEI}
&~~~ \vdashg \bmuB = \bmuC \leftrightarrow \intvE{\subst{\bmc}{\bmx}} (\bmuB = \bmuC)
\\[-0.4ex]
&~~~ \mbox{ if } \fv{\bmuB} = \fv{\bmuC} = \emptyset
\\
\mbox{\axSplitE}
&~~~\vdashg
\intvE{\subst{\bmcB}{\bmxB},\, \subst{\bmcC}{\bmxC}} \phi
\rightarrow
\intvE{\subst{\bmcB}{\bmxB}} \intvE{\subst{\bmcC}{\bmxC}} \phi
\\
\mbox{\axSimulE{}}
&~~~\vdashg
\intvE{\subst{\bmcB}{\bmxB}} \intvE{\subst{\bmcC}{\bmxC}} \phi
\rightarrow
\intvE{\subst{\bmcB'}{\bmxB'},\, \subst{\bmcC}{\bmxC}} \phi
\\
&
~~~\mbox{ for }\,
\bmxB' = \bmxB \setminus \bmxC
~\mbox{ and }~
\bmcB' = \bmcB \setminus \bmcC
\\
\mbox{\axRptE{}}
&~~~\vdashg
\intvE{\subst{\bmc}{\bmx}} \phi
\rightarrow
\intvE{\subst{\bmc}{\bmx}} \intvE{\subst{\bmc}{\bmx}} \phi
\\
\mbox{\axCmpEI}
&~~~ \vdashg 
\big(
\intvE{\subst{\bmcB}{\bmxB}} (\bmxC = \bmcC) \land
\intvE{\subst{\bmcB}{\bmxB}} (\bmxD = \bmu)
\big)
\\[-0.3ex]
&~~~ \hspace{3.3ex} \rightarrow
\intvE{\subst{\bmcB}{\bmxB}, \subst{\bmcC}{\bmxC}} (\bmxD = \bmu)
\\
\mbox{\axDistrE}^{\neg}
&~~~\vdashg
(\intvE{\subst{\bmc}{\bmx}} \neg \phi)
\leftrightarrow
(\neg \intvE{\subst{\bmc}{\bmx}} \phi)
\\
\mbox{\axDistrE}^{\land}
&~~~\vdashg
(\intvE{\subst{\bmc}{\bmx}} (\phi_1 \land \phi_2))
\leftrightarrow
(\intvE{\subst{\bmc}{\bmx}} \phi_1 \land \intvE{\subst{\bmc}{\bmx}} \phi_2)
\end{array}
\]

Next, we show basic laws of eager interventions as follows.

\begin{restatable}[Basic laws of $\intvE{\cdot}$]{prop}{PropBasicIntvE}
\label{prop:basic:intvE}
Let $\bmx, \bmxB, \allowbreak \bmxC, \allowbreak \bmxD, \allowbreak \bmy, \bmz \in \cvar^+$ be disjoint,
$\bmc, \bmcB, \bmcC \in\dConst^+$,
$\bmu, \allowbreak \bmuB, \allowbreak \bmuC \in \Term^+$,
and $\phi \in \Fml$.
\begin{enumerate}
\item\!\label{item:intvE:DGE}%
\axDGEI{}
~\hspace{0ex}~
$\modelsg \intvE{\subst{\bmc}{\bmx}} \phi$
~iff~ $\models_{\datagenerator\intvE{\subst{\bmc}{\bmx}}} \phi$.
\item\!\label{item:intvE:Effect}%
\axEffect{}
~\hspace{0ex}~
$\models 
 \intvE{\subst{\bmc}{\bmx}} (\bmx = \bmc)$.
\item\!\label{item:intvE:EqEI}%
\axEqEI{}
~\hspace{0ex}~
$\models \bmuB = \bmuC \leftrightarrow \intvE{\subst{\bmc}{\bmx}} (\bmuB = \bmuC)$
\\
~\phantom{~\axEqEI{}\hspace{0.0ex}~}~
if $\fv{\bmuB} = \fv{\bmuC} = \emptyset$.
\item\!\label{item:intvE:SplitE}%
\axSplitE{}
~\hspace{0ex}~
$\models \intvE{\subst{\bmcB}{\bmxB},\, \subst{\bmcC}{\bmxC}} \phi
\rightarrow
\intvE{\subst{\bmcB}{\bmxB}} \intvE{\subst{\bmcC}{\bmxC}} \phi$.
\item\!\label{item:intvE:SimulE}%
\axSimulE{}
~\hspace{0ex}~
$\models \intvE{\subst{\bmcB}{\bmxB}} \intvE{\subst{\bmcC}{\bmxC}} \phi
\rightarrow
\intvE{\subst{\bmcB'}{\bmxB'},\, \subst{\bmcC}{\bmxC}} \phi$
\\
~\phantom{~\axSimulE{}\hspace{0.0ex}~}~
for
$\bmxB' = \bmxB \setminus \bmxC$
and
$\bmcB' = \bmcB \setminus \bmcC$.
\item\!\label{item:intvE:RptE}%
\axRptE{}
~\hspace{0ex}~
$\models \intvE{\subst{\bmc}{\bmx}} \phi
\rightarrow
\intvE{\subst{\bmc}{\bmx}} \intvE{\subst{\bmc}{\bmx}} \phi$.
\item\!\label{item:intvE:CmpEI}%
\axCmpEI{}
~\hspace{0ex}~
$\models 
\big(
\intvE{\subst{\bmcB}{\bmxB}} (\bmxC = \bmcC) \land
\intvE{\subst{\bmcB}{\bmxB}} (\bmxD = \bmu)
\big)$
\\[-0.1ex]
~\phantom{\axCmpEI{}~$\models$~}\hspace{0.5ex}~
$\rightarrow
\intvE{\subst{\bmcB}{\bmxB}, \subst{\bmcC}{\bmxC}} (\bmxD = \bmu)$.
\end{enumerate}
\end{restatable}

\begin{proof}
Let $w = (\datagen{w}, \semf_{w}, \mem{w})$
be a world such that $\bmx, \bmxB, \allowbreak \bmxC, \bmxD, \allowbreak \bmy, \bmz \in \var(w)^+$.
\begin{enumerate}
\item\!%
Assume that 
$\models_{\datagenerator\intvE{\subst{\bmc}{\bmx}}} \phi$.
Then for any world $w'$ having the data generator $\dgen{}$,
we have $w' \intvE{\subst{\bmc}{\bmx}} \models \phi$,
hence $w' \models \intvE{\subst{\bmc}{\bmx}} \phi$.
Therefore, $\modelsg \intvE{\subst{\bmc}{\bmx}} \phi$.
The other direction is also shown analogously.
\item\!%
By the definition of an eagerly intervened world in \Sec{sec:model},
we have $\datagen{w\intvE{\subst{\bmc}{\bmx}}}(\bmx) = \bmc$.
By $\semf_{w} = \semf_{w\intvE{\subst{\bmc}{\bmx}}}$,
$\sema{ \bmc }{w} = \sema{ \bmc }{w\intvE{\subst{\bmc}{\bmx}}}$.

Then
$\sema{\bmx}{w\intvE{\subst{\bmc}{\bmx}}} =
 \sema{ \datagen{w\intvE{\subst{\bmc}{\bmx}}}(\bmx)}{w\intvE{\subst{\bmc}{\bmx}}} =
 \sema{ \bmc }{w\intvE{\subst{\bmc}{\bmx}}}$.
Hence
$w\intvE{\subst{\bmc}{\bmx}} \models \bmx = \bmc$.
Therefore, 
$w \models \intvE{\subst{\bmc}{\bmx}} (\bmx = \bmc)$.
\item\!%
Assume that $\fv{\bmuB} = \fv{\bmuC} = \emptyset$.
Then for each $i = 1, 2$,
$\sema{\bmu_i}{w} = \sema{\bmu_i}{w\intvE{\subst{\bmc}{\bmx}}}$.
Hence,
$\models \bmuB = \bmuC \leftrightarrow \intvE{\subst{\bmc}{\bmx}} (\bmuB = \bmuC)$.
\item\!%
The proof is straightforward from the definition.
\item\!%
The proof is straightforward from the definition.
\item\!%
The proof is straightforward from the definition.
\item\!%
Assume that 
$w \models 
\intvE{\subst{\bmcB}{\bmxB}} (\bmxC = \bmcC) \land
\intvE{\subst{\bmcB}{\bmxB}} (\bmxD = \bmu)$.
Then $\datagen{w\intvE{\subst{\bmcB}{\bmxB}}}(\bmxC) = \bmcC$
and $\datagen{w\intvE{\subst{\bmcB}{\bmxB}}}(\bmxD) = \bmu$.
Let $w' = w \intvE{\subst{\bmcB}{\bmxB}}$.
Thus, 
\begin{align*}
& \phantom{~=~}
\sema{\datagen{w\intvE{\subst{\bmcB}{\bmxB}, \subst{\bmcC}{\bmxC}}}(\bmxD)}{w\intvE{\subst{\bmcB}{\bmxB}, \subst{\bmcC}{\bmxC}}} 
\\ &
= \sema{\datagen{w' \intvE{\subst{\bmcC}{\bmxC}}}(\bmxD)}{w' \intvE{\subst{\bmcC}{\bmxC}}}
\\ &
= \sema{\datagen{w'}(\bmxD)}{w' \intvE{\subst{\bmcC}{\bmxC}}}
\\ &
= \sema{\bmu}{w\intvE{\subst{\bmcB}{\bmxB}, \subst{\bmcC}{\bmxC}}}.
\end{align*}

Therefore,
$w \models \intvE{\subst{\bmcB}{\bmxB}, \subst{\bmcC}{\bmxC}} (\bmxD = \bmu)$.
\myqed
\end{enumerate}
\end{proof}

The eager intervention operator $\intvE{\cdot}$ is distributive w.r.t. logical connectives.
\begin{restatable}[Distributive laws of $\intvE{\cdot}$]{prop}{PropIntvEDistr}
\label{prop:intvE-distr}
Let $\bmx\in\cvar^+$,
$\bmc\in\dConst^+$, and $\phi,\phi'\in\Fml$.
\begin{enumerate}\renewcommand{\labelenumi}{(\roman{enumi})}
\item\!%
\axDistrE$^{\neg}$
~\hspace{1ex}~
$\models 
 \intvE{\subst{\bmc}{\bmx}} \neg\phi
 \leftrightarrow
 \neg\intvE{\subst{\bmc}{\bmx}} \phi$.
\item\!%
\axDistrE$^{\rightarrow}$
~\hspace{1ex}~
$\models 
 \intvE{\subst{\bmc}{\bmx}} (\phi \rightarrow \phi')
 \leftrightarrow
 \big( \intvE{\subst{\bmc}{\bmx}} \phi \rightarrow \intvE{\subst{\bmc}{\bmx}} \phi' \bigr)$.
\end{enumerate}
Similarly, the eager intervention operator $\intvE{\cdot}$ is distributive w.r.t. $\lor$ and $\land$.
\end{restatable}
\begin{proof}
Let $w$ be a world such that $\bmx\in\var(w)^+$.
\begin{enumerate}\renewcommand{\labelenumi}{(\roman{enumi})}
\item\!%
\begin{align*}
w\models \intvE{\subst{\bmc}{\bmx}} \neg\phi
~\mbox{ iff } &~
w\intvE{\subst{\bmc}{\bmx}} \models \neg\phi
\\ ~\mbox{ iff } &~
w\intvE{\subst{\bmc}{\bmx}} \not\models \phi
\\ ~\mbox{ iff } &~
w \not\models \intvE{\subst{\bmc}{\bmx}} \phi
\\ ~\mbox{ iff } &~
w \models \neg \intvE{\subst{\bmc}{\bmx}} \phi.
\end{align*}

\item\!%
We first show the direction from left to right as follows.
\begin{align*}
&~
w\models 
 \intvE{\subst{\bmc}{\bmx}} (\phi \rightarrow \phi')
\mbox{ and }
w\models 
 \intvE{\subst{\bmc}{\bmx}} \phi
\\ \Longrightarrow &~
w\intvE{\subst{\bmc}{\bmx}} \models \phi \rightarrow \phi'
\mbox{ and }
w\intvE{\subst{\bmc}{\bmx}} \models \phi
\\ \Longrightarrow &~
w\intvE{\subst{\bmc}{\bmx}} \models \phi'
\\ \Longrightarrow &~
w\models 
 \intvE{\subst{\bmc}{\bmx}} \phi'.
\end{align*}

We next show the other direction as follows.
Assume that
$w\models 
 \intvE{\subst{\bmc}{\bmx}} \phi \rightarrow \intvE{\subst{\bmc}{\bmx}} \phi'$.
Then:
\begin{align*}
&~
w\intvE{\subst{\bmc}{\bmx}} \models \phi
\\ \Longrightarrow &~
w\models \intvE{\subst{\bmc}{\bmx}} \phi
\\ \Longrightarrow &~
w\models \intvE{\subst{\bmc}{\bmx}} \phi'
& \text{(by assumption)}
\\ \Longrightarrow &~
w\intvE{\subst{\bmc}{\bmx}} \models \phi'.
\end{align*}
Hence
$w\models 
 \intvE{\subst{\bmc}{\bmx}} (\phi \rightarrow \phi')$
iff
$w\models 
 \intvE{\subst{\bmc}{\bmx}} \phi \rightarrow \intvE{\subst{\bmc}{\bmx}} \phi'$.
\myqed
\end{enumerate}
\end{proof}

\subsection{Validity of the Axioms for Lazy Interventions}
\label{sub:app:lazy:intv}

Here are the axioms of \AX{} with the lazy interventions $\intvL{\cdot}$.
\[
\renewcommand{\arraystretch}{1.3}
\begin{array}{l@{\hspace{0ex}}l}
\mbox{\axDGL}
&~~~ \vdashg\intvL{\subst{\bmc}{\bmx}} \phi
\mbox{~\,iff } \vdashx{\datagenerator\intvL{\subst{\bmc}{\bmx}}}\phi
\\
\mbox{\axCondL{}}
&~~~\vdashg
( f = \bmy|_{\bmx = \bmc} ) \leftrightarrow
 \intvL{\subst{\bmc}{\bmx}} ( f = \bmy|_{\bmx = \bmc} )
\\
\mbox{\axEqLI}
&~~~ \vdashg \bmuB = \bmuC \leftrightarrow \intvL{\subst{\bmc}{\bmx}} (\bmuB = \bmuC)
\\[-0.4ex]
&~~~ \mbox{ if } \fv{\bmuB} = \fv{\bmuC} = \emptyset
\\
\mbox{\axSplitL}
&~~~\vdashg
\intvL{\subst{\bmcB}{\bmxB},\, \subst{\bmcC}{\bmxC}} \phi
\rightarrow
\intvL{\subst{\bmcB}{\bmxB}} \intvL{\subst{\bmcC}{\bmxC}} \phi
\\
\mbox{\axSimulL}
&~~~\vdashg
\intvL{\subst{\bmcB}{\bmxB}} \intvL{\subst{\bmcC}{\bmxC}} \phi
\rightarrow
\intvL{\subst{\bmcB'}{\bmxB'},\, \subst{\bmcC}{\bmxC}} \phi
\\
&~~~\mbox{ if }\,
\bmxB' = \bmxB \setminus \bmxC
~\mbox{ and }~
\bmcB' = \bmcB \setminus \bmcC
\\
\mbox{\axRptL{}}
&~~~\vdashg
\intvL{\subst{\bmc}{\bmx}} \phi
\rightarrow
\intvL{\subst{\bmc}{\bmx}} \intvL{\subst{\bmc}{\bmx}} \phi
\\
\mbox{\axCmpLI}
&~~~ \vdashg 
\big(
\intvL{\subst{\bmcB}{\bmxB}} (\bmxC = \bmcC) \land
\intvL{\subst{\bmcB}{\bmxB}} (\bmxD = \bmu)
\big)
\\[-0.3ex]
&~~~ \hspace{3.3ex} \rightarrow
\intvL{\subst{\bmcB}{\bmxB}, \subst{\bmcC}{\bmxC}} (\bmxD = \bmu)
\\
\mbox{\axDistrL}^{\neg}
&~~~\vdashg
(\intvL{\subst{\bmc}{\bmx}} \neg \phi)
\leftrightarrow
(\neg \intvL{\subst{\bmc}{\bmx}} \phi)
\\
\mbox{\axDistrL}^{\land}
&~~~\vdashg
(\intvL{\subst{\bmc}{\bmx}} (\phi_1 \land \phi_2))
\leftrightarrow
(\intvL{\subst{\bmc}{\bmx}} \phi_1 \land \intvL{\subst{\bmc}{\bmx}} \phi_2)
\end{array}
\]

\begin{restatable}[Basic properties of $\intvL{\cdot}$]{prop}{PropIntvL}
\label{prop:intvL}
Let $\bmx,\allowbreak \bmxB, \allowbreak \bmxC, \bmxD, \bmy\in\cvar^+$ be disjoint, 
$\bmc, \bmcB, \bmcC \in\dConst^+$,
$\bmu, \allowbreak \bmuB, \allowbreak \bmuC \in \Term^+$,
$f \in \Func$, and 
$\phi \in \Fml$.
\begin{enumerate}
\item\!%
\axDGL{}
~\hspace{1ex}~
$\modelsg \intvL{\subst{\bmc}{\bmx}} \phi$
~\,iff $\vdashx{\datagenerator\intvL{\subst{\bmc}{\bmx}}}\phi$.

\item\!%
\axCondL{}
~\hspace{1ex}~
$\models ( f = \bmy|_{\bmx = \bmc} ) \leftrightarrow
 \intvL{\subst{\bmc}{\bmx}} ( f = \bmy|_{\bmx = \bmc} )$.
\item\!%
\axEqLI{}
~\hspace{1ex}~
$\models \bmuB = \bmuC \leftrightarrow \intvL{\subst{\bmc}{\bmx}} (\bmuB = \bmuC)$
\mbox{~ if } $\fv{\bmuB} = \fv{\bmuC} = \emptyset$.
\item\!%
\mbox{\axSplitL}
$\models \intvL{\subst{\bmcB}{\bmxB},\, \subst{\bmcC}{\bmxC}} \phi
\rightarrow
\intvL{\subst{\bmcB}{\bmxB}} \intvL{\subst{\bmcC}{\bmxC}} \phi$
\item\!\label{item:intvL:SimulL}%
\mbox{\axSimulL}
$\models \intvL{\subst{\bmcB}{\bmxB}} \intvL{\subst{\bmcC}{\bmxC}} \phi
\rightarrow
\intvL{\subst{\bmcB'}{\bmxB'},\, \subst{\bmcC}{\bmxC}} \phi$
\mbox{if}
$\bmxB' {=} \bmxB {\setminus} \bmxC$
\mbox{and}
$\bmcB' {=} \bmcB {\setminus} \bmcC$.
\item\!\label{item:intvL:RptL}%
\axRptL{}
~\hspace{1ex}~
$\models \intvL{\subst{\bmc}{\bmx}} \phi
\rightarrow
\intvL{\subst{\bmc}{\bmx}} \intvL{\subst{\bmc}{\bmx}} \phi$.
\item\!%
\axCmpLI{}
~\hspace{1ex}~
$\models 
\big(
\intvL{\subst{\bmcB}{\bmxB}} (\bmxC = \bmcC) \land
\intvL{\subst{\bmcB}{\bmxB}} (\bmxD = \bmu)
\big)$
\\[-0.1ex]
~\phantom{\axCmpEI{}~$\models$~}~\hspace{1ex}~
$\rightarrow \intvL{\subst{\bmcB}{\bmxB}, \subst{\bmcC}{\bmxC}} (\bmxD = \bmu)$.
\end{enumerate}
\end{restatable}
\begin{proof}
Let $w = (\datagen{w}, \semf_{w}, \mem{w})$ be a world such that $\bmx, \bmxB, \bmxC, \bmxD, \allowbreak\bmy \in\var(w)^+$.
\begin{enumerate}
\item\!
The proof is similar to \Propo{prop:basic:intvE}.

\item
Let 
$\datagen{w}(\bmx) = \bmu$,
$\diag{w}$ be the causal diagram corresponding to $\datagen{w}$, 
and $\Dec(\bmx)$ be the set of all descendant variables of $\bmx$.
Let $\bmyA, \bmyB$ be possibly empty tuples of variables such that
$\bmy = \bmyA \joint \bmyB$,
$\bmyA \subseteq \Dec(\bmx)$, and
$\bmyB \cap \Dec(\bmx) = \emptyset$.
Then on every undirected path between $\bmyA$ and $\bmyB$ in $\diag{w}$,\,
$\bmx$ are on chains or forks.
Hence $P_{\diag{w}}(\bmyA\joint\bmyB | \bmx) = P_{\diag{w}}(\bmyA | \bmx)\cdot P_{\diag{w}}(\bmyB | \bmx)$.
Thus, we obtain:
\begin{align*}
&~~
\sema{\bmy|_{\bmx=\bmc}}{w}
\\ =&~~
P_{\diag{w}}(\bmyA\joint\bmyB | \bmx=\bmc) 
\\ =&~~
P_{\diag{w}}(\bmyA | \bmx=\bmc)\cdot P_{\diag{w}}(\bmyB | \bmx=\bmc)
\\ =&~~
P_{\diag{w \intvL{\subst{\bmc}{\bmx}}}}(\bmyA | \bmx=\bmc) \cdot 
P_{\diag{w} \intvL{\subst{\bmc}{\bmx}}}(\bmyB | \bmx=\bmc)
\\ =&~~
P_{\diag{w \intvL{\subst{\bmc}{\bmx}}}}(\bmyA\joint\bmyB | \bmx=\bmc)
\\ =&~~
\sema{\bmy|_{\bmx=\bmc}}{w  \intvL{\subst{\bmc}{\bmx}}}
{.}
\end{align*}

Therefore, $w \models f = \bmy|_{\bmx=\bmc}$ iff
$w \models \intvL{\subst{\bmc}{\bmx}} f = \bmy|_{\bmx=\bmc}$.

\item\!%
Assume that $\fv{\bmuB} = \fv{\bmuC} = \emptyset$.
Then for each $i = 1, 2$,
$\sema{\bmu_i}{w} = \sema{\bmu_i}{w\intvL{\subst{\bmc}{\bmx}}}$.
Hence,
$\models \bmuB = \bmuC \leftrightarrow \intvL{\subst{\bmc}{\bmx}} (\bmuB = \bmuC)$.

\item
The proof is straightforward from the definition.

\item
The proof is straightforward from the definition.

\item
The proof is straightforward from the definition.

\item
Assume that 
$w \models 
\intvL{\subst{\bmcB}{\bmxB}} (\bmxC = \bmcC) \land
\intvL{\subst{\bmcB}{\bmxB}} (\bmxD = \bmu)$.
Then $\datagen{w\intvL{\subst{\bmcB}{\bmxB}}}(\bmxC) = \bmcC$
and $\datagen{w\intvL{\subst{\bmcB}{\bmxB}}}(\bmxD) = \bmu$.
Let $w' = w \intvL{\subst{\bmcB}{\bmxB}}$.
Thus, 
\begin{align*}
& \phantom{~=~}
\sema{\datagen{w\intvL{\subst{\bmcB}{\bmxB}, \subst{\bmcC}{\bmxC}}}(\bmxD)}{w\intvL{\subst{\bmcB}{\bmxB}, \subst{\bmcC}{\bmxC}}} 
\\ &
= \sema{\datagen{w' \intvL{\subst{\bmcC}{\bmxC}}}(\bmxD)}{w' \intvL{\subst{\bmcC}{\bmxC}}}
\\ &
= \sema{\datagen{w'}(\bmxD) [\bmxC \mapsto \bmcC]}{w' \intvL{\subst{\bmcC}{\bmxC}}}
\\ &
= \sema{\bmu [\bmxC \mapsto \bmcC]}{w' \intvL{\subst{\bmcC}{\bmxC}}}
\\ &
= \sema{\bmu}{w' \intvL{\subst{\bmcC}{\bmxC}}}.
\\ &
= \sema{\bmu}{w \intvL{\subst{\bmcB}{\bmxB}, \subst{\bmcC}{\bmxC}}}.
\end{align*}

Therefore,
$w \models \intvL{\subst{\bmcB}{\bmxB}, \subst{\bmcC}{\bmxC}} (\bmxD = \bmu)$.
\myqed
\end{enumerate}
\end{proof}

The lazy intervention operator$\intvL{\cdot}$ is also distributive w.r.t. logical connectives.
\begin{restatable}[Distributive laws of $\intvL{\cdot}$]{prop}{PropIntvLDistr}
\label{prop:intvL-distr}
Let $\bmx\in\cvar^+$,
$\bmc\in\dConst^+$, and $\phi,\phi'\in\Fml$.
\begin{enumerate}\renewcommand{\labelenumi}{(\roman{enumi})}
\item\!%
\axDistrL$^{\neg}$
~\hspace{1ex}~
$\models 
 \intvL{\subst{\bmc}{\bmx}} \neg\phi
 \leftrightarrow
 \neg\intvL{\subst{\bmc}{\bmx}} \phi$.
\item\!%
\axDistrL$^{\rightarrow}$
~\hspace{1ex}~
$\models 
 \intvL{\subst{\bmc}{\bmx}} (\phi \rightarrow \phi')
 \leftrightarrow
 \big( \intvL{\subst{\bmc}{\bmx}} \phi \rightarrow \intvL{\subst{\bmc}{\bmx}} \phi' \big)$.
\end{enumerate}
Similarly, the lazy intervention operator $\intvL{\cdot}$ is distributive w.r.t. $\lor$ and $\land$.
\end{restatable}
\begin{proof}
The proofs are analogous to those for \Propo{prop:intvE-distr}.
\myqed
\end{proof}

\subsection{Validity of the Exchanges of Eager/Lazy Interventions}
\label{sub:app:two:intv}

Here are the axioms of \AX{} for the exchanges of eager/lazy interventions.
\[
\renewcommand{\arraystretch}{1.3}
\begin{array}{l@{\hspace{0ex}}l}
\mbox{\axXpdEL{}}
&~\vdashg
 (\intvE{\subst{\bmc}{\bmx}} \bmc' = \bmy)
 \leftrightarrow
 (\intvL{\subst{\bmc}{\bmx}} \bmc' = \bmy)
\\
\mbox{\axXcdEL{}}
&~\vdashg
 \pos(\bmz) \rightarrow
 \big(
 (\intvE{\subst{\bmc}{\bmx}} f = \bmy|_{\bmz})
 {\leftrightarrow}
 (\intvL{\subst{\bmc}{\bmx}} f = \bmy|_{\bmz})
 \big)
\end{array}
\]

\begin{restatable}[Exchanges of $\intvE{\cdot}$ and $\intvL{\cdot}$]{prop}{PropIntvEIntvL}
\label{prop:IntvE-intvL}
Let $\bmx,\bmy,\bmz\in\cvar^+$, $\bmc, \bmc' \in\dConst^+$, 
and $f\in\Func$.
\begin{enumerate}\renewcommand{\labelenumi}{(\roman{enumi})}
\item\!%
\axXpdEL{}
$\models 
 (\intvE{\subst{\bmc}{\bmx}} \bmc' = \bmy)
 \leftrightarrow
 (\intvL{\subst{\bmc}{\bmx}} \bmc' = \bmy)$.
\item\!%
\axXcdEL{}
$\models 
 \pos(\bmz) \,{\rightarrow}\,
 \big(
 (\intvE{\subst{\bmc}{\bmx}} f \,{=}\, \bmy|_{\bmz})
 {\leftrightarrow}
 (\intvL{\subst{\bmc}{\bmx}} f \,{=}\, \bmy|_{\bmz})
 \big)$.
\end{enumerate}
\end{restatable}

\begin{proof}
\begin{enumerate}\renewcommand{\labelenumi}{(\roman{enumi})}
\item\!%
Let $w$ be a world.
By \Propo{prop:dist-formula} (ii), we have:
$\sema{\bmy}{w \intvE{\subst{\bmc}{\bmx}}} =
\mem{w \intvE{\subst{\bmc}{\bmx}}}(\bmy) =
\mem{w \intvL{\subst{\bmc}{\bmx}}}(\bmy) =
\sema{\bmy}{w \intvL{\subst{\bmc}{\bmx}}}$.
Therefore, we obtain the claim.

\item\!%
Let $w$ be a world such that $w \models \pos(\bmz)$.
Let $\bmcB, \bmcC \in \dConst^+$.
By the first claim, we have
$w\models 
 (\intvE{\subst{\bmc}{\bmx}} \bmcB = \bmy\joint\bmz)
 \leftrightarrow
 (\intvL{\subst{\bmc}{\bmx}} \bmcB = \bmy\joint\bmz)$
and
$w\models 
 (\intvE{\subst{\bmc}{\bmx}} \bmcC = \bmz)
 \leftrightarrow
 (\intvL{\subst{\bmc}{\bmx}} \bmcC = \bmz)$.
Then we have:
\begin{align*}
\sema{\bmy|_{\bmz}}{w \intvE{\subst{\bmc}{\bmx}}} 
&=
\frac{ \sema{\bmy\joint\bmz}{w \intvE{\subst{\bmc}{\bmx}}} }{ \sema{\bmz}{w \intvE{\subst{\bmc}{\bmx}}} }
\\ &=
\frac{ \sema{\bmy\joint\bmz}{w \intvL{\subst{\bmc}{\bmx}}} }{ \sema{\bmz}{w \intvL{\subst{\bmc}{\bmx}}} }
\\ &=
\sema{\bmy|_{\bmz}}{w \intvL{\subst{\bmc}{\bmx}}} 
\end{align*}
Therefore, we obtain the claim.
\myqed
\end{enumerate}
\end{proof}

\subsection{Remarks on Axioms and Invalid Formulas}
\label{sub:StaCL:invalid}

From our axioms, we can derive the following formulas that are considered as axioms in the  previous work~\cite{Barbero:21:jphil}.
\[
\renewcommand{\arraystretch}{1.3}
\begin{array}{l@{\hspace{0ex}}l}
\mbox{\axUniq~~}
&~~~\vdashg
\intvE{\subst{\bmu}{\bmx}} (\bmy = \bmd)
\rightarrow
\intvE{\subst{\bmu}{\bmx}} (\bmy \neq \bmd')
~\mbox{ for }
 \bmd \neq \bmd'
\end{array}
\]

We show examples formulas that are not valid in our model.
The following formulas suggest the difference between the intervention $\intvE{\subst{\bmc}{\bmx}} \phi$ and the conditioning $(\bmx = \bmc) \rightarrow \phi$.
\begin{itemize}
\item Strengthened intervention: \\
$\intvE{\subst{\bmu_1}{\bmx_1}} \phi
\not\models 
\intvE{\subst{\bmu_1}{\bmx_1}, \subst{\bmu_2}{\bmx_2}} \phi$.

\item Pseudo transitivity: \\
$( \intvE{\subst{\bmu}{\bmx}} \bmy = \bmd) \land
\intvE{\subst{\bmd}{\bmy}} \phi
\not\models 
\intvE{\subst{\bmu}{\bmx}} \phi$.

\item Weak pseudo transitivity: \\
$( \intvE{\subst{\bmu}{\bmx}} \bmy = \bmd) \land
\intvE{\subst{\bmu}{\bmx},\subst{\bmd}{\bmy}} \phi
\not\models 
\intvE{\subst{\bmu}{\bmx}} \phi$.

\item Pseudo contraposition: \\
$\intvE{\subst{\bmu}{\bmx}} \bmd = \bmy
\not\models 
\intvE{\subst{\bmd}{\bmy}} (\bmx = \bmu)$.

\item Replacing conjunction with intervention: \\
$\bmx = \bmu \land \phi
 \not\models
 \intvE{\subst{\bmu}{\bmx}} \phi$.
\item Pseudo Modus Ponens: \\
$( \bmx = \bmu \land \intvE{\subst{\bmu}{\bmx}} \phi )
\not\models
\phi$.

\item Pseudo Modus Tollens: \\
$( \neg \phi \land \intvE{\subst{\bmu}{\bmx}} \phi )
\not\models 
(\bmx \neq \bmu)$.
\end{itemize}
Similar formulas are not valid also in~\cite{Barbero:21:jphil}, which does not deal with probability distributions.

\subsection{Remark on Defining Lazy Interventions as Syntax Sugar}
\label{sub:lazy:form}

We remark that a lazy intervention  $\intvL{\subst{\bmc}{\bmx}}$ can be defined as syntax sugar if we expand data generators.

Recall that 
$\intvL{\subst{\bmc}{\bmx}} \phi$ expresses that $\phi$ is satisfied in the lazy intervened world:
\begin{align*}
\M, w \models \intvL{\subst{\bmc}{\bmx}} \phi
& ~\mbox{ iff }~
\M, w\intvL{\subst{\bmc}{\bmx}} \models \phi
{.}
\end{align*}
To define $\intvL{\subst{\bmc}{\bmx}} \phi$ as syntax sugar, 
we \emph{expand} the data generator $\datagen{w}$ as follows.
For each $x \in \dom{\datagen{w}}$, 
we introduce a fresh auxiliary variable $x'$,
add $\datagen{w}(x') = x$, and 
replace every occurrence of $x$ in $\range(\datagen{w})$ with $x'$.
Then the corresponding causal diagram has arrows $x \garrow x'$ and $x' \garrow y$ instead of $x \garrow y$.
Now the lazy intervention $\intvL{\subst{\bmc}{\bmx}} \phi$ can be defined as the eager intervention $\intvE{\subst{\bmc}{\bmx'}} \phi$.

In summary, we can replace the lazy intervention $\intvL{\subst{\bmc}{\bmx}} \phi$ with its corresponding eager intervention $\intvE{\subst{\bmc}{\bmx'}} \phi$ by considering a model $\M$ that have possible worlds equipped only with expanded data generators.

\section{Proof for the Soundness of \AXwCP}
\label{sec:StaCL:ax:pred}

In this section, we prove the soundness of \AXwCP{}  w.r.t. the Kripke semantics for statistical causality
by showing the validity of the axioms with the $d$-separation predicate $\dsep$ (\App{sub:StaCL:ax:dsep}),
with the non-ancestor causal predicate $\nanc$ (\App{sub:StaCL:ax:nanc}),
and with other causal predicates (\App{sub:StaCL:ax:CP})

\subsection{Validity of the Axioms with $d$-Separation}
\label{sub:StaCL:ax:dsep}

Here are the axioms of \AXwCP{} with the $d$-separation $\dsep$.
\[
\renewcommand{\arraystretch}{1.3}
\begin{array}{l@{\hspace{0.5ex}}l}
\!\mbox{\axDsepCIndB{}}
&~~~ 
\vdashg ( \dsep(\bmx, \bmy, \bmz) \land \pos(\bmz)) \rightarrow
\bmy|_{\bmz,\bmx=\bmc} = \bmy|_{\bmz}
\\
\!\mbox{\axDsepSm{}}
&~~~ 
\vdashg \dsep(\bmx, \bmy, \bmz) \leftrightarrow \dsep(\bmy, \bmx, \bmz)
\\
\!\mbox{\axDsepDc{}}
&~~~ 
\vdashg \dsep(\bmx, \bmy{\cup}\bmy', \bmz) \rightarrow (\dsep(\bmx, \bmy, \allowbreak \bmz) \land  \dsep(\bmx, \bmy', \bmz))
\\
\!\mbox{\axDsepWU{}}
&~~~ 
\vdashg \dsep(\bmx, \bmy{\cup}\bmv, \bmz) \rightarrow \dsep(\bmx, \bmy, \bmz{\cup}\bmv)
\\
\!\mbox{\axDsepCn{}}
&~~~ 
\vdashg (\dsep(\bmx, \bmy, \bmz) {\land} \dsep(\bmx, \bmv, \bmz{\cup}\bmy)) \rightarrow \dsep(\bmx, \bmy{\cup}\bmv, \bmz)
\\
\!\mbox{\axDsepEN{}}
&~~~ 
\vdashg (\intvE{\subst{\bmc}{\bmz}} \dsep(\bmx, \bmy, \bmz)) \leftrightarrow
  \dsep(\bmx, \bmy, \bmz)
\\
\!\mbox{\axDsepLN{}}
&~~~ 
\vdashg (\intvL{\subst{\bmc}{\bmz}} \dsep(\bmx, \bmy, \bmz)) \leftrightarrow
  \dsep(\bmx, \bmy, \bmz)
\\
\!\mbox{\axDsepLNC{}}
&~~~ 
\vdashg \dsep(\bmx, \bmy, \bmz\cup\bmz') \rightarrow
 \intvL{\subst{\bmc}{\bmz}} \dsep(\bmx, \bmy, \bmz')
\end{array}
\]

We first show the validity of 
\axDsepCIndB{}.
It is well-known that the $d$-separation in a causal diagram $\diag{}$ implies the conditional independence, but not vice versa~\cite{Pearl:09:causality}.
However, if $\sema{\bmx}{w}$ and $\sema{\bmy}{w}$ are conditionally independent given $\sema{\bmz}{w}$ for any interpretation $\sema{\_}{w}$ factorizing $\diag{}$ (i.e., for any world $w$ with the data generator $\datagen{w}$ corresponding to $\diag{}$), then they are $d$-separated by~$\bmz$.

\begin{restatable}[$d$-separation and conditional independence]{prop}{PropDsepBasic}
\label{prop:dsep:CInd}
Let $\bmx,\bmy\in\cvar^+$ and $\bmz\in\cvar^*$ be disjoint.
Let $\bmc \in\dConst^+$.
\begin{enumerate}
\item\!%
\mbox{\axDsepCIndB{}~\hspace{2ex}~} \\
$\models\, ( \dsep(\bmx, \bmy, \bmz) \land \pos(\bmz)) \rightarrow
\bmy|_{\bmz,\bmx=\bmc} = \bmy|_{\bmz}$.
\item\!%
For any finite, closed, acyclic data generator $\dgen{}$, we have:
\begin{align} 
\label{eq:dsep-cond-ind3}
& \modelsg (\pos(\bmz) \rightarrow \bmy|_{\bmz,\bmx=\bmc} = \bmy|_{\bmz})
~~\mbox{ implies }~~
\modelsg \dsep(\bmx, \bmy, \bmz).
\end{align}
\end{enumerate}
\end{restatable}

\begin{proof}
We show the first claim as follows.
Let $w$ be a world.
Assume that $w \models \dsep(\bmx, \bmy, \bmz) \land \pos(\bmz)$.
Then in the causal diagram $\diag{w}$,\, $\bmx$ and $\bmy$ are $d$-separated by $\bmz$.
Thus, $\bmx$ and $\bmy$ are conditionally independent given $\bmz$.
Therefore,
$w \models\, \bmy|_{\bmz,\bmx=\bmc} = \bmy|_{\bmz}$.

We show the second claim as follows.
Assume that $\modelsg\, (\pos(\bmz) \rightarrow \bmy|_{\bmz,\bmx=\bmc} = \bmy|_{\bmz})$.
Then, for any world $w$ with a data generator $\datagen{}$,\, 
$\sema{\bmx}{w}$ and $\sema{\bmy}{w}$ are conditionally independent given $\sema{\bmz}{w}$.
Let $\diag{}$ be the causal diagram corresponding to $\datagen{}$.
We recall that if $\bmx$ and $\bmy$ are conditionally independent given $\bmz$ for any joint distribution $P_{\diag{}}$ factorized according to $\diag{}$, 
then they are $d$-separated by $\bmz$ in $\diag{}$
(see e.g.,~\cite{Pearl:09:causality}).
Therefore, we obtain $\modelsg\, \dsep(\bmx, \bmy, \bmz)$.
\myqed
\end{proof}

$d$-separation is known to satisfy the \emph{semi-graphoid} axioms~\cite{Verma:88:UAI}, which we can describe using our logic as follows:

\begin{restatable}[Semi-graphoid]{prop}{PropDsep}
\label{prop:d-sep-intv}
Let $\bmx,\bmy,\bmy' \in\cvar^+$ and
$\bmz,\bmv \in\cvar^*$
be disjoint.
Then $\dsep$ satisfies:
\begin{enumerate}
\item\!%
\axDsepSm{} (symmetry): \\
$\models \dsep(\bmx, \bmy, \bmz) \leftrightarrow \dsep(\bmy, \bmx, \bmz)$.
\item\!%
\axDsepDc{} (decomposition): \\
$\models \dsep(\bmx, \bmy\cup\bmy', \bmz) \rightarrow (\dsep(\bmx, \bmy, \allowbreak \bmz) \land  \dsep(\bmx, \bmy', \bmz))$.
\item\!%
\axDsepWU{} (weak union): \\
$\models \dsep(\bmx, \bmy\cup\bmv, \bmz) {\rightarrow} \dsep(\bmx, \bmy, \bmz \cup \bmv)$.
\item\!%
\axDsepCn{} (contraction): \\
$\models (\dsep(\bmx, \bmy, \bmz) \land \dsep(\bmx, \bmv, \bmz\cup\bmy)) \rightarrow \dsep(\bmx, \bmy\cup\bmv, \bmz)$.
\end{enumerate}
\end{restatable}
~

The causal predicates and interventions satisfy the following axioms,
which are later used in
\conference{\App{sec:appendix:StaCL:reasoning}}
\arxiv{\App{sec:appendix:StaCL:reasoning}}
to prove the soundness of Pearl's do-calculus rules (Proposition~\ref{prop:do-calculus}).
We prove the validity of these axioms 
and an additional property \axDsepLNC{}
as follows.

\begin{restatable}[Relationships between $\dsep$ and $\intvE{\cdot}$]{prop}{PropDsepIntv}
\label{prop:d-sep:intv}
Let $\bmx,\bmy\in\cvar^+$ and $\bmz,\bmz'\in\cvar^*$ be disjoint, and 
$\bmc\in\dConst^+$.
\begin{enumerate}
\item \label{claim:dsep4}
\axDsepENA{}~\hspace{2ex}~
$\models (\intvE{\subst{\bmc}{\bmz}} \dsep(\bmx, \bmy, \bmz)) \rightarrow
  \dsep(\bmx, \bmy, \bmz)$.
\item \label{claim:dsep1}
\axDsepENB{}~\hspace{2ex}~
$\models \dsep(\bmx, \bmy, \bmz) \rightarrow
 \intvE{\subst{\bmc}{\bmx}} \dsep(\bmx, \bmy, \bmz)$.
\item \label{claim:dsep5}
\axDsepLNA{}~\hspace{2ex}~
$\models (\intvL{\subst{\bmc}{\bmz}} \dsep(\bmx, \bmy, \bmz)) \rightarrow
  \dsep(\bmx, \bmy, \bmz)$.
\item \label{claim:dsep2}
\axDsepLNB{}~\hspace{2ex}~
$\models \dsep(\bmx, \bmy, \bmz) \rightarrow
 \intvL{\subst{\bmc}{\bmx}} \dsep(\bmx, \bmy, \bmz)$.
\item \label{claim:dsep3}
\axDsepLNC{}~\hspace{2ex}~
$\models \dsep(\bmx, \bmy, \bmz\cup\bmz') \rightarrow
 \intvL{\subst{\bmc}{\bmz}} \dsep(\bmx, \bmy, \bmz')$.
\end{enumerate}
\end{restatable}

\begin{proof}
Let $w$ be a world such that $\bmx,\bmy \in\var(w)^+$ and $\bmz,\bmz'\in\var(w)^*$.
Recall that a data generator corresponds to a causal diagram that is defined as a directed acyclic graph (DAG) in 
\Sec{sec:model}.
Let $G$ be the causal diagram corresponding to the data generator $\datagen{w}$ in the world $w$.

Then the causal diagram $G\intvE{\subst{\bmc}{\bmx}}$ corresponding to $\datagen{w\intvE{\subst{\bmc}{\bmx}}}$ is obtained by removing all arrows pointing to $\bmx$ in $G$.
Similarly, the causal diagram $G\intvL{\subst{\bmc}{\bmx}}$ corresponding to $\datagen{w\intvL{\subst{\bmc}{\bmx}}}$ is obtained by removing all arrows emerging from $\bmx$ in~$G$.

\begin{enumerate}
\item
Assume that $w \models \intvE{\subst{\bmc}{\bmz}} \dsep(\bmx, \bmy, \bmz)$.
Then $w \intvE{\subst{\bmc}{\bmz}} \models \dsep(\bmx, \bmy, \bmz)$.
Let $p$ be an undirected path between $\bmx$ and $\bmy$ in $G\intvE{\subst{\bmc}{\bmx}}$.
Since $\bmx$ and $\bmy$ are $d$-separated by $\bmz$ in the diagram $G \intvE{\subst{\bmc}{\bmz}}$, we have:
\begin{enumerate}\renewcommand{\labelenumi}{(\alph{enumi})}
\item there is no path $p$ in $G \intvE{\subst{\bmc}{\bmz}}$ that has a chain $v' \garrow v \garrow v''$ s.t. $v\in\bmz$;
\item there is no path $p$ in $G \intvE{\subst{\bmc}{\bmz}}$ that has a fork $v' \leftgarrow v \garrow v''$ s.t. $v\in\bmz$;
\item if $G \intvE{\subst{\bmc}{\bmz}}$ has a path with a collider $v' \garrow v \leftgarrow v''$, then $v\not\in\Anca(\bmz)$.
\end{enumerate}
By (a), if $G$ has an undirected path with a chain $v' \garrow v \garrow v''$, then $v\in\bmz$,
because $v\not\in\bmz$ contradicts (a).

By (b), if $G$ has an undirected path with a fork $v' \leftgarrow v \garrow v''$, then $v\in\bmz$,
because $v\not\in\bmz$ contradicts (b).

Let $p$ be an undirected path in $G \intvE{\subst{\bmc}{\bmz}}$ that has a collider $v' \garrow v \leftgarrow v''$.
By (c), we have $v\not\in\Anca(\bmz)$ in $G \intvE{\subst{\bmc}{\bmz}}$.
Then, $G$ also has the same path $p$, and the arrows connecting with $v$ in $G$ are the same as those in $G \intvE{\subst{\bmc}{\bmz}}$.
Hence, we obtain $v\not\in\Anca(\bmz)$ in $G$.

Therefore, $w \models \dsep(\bmx, \bmy, \bmz)$.

\item 
Assume that $w\models \dsep(\bmx, \bmy, \bmz)$.
Then in the diagram $G$, $\bmx$ and $\bmy$ are $d$-separated by $\bmz$.
By definition, $G\intvE{\subst{\bmc}{\bmx}}$ is the same as $G$ except that it has no arrows pointing to $\bmx$.
Hence, also in $G\intvE{\subst{\bmc}{\bmx}}$, $\bmx$ and $\bmy$ are $d$-separated by $\bmz$.
Therefore, $w\models \intvE{\subst{\bmc}{\bmx}} \dsep(\bmx, \bmy, \bmz)$.

\item
Assume that $w \models \intvL{\subst{\bmc}{\bmz}} \dsep(\bmx, \bmy, \bmz)$.
Then $w \intvL{\subst{\bmc}{\bmz}} \models \dsep(\bmx, \bmy, \bmz)$.
Let $p$ be an undirected path between $\bmx$ and $\bmy$ in $G\intvL{\subst{\bmc}{\bmx}}$.
Since $\bmx$ and $\bmy$ are $d$-separated by $\bmz$ in the diagram $G \intvL{\subst{\bmc}{\bmz}}$, we have:
\begin{enumerate}\renewcommand{\labelenumi}{(\alph{enumi})}
\item there is no path $p$ in $G \intvL{\subst{\bmc}{\bmz}}$ that has a chain $v' \garrow v \garrow v''$ s.t. $v\in\bmz$;
\item there is no path $p$ in $G \intvL{\subst{\bmc}{\bmz}}$ that has a fork $v' \leftgarrow v \garrow v''$ s.t. $v\in\bmz$;
\item if $G \intvL{\subst{\bmc}{\bmz}}$ has a path with a collider $v' \garrow v \leftgarrow v''$, then $v\not\in\Anca(\bmz)$.
\end{enumerate}
By (a), if $G$ has an undirected path with a chain $v' \garrow v \garrow v''$, then $v\in\bmz$,
because $v\not\in\bmz$ contradicts (a).

By (b), if $G$ has an undirected path with a fork $v' \leftgarrow v \garrow v''$, then $v\in\bmz$,
because $v\not\in\bmz$ contradicts (b).

Let $p$ be an undirected path in $G \intvL{\subst{\bmc}{\bmz}}$ that has a collider $v' \garrow v \leftgarrow v''$.
By (c), we obtain $v\not\in\Anca(\bmz)$ in $G \intvL{\subst{\bmc}{\bmz}}$.
Then, $G$ also has the path $p$, and may have additional arrows pointing to $v$ and no arrows pointing from $v$.
Hence, we obtain $v\not\in\Anca(\bmz)$ in $G$.

Therefore, $w \models \dsep(\bmx, \bmy, \bmz)$.

\item 
The proof for Claim~\ref{claim:dsep2} is analogous to that for Claim~\ref{claim:dsep1}.

\item
Assume that $w\models \dsep(\bmx, \bmy, \bmz\cup\bmz')$.
Then in the diagram $G$, $\bmx$ and $\bmy$ are $d$-separated by $\bmz\cup\bmz'$.
By definition, $G\intvL{\subst{\bmc}{\bmz}}$ has no arrows emerging from $\bmz$.

If $G\intvL{\subst{\bmc}{\bmx}}$ has no \emph{undirected} path between $\bmx$ and $\bmy$, then $\bmx$ and $\bmy$ are $d$-separated by $\bmz'$, 
hence $w \models \intvL{\subst{\bmc}{\bmx}}\dsep(\bmx, \bmy, \bmz')$.

Otherwise, let $p$ be an undirected path between $\bmx$ and $\bmy$ in $G\intvL{\subst{\bmc}{\bmx}}$.
Since $\bmx$ and $\bmy$ are $d$-separated by $\bmz\cup\bmz'$, we fall into one of the three cases in Definition~\ref{def:d-separate}.
\begin{enumerate}\renewcommand{\labelenumi}{(\alph{enumi})}
\item
If $p$ has a chain $v' \rightarrow v \rightarrow v''$ s.t. $v\in\bmz\cup\bmz'$, then $v\in\bmz'$, because $G\intvL{\subst{\bmc}{\bmz}}$ has no arrows pointing from $\bmz$.
Hence, $p$ is $d$-separated by $\bmz'$.
\item 
For the same reason as (a), if $p$ has a fork $v' \leftarrow v \rightarrow v''$ s.t. $v\in\bmz\cup\bmz'$, then $v\in\bmz'$.
Hence, $p$ is $d$-separated by $\bmz'$.
\item 
If $p$ has a collider $v' \rightarrow v \leftarrow v''$ s.t. $v\not\in\Anca(\bmz\cup\bmz')$,
then $v\not\in\Anca(\bmz')$.
Thus $p$ is $d$-separated by $\bmz'$\!.
\end{enumerate}
Therefore, $w \models \intvL{\subst{\bmc}{\bmx}}\dsep(\bmx, \bmy, \bmz')$.
\myqed
\end{enumerate}
\end{proof}

\begin{remark}
In contrast with Claim~\ref{claim:dsep3} in Proposition~\ref{prop:d-sep:intv},
there exists a world $w$ s.t. 
\begin{align*}
w\not\models \dsep(\bmx, \bmy, \bmz\cup\bmz') \rightarrow
 \intvE{\subst{\bmc}{\bmz}} \dsep(\bmx, \bmy, \bmz').
\end{align*}
To see this, assume that $w\models \dsep(\bmx, \bmy, \bmz\cup\bmz')$.
Suppose that $w$ has a causal diagram $G$ where there is an undirected path $p$ between $\bmx$ and $\bmy$ that has a fork $v' \leftarrow v \rightarrow v''$ s.t. $v\in\bmz$ and no other variable in $\bmz\cup\bmz'$ appears on $p$.
Then $G\intvE{\subst{\bmc}{\bmz}}$ also has the path $p$, because the intervention $\intvE{\subst{\bmc}{\bmz}}$ removes no arrows in $p$.
Hence, $p$ is $d$-separated by $\bmz$ but not by $\bmz'$ in $G\intvE{\subst{\bmc}{\bmz}}$.
Therefore, $w\not\models \intvE{\subst{\bmc}{\bmz}} \dsep(\bmx, \bmy, \bmz')$.
\end{remark}
~

\subsection{Validity of the Axioms with $\nanc$}
\label{sub:StaCL:ax:nanc}

Here are the axioms of \AXwCP{} with the non-anscestor predicate $\nanc$ and a property \axNancA{}.
\[
\renewcommand{\arraystretch}{1.3}
\begin{array}{l@{\hspace{0ex}}l}
\!\mbox{\axNancA{}\hspace{1ex}}
&~~~ 
\vdashg \nanc(\bmx, \bmy) \rightarrow 
((\bmc' = \bmy) \leftrightarrow \intvE{\subst{\bmc}{\bmx}} (\bmc' = \bmy))
\\
\!\mbox{\axNancAB{}\hspace{1ex}}
&~~~ 
\vdashg ( \nanc(\bmx, \bmy) \land \nanc(\bmx, \bmz) ) \rightarrow 
(f {=} \bmy|_{\bmz} \leftrightarrow \intvE{\subst{\bmc}{\bmx}} (f {=} \bmy|_{\bmz}))
\\
\!\mbox{\axNancB{}\hspace{1ex}}
&~~~ 
\vdashg \nanc(\bmx, \bmy) \leftrightarrow  \intvE{\subst{\bmc}{\bmx}} \nanc(\bmx, \bmy)
\\
\!\mbox{\axNancC{}\hspace{1ex}}
&~~~ 
\vdashg \nanc(\bmx, \bmy) \rightarrow \intvE{\subst{\bmc}{\bmx}} \dsep(\bmx, \bmy, \emptyset)
\\
\!\mbox{\axNancD{}\hspace{1ex}}
&~~~ 
\vdashg (\nanc(\bmx, \bmz) \land \dsep(\bmx, \bmy, \bmz)) \rightarrow \nanc(\bmx, \bmy)
\end{array}
\]
Concerning $\nanc$, the axioms \axNancAB{} to \axNancD{} are sufficient for us to derive the rules of Pearl's do-calculus.

\begin{restatable}[Validity of axioms with $\nanc$]{prop}{PropNancIntv}
\label{prop:nanc:intv}
Let $\bmx,\bmy,\bmz\in\cvar^+$ be disjoint,
$\bmc\in\dConst^+$, 
$\bmc' \in \Const^+$,
and $f\in\Func$.
\begin{enumerate}
\item\!\label{item:basic:intvE:c}%
\axNancA{}~\hspace{2ex}~ \\
$\models \nanc(\bmx, \bmy) \rightarrow 
(\bmc' = \bmy \leftrightarrow \intvE{\subst{\bmc}{\bmx}} (\bmc' = \bmy))$.
\item\!\label{item:basic:intvE:cd}%
\axNancAB{}~\hspace{2ex}~ \\
$\models ( \nanc(\bmx, \bmy) \land \nanc(\bmx, \bmz) ) \,{\rightarrow}\, 
(f {=} \bmy|_{\bmz} \leftrightarrow \intvE{\subst{\bmc}{\bmx}} (f {=} \bmy|_{\bmz}))$.
\item \label{claim:nanc-intv}
\axNancB{}~\hspace{2ex}~ \\
$\models \nanc(\bmx, \bmy) \leftrightarrow  \intvE{\subst{\bmc}{\bmx}} \nanc(\bmx, \bmy)$.
\item \label{claim:nanc-dsep}
\axNancC{}~\hspace{2ex}~ \\
$\models \nanc(\bmx, \bmy) \rightarrow \intvE{\subst{\bmc}{\bmx}} \dsep(\bmx, \bmy, \emptyset)$.
\item \label{claim:nanc}
\axNancD{}~\hspace{2ex}~ \\
$\models (\nanc(\bmx, \bmz) \land \dsep(\bmx, \bmy, \bmz)) \rightarrow \nanc(\bmx, \bmy)$.
\end{enumerate}
\end{restatable}

\begin{proof}
Let $w = (\datagen{w}, \semf_{w}, \mem{w})$
be a world such that $\bmx, \bmy, \bmz \in \var(w)^+$.
\begin{enumerate}
\item\!%
Assume that $w\models \nanc(\bmx, \bmy)$.
Let $G_{w}$ be the causal diagram corresponding to the data generator $\datagen{w}$.
Then $\bmx\cap\Anc(\bmy) = \emptyset$ in $G_{w}$.
This means that the value of $\bmy$ does not depend on that of $\bmx$.
Thus we obtain:
\begin{align*}
\mem{w\intvE{\subst{\bmc}{\bmx}}}(\bmy)
&=
\sema{\datagen{w\intvE{\subst{\bmc}{\bmx}}}(\bmy)}{w\intvE{\subst{\bmc}{\bmx}}}
\\&=
\sema{\datagen{w\intvE{\subst{\bmc}{\bmx}}}(\bmy)}{w}
& \hspace{-6ex}\text{(by $\xi_{w\intvE{\subst{\bmc}{\bmx}}} = \xi_{w}$)}
\\&=
\sema{\datagen{w}(\bmy)}{w}
& \hspace{-6ex}\text{(by $\datagen{w\intvE{\subst{\bmc}{\bmx}}}(\bmy) = \datagen{w}(\bmy)$)}
\\&=
\mem{w}(\bmy).
\end{align*}

Thus, 
$w \models \bmc' = \bmy$ iff  $w\intvE{\subst{\bmc}{\bmx}} \models \bmc' = \bmy$.
Therefore, 
$w \models \bmc' {=} \bmy \leftrightarrow \intvE{\subst{\bmc}{\bmx}} (\bmc' {=} \bmy)$.

\item\!%
Let $\nA, \nB \in \Const$.
Assume that $w \models \nanc(\bmx, \bmy) \land \nanc(\bmx, \bmz)$.
Then $w \models \nanc(\bmx, \bmy\joint\bmz)$.
By Claim~\ref{item:basic:intvE:c},
$w \models \nA = \bmy\joint\bmz \leftrightarrow \intvE{\subst{\bmc}{\bmx}} (\nA = \bmy\joint\bmz)$
and
$w \models \nB = \bmz \leftrightarrow \intvE{\subst{\bmc}{\bmx}} (\nB = \bmz)$.
By $\sema{\bmy\joint\bmz}{w\intvE{\subst{\bmc}{\bmx}}} = \sema{\bmy\joint\bmz}{w}$ and 
$\sema{\bmz}{w\intvE{\subst{\bmc}{\bmx}}} = \sema{\bmz}{w}$,
we have
$\sema{\bmy|_{\bmz}}{w\intvE{\subst{\bmc}{\bmx}}} = \sema{\bmy|_{\bmz}}{w}$.
Therefore,
$w \models f = \bmy|_{\bmz} \leftrightarrow \intvE{\subst{\bmc}{\bmx}} (f = \bmy|_{\bmz})$.

\item
We show the direction from left to right as follows.
Assume that $w\models \nanc(\bmx, \bmy)$.
Then, in the diagram $G$, all variables in $\bmx$ are non-ancestors of the variables in $\bmy$;
i.e., $G$ has no \emph{directed} path from $\bmx$ to $\bmy$.
Since the eager intervention $\intvE{\subst{\bmc}{\bmx}}$ removes only arrows pointing to $\bmx$,\, $G\intvE{\subst{\bmc}{\bmx}}$ still has no \emph{directed} path from $\bmx$ to $\bmy$.
Therefore, $w\models \intvE{\subst{\bmc}{\bmx}} \nanc(\bmx, \bmy)$.

The other direction is shown in a similar way, since the eager intervention $\intvE{\subst{\bmc}{\bmx}}$ only remove arrows pointing to~$\bmx$.

\item 
Assume that $w\models \nanc(\bmx, \bmy)$.
By Claim~\ref{claim:nanc-intv}, 
$w\models \intvE{\subst{\bmc}{\bmx}} \nanc(\bmx, \bmy)$, hence
$w \intvE{\subst{\bmc}{\bmx}} \models \nanc(\bmx, \bmy)$.
Then, in the diagram $G\intvE{\subst{\bmc}{\bmx}}$, all variables in $\bmx$ are non-ancestors of the variables in $\bmy$;
i.e., $G\intvE{\subst{\bmc}{\bmx}}$ has no \emph{directed} path from $\bmx$ to $\bmy$.

Suppose that $G\intvE{\subst{\bmc}{\bmx}}$ has no \emph{undirected} path between $\bmx$ and $\bmy$.
By Definition~\ref{def:d-separate}, $\bmx$ and $\bmy$ are $d$-separated by $\emptyset$, namely, they are independent.
Hence, $w\intvE{\subst{\bmc}{\bmx}} \models \dsep(\bmx, \bmy, \emptyset)$.
Therefore, $w \models \intvE{\subst{\bmc}{\bmx}}\dsep(\bmx, \bmy, \emptyset)$.

Suppose that $G\intvE{\subst{\bmc}{\bmx}}$ has some undirected path $p$ between $\bmx$ and $\bmy$.
By the definition of the eager intervention, $G\intvE{\subst{\bmc}{\bmx}}$ has no arrows pointing to $\bmx$, hence has arrows pointing from $\bmx$.
On the other hand, since $G\intvE{\subst{\bmc}{\bmx}}$ has no \emph{directed} path from $\bmx$ to $\bmy$, $p$ is not directed.
Thus, $p$ has a collider node $v$; i.e., it is of the form $\bmx \rightarrow \cdots \rightarrow v \leftarrow \cdots \bmy$.
Then, by (c) in Definition~\ref{def:d-separate}, $p$ is $d$-separated by $\emptyset$;
namely, $w\intvE{\subst{\bmc}{\bmx}} \models \dsep(\bmx, \bmy, \emptyset)$.
Therefore, $w \models \intvE{\subst{\bmc}{\bmx}}\dsep(\bmx, \bmy, \emptyset)$.

\item 
We show the contraposition as follows.
Assume that 
$w\models \neg \nanc(\bmx, \bmy) \land \dsep(\bmx, \bmy, \bmz)$.
Then it is sufficient to prove $w\models \neg \nanc(\bmx, \bmz)$.

Recall the definition in%
\conference{ \Sec{sub:appendix:def:CP}.}%
\arxiv{ Appendix~\ref{sub:appendix:def:CP}.}
By assumption, $\bmx \cap \Anc(\bmy) \neq \emptyset$.
Then there are $x_0\in\bmx$ and $y_0\in\bmy$ s.t. $x_0$ is an ancestor of $y_0$, i.e., $x_0\in\Anc(y_0)$.
Then there exists a directed path from $x_0$ to $y_0$.
Let $p$ be a directed path from $x_0$ to $y_0$.
By $w\models \dsep(\bmx, \bmy, \bmz)$,
$p$ is $d$-separated by $\bmz$.
By (a) of Definition~\ref{def:d-separate}, there is a variable $z_0\in\bmz$ on $p$,
hence $x_0\in\Anc(z_0)$.
Therefore, $\bmx \cap \Anc(\bmz) \neq \emptyset$, i.e.,
$w\models \neg \nanc(\bmx, \bmz)$.
\myqed
\end{enumerate}
\end{proof}

\subsection{Validity of the Axioms with Other Causal Predicates}
\label{sub:StaCL:ax:CP}

Here are the axioms of \AXwCP{} that replace $\allnanc$ with $\nanc$ and $\pa$ with $\nanc$ or $\dsep$.
\[
\renewcommand{\arraystretch}{1.3}
\begin{array}{l@{\hspace{0ex}}l}
\!\mbox{\axNancAll{}\hspace{1ex}}
&~~~ \vdashg \allnanc(\bmx, \bmy, \bmz) \rightarrow \nanc(\bmx, \bmz)
\\
\!\mbox{\axPaToNanc{}\hspace{1ex}}
&~~~ \vdashg \pa(\bmx, \bmy) \rightarrow \nanc(\bmy, \bmx)
\\
\!\mbox{\axPaToDsep{}\hspace{1ex}}
&~~~ \vdashg \pa(\bmz, \bmx) \rightarrow
\intvL{\subst{\bmc}{\bmx}} \dsep(\bmx, \bmy, \bmz)
\end{array}
\]
We prove the validity of these axioms as follows.

\begin{restatable}[Validity of axioms with other causal predicates]{prop}{PropCPIntv}
\label{prop:CP:intv}
Let $\bmx,\bmy,\bmz\in\cvar^+$ be disjoint,
and $\bmc\in\dConst^+$.
\begin{enumerate}
\item \label{claim:allnanc-nanc}
\axNancAll{}~\hspace{2ex}~
$\models \allnanc(\bmx, \bmy, \bmz) \rightarrow \nanc(\bmx, \bmz)$.
\item \label{claim:pa2anc}
\axPaToNanc{}~\hspace{2ex}~
$\models \pa(\bmx, \bmy) \rightarrow \nanc(\bmy, \bmx)$.
\item \label{claim:pa-dsep}
\axPaToDsep{}~\hspace{2ex}~
$\models \pa(\bmz, \bmx) \rightarrow
 \intvL{\subst{\bmc}{\bmx}} \dsep(\bmx, \bmy, \bmz)$.
\end{enumerate}
\end{restatable}

\begin{proof}
Let $w = (\datagen{w}, \semf_{w}, \mem{w})$
be a world such that $\bmx, \bmy, \bmz \in \var(w)^+$.
\begin{enumerate}
\item
This claim is straightforward from the definitions of the semantics of $\allnanc$ and $\nanc$.

\item
This claim is straightforward from \Propo{prop:relation:c-predicates}.

\item 
Assume that $w\models \pa(\bmz, \bmx)$.
Then, in the diagram $G$, $\bmz$ is the set of all variables pointing to $\bmx$.

If $G\intvL{\subst{\bmc}{\bmx}}$ has no \emph{undirected} path between $\bmx$ and $\bmy$, then $\bmx$ and $\bmy$ are $d$-separated by $\bmz$, 
hence $w \models \intvL{\subst{\bmc}{\bmx}}\dsep(\bmx, \bmy, \bmz)$.

Otherwise, let $p$ be an undirected path between $\bmx$ and $\bmy$ in $G\intvL{\subst{\bmc}{\bmx}}$.
By definition, $G\intvL{\subst{\bmc}{\bmx}}$ has no arrows emerging from $\bmx$.
Since $\bmz$ is the set of all variables pointing to~$\bmx$,\, $p$ has:
\begin{itemize}
\item either a chain $x_0 \leftarrow z_0 \leftarrow v$ s.t. $x_0\in\bmx$, $z_0\in\bmz$, and $v\not\in \bmx\cup\bmz$,
\item or a fork $x_0 \leftarrow z_0 \rightarrow v$ s.t. $x_0\in\bmx$, $z_0\in\bmz$, and $v\not\in \bmx\cup\bmz$.
\end{itemize}
Thus, $p$ is $d$-separated by $z_0$.
Hence, in $G\intvL{\subst{\bmc}{\bmx}}$,\, $\bmx$ and $\bmy$ are $d$-separated by $\bmz$.
Therefore, $w\models \intvL{\subst{\bmc}{\bmx}} \dsep(\bmx, \bmy, \bmz)$.
\myqed
\end{enumerate}
\end{proof}

\section{Details of the Derivation of the Do-Calculus Rules Using \AXwCP{}}
\label{sec:appendix:StaCL:reasoning}

In this section, we formalize and derive the three rules of Pearl's do-calculus~\cite{Pearl:95:biometrika} using our statistical causal language (\StaCL{}).

By Proposition~\ref{prop:causal}, \StaCL{} formulas correspond to the do-calculus notations as follows.
\begin{itemize}
\item
$\intvE{\subst{\bmc}{\bmx}} (\bmc' \,{=}\, \bmy)$ describes the post-intervention distribution $P_{\diag{w}}( \bmy \,|\, \allowbreak do(\bmx\,{=}\,\bmc) )$ of $\bmy$.
For instance, given a world $w$,\, $w \models \intvE{\subst{\bmc}{\bmx}} \bmc' \,{=}\, \bmy$ represents 
$P_{\diag{w}}( \bmy \,|\,  \allowbreak do(\bmx\,{=}\,\bmc) ) \,{=}\, \sema{\bmc'}{w}$.
Note that the $do(\bmx\,{=}\,\bmc)$ operation is expressed as the eager intervention $\intvE{\subst{\bmc}{\bmx}}$ in our formulation.
\item
$\intvE{\subst{\bmc}{\bmx}} (f \,{=}\, \bmy|_{\bmz})$ describes the post-intervention conditional distribution $P_{\diag{w}}( \bmy \,|\,  \allowbreak do(\bmx\,{=}\,\bmc), \bmz)$ of $\bmy$ given $\bmz$.
Note that the conditioning on $\bmz$ takes place after the intervention $do(\bmx\,{=}\,\bmc)$ is performed.
\end{itemize}

To formalize the rules of the do-calculus, 
we denote the set of all \emph{conditioning variables} appearing in a formula $\phi$ by:
\begin{align*}
\cdv{\phi} 
=&~
 \{ \bmz \,\mid\, \bmy|_{\bmz} \in \fv{\phi}\cap\fvar \}
\cup \{ (\bmz\joint\bmx)|_{\bmx=\bmc} \,\mid\, \bmy|_{\bmz\!,\bmx=\bmc} \in \fv{\phi}\cap\fvar \}.
\end{align*}

Now we formalize the rules of Pearl's do-calculus in Proposition~\ref{prop:do-calculus}.
After that, we explain the meaning of these rules.

\arxiv{%
\PropDoCalculus*
\vspace{1ex}
}%
\conference{%
\setcounter{prop}{1}
\begin{restatable}[Do-calculus rules]{prop}{PropDoCalculusAPP}
Let $\bmv, \bmx, \bmy,\bmz\in\cvar^+$ be disjoint,
$\bmxB, \bmxC \in \cvar^+$,
and $\bmcA,\bmcB,\bmcC\in\dConst^+$.
\begin{enumerate}
\item 
\axDoA{}.
~\hspace{0ex}~
Introduction/elimination of conditioning:
\begin{align*}
\hspace{-3ex}
\vdashg
& \intvE{\subst{\bmcA}{\bmv}} 
 ( \dsep(\bmx, \bmy, \bmz) \land {\textstyle \bigwedge_{\bms \in S}}\, \pos(\bms) ) \\
& \hspace{0.0ex}\!\rightarrow 
( (\intvE{\subst{\bmcA}{\bmv}} \phi_0 )
  \leftrightarrow
  \intvE{\subst{\bmcA}{\bmv}} \phi_1)
\end{align*} 
where $\phi_1$ is obtained by replacing some occurrences of $\bmy|_{\bmz}$ in $\phi_0$ with $\bmy|_{\bmz,\bmx}$,
and $S = \cdv{\phi_0} \cup \cdv{\phi_1}$;
\item 
\axDoB{}.
~\hspace{0ex}~
Exchange between intervention and conditioning:
\begin{align*}
\hspace{-3ex}
\vdashg
& \intvE{\subst{\bmcA}{\bmv}} \intvL{\subst{\bmcB}{\bmx}} 
 ( \dsep(\bmx, \bmy, \bmz) \land {\textstyle \bigwedge_{\bms \in S}}\, \pos(\bms) ) \\
& \hspace{0ex}\!\rightarrow\!%
( ( \intvE{\subst{\bmcA}{\bmv},\subst{\bmcB}{\bmx}} \phi_0 )
  \leftrightarrow
  \intvE{\subst{\bmcA}{\bmv}} \phi_1 )
\end{align*}
where $\phi_1$ is obtained by replacing every occurrence of $\bmy|_{\bmz}$ in $\phi_0$ with $\bmy|_{\bmz,\bmx=\bmcB}$,
and $S = \cdv{\phi_0} \cup \cdv{\phi_1}$;
\item 
\axDoC{}
~\hspace{0ex}~
Introduction/elimination of intervention:
\begin{align*}
\hspace{-3ex}
\vdashg
& 
\intvE{\subst{\bmcA}{\bmv}} 
(\allnanc(\bmxB, \bmx, \bmy) \land
\intvE{\subst{\bmcB}{\bmxB}} 
 ( \dsep(\bmx, \bmy, \bmz) \land \pos(\bmz) ) ) \\
& \hspace{0ex}\!\rightarrow 
( ( \intvE{\subst{\bmcA}{\bmv}} \phi)
  \leftrightarrow
  \intvE{\subst{\bmcA}{\bmv},\subst{\bmcB}{\bmxB},\subst{\bmcC}{\bmxC}} \phi )
\end{align*}
where 
$\fv{\phi} = \{ \bmy|_{\bmz} \}$
and $\bmx \eqdef \bmxB\joint\bmxC$.
\end{enumerate}
\end{restatable}
}%

We explain these three rules as follows.
\begin{enumerate}
\item The first rule allows for adding/removing the conditioning on $\bmx$ when $\bmx$ and $\bmy$ are $d$-separated by $\bmz$ (hence when they are conditionally independent given $\bmz$).

In the do-calculus, this is expressed by:
\begin{align*}
P( \bmy \mid do(\bmv), \bmz )
= P( \bmy \mid do(\bmv), \bmz, \bmx )
\\ \mbox{ if } 
(\bmx \indep \bmy \mid \bmv, \bmz )_{G_{\ov{\bmv}}}
\end{align*}
where 
\begin{itemize}
\item $G_{\ov{\bmv}}$ is the diagram obtained by deleting all arrows pointing to nodes in~$\bmv$;
\item $(\bmx \indep \bmy \mid \bmv, \bmz )_{G_{\ov{\bmv}}}$ represents that $\bmx$ and $\bmy$ are $d$-separated by $\bmv\cup\bmz$ in the causal diagram $G_{\ov{\bmv}}$.
\end{itemize}
In our formulation, the deletion of arrows pointing to $\bmv$ is expressed by the eager intervention $\intvE{\subst{\bmcA}{\bmv}}$.

\item The second rule represents that the conditioning on $\bmx$ and the intervention to $\bmx$ result in the same conditional distribution of $\bmy$ given $\bmz$ under the condition that all back-door paths from $\bmx$ to $\bmy$ (Definition~\ref{def:back-path}) are $d$-separated by $\bmv\cup\bmz$ (Definition~\ref{def:d-separate}).\footnote{This condition is denoted by $\intvL{\subst{\bmcB}{\bmx}} \dsep(\bmx, \bmy, \bmz\cup\bmv)$ and follows from Proposition~\ref{prop:d-sep:intv} and $\dsep(\bmx, \bmy, \bmz\cup\bmv)$.}

In the do-calculus, this is expressed by:
\begin{align*}
P( \bmy \mid do(\bmv), \bmx, \bmz )
= P( \bmy \mid do(\bmv), do(\bmx), \bmz )
\\ \mbox{ if } 
(\bmx \indep \bmy \mid \bmz, \bmv )_{G_{\ov{\bmv}\ud{\bmx}}}
\end{align*}
where $G_{\ov{\bmv}\ud{\bmx}}$ is the diagram obtained by deleting all arrows pointing to nodes in $\bmv$ and deleting all arrows emerging from nodes in $\bmx$.

In our formulation, the ``upper manipulation'' $\ov{\bmv}$ is expressed by the eager intervention $\intvE{\subst{\bmcA}{\bmv}}$ whereas the ``lower-manipulation'' $\ud{\bmx}$ is expressed by the lazy intervention $\intvL{\subst{\bmcB}{\bmx}}$.

Recall that the lazy intervention $\intvL{\subst{\bmcB}{\bmx}}$ removes all arrows emerging from $\bmx$, and hence preserves only back-door paths from $x$ to $y$ while removing all other undirected paths between $x$ and $y$ (Remark~\ref{rem:backdoor-path}).
Thus, $\intvL{\subst{\bmcB}{\bmx}} \dsep(\bmx, \bmy, \bmz)$ represents that all back-door paths from $\bmx$ to $\bmy$ are $d$-separated by $\bmz$.

\item The third rule allows for adding/removing the intervention to $\bmx$ without changing the conditional probability distribution of $\bmy$ given $\bmz$ under a certain condition.

In the do-calculus, this is expressed by:
\begin{align*}
P( \bmy \mid do(\bmv), \bmz )
= P( \bmy \mid do(\bmv), do(\bmx), \bmz )
\\ \mbox{ if } 
(\bmx \indep \bmy \mid \bmz, \bmv )_{G_{\ov{\bmv}\ov{\bmx\setminus\Anc(\bmz)}}}
\end{align*}
where $G_{\ov{\bmv}\ov{\bmx\setminus\Anc(\bmz)}}$ is the diagram obtained by deleting all arrows pointing to nodes in $\bmv$ and then deleting those in $\bmx\setminus\Anc(\bmz)$.
\end{enumerate}
~

Now, we derive these three rules using \AXwCP{} as follows.

\begin{proof}
Let $w$ be a world such that $\bmv, \bmx, \bmy, \bmz, \bmxB,\bmxC \in\cvar(w)^+$.
\begin{enumerate}
\item
We prove the first claim as follows.
Let $\psiPre \eqdef \dsep(\bmx, \bmy, \bmz) \land \bigwedge_{\bms\in S} \pos(\bms)$.
Then:
\begin{align}
&
\mbox{By \axDsepCIndB},
\nonumber \\&
\label{eq:claim1:CIndA}
\vdash_{\datagen{}\intvE{\subst{\bmcA}{\bmv}}}
 ( \dsep(\bmx, \bmy, \bmz) \land \pos(\bmz) ) \rightarrow
 ( \bmy|_{\bmz,\bmx=\bmcB} = \bmy|_{\bmz} )
\\&
\mbox{By \eqref{eq:claim1:CIndA}, \axEqB, \axPT, \axMP},
\nonumber \\&
\label{eq:claim1:EqB}
\vdash_{\datagen{}\intvE{\subst{\bmcA}{\bmv}}}
\psiPre \rightarrow (\phi_0 \leftrightarrow \phi_1)
\\&
\mbox{By \eqref{eq:claim1:EqB}, \axDGEI, \axMP},
\nonumber \\&
\label{eq:claim1:DGEI}
\vdashg \intvE{\subst{\bmcA}{\bmv}} (\psiPre \rightarrow 
(\phi_0 \leftrightarrow \phi_1) )
\\&
\mbox{By \eqref{eq:claim1:DGEI}, \axDistrE$^{\land}$, \axDistrE$^{\rightarrow}$, \axPT, \axMP},
\nonumber \\&
\vdashg ( \intvE{\subst{\bmcA}{\bmv}} \psiPre ) \rightarrow 
( (\intvE{\subst{\bmcA}{\bmv}} \phi_0 )
\leftrightarrow
\intvE{\subst{\bmcA}{\bmv}} \phi_1 )
\nonumber 
{.}
\end{align}
Therefore, Claim (1) follows.
~\\~

\item
We prove the second claim as follows.
Let $\psiPre \eqdef \dsep(\bmx, \bmy, \bmz) \land \bigwedge_{\bms\in S} \pos(\bms)$.
Then:
\begin{align}
&
\mbox{By \axDsepCIndB},
\nonumber \\&
\label{eq:claim2:CIndB}
\vdash_{\datagen{}\intvE{\subst{\bmcA}{\bmv}} \intvL{\subst{\bmcB}{\bmx}}}
( \dsep(\bmx, \bmy, \bmz) \land \pos(\bmz)) \rightarrow
( \bmy|_{\bmz,\bmx=\bmcB} = \bmy|_{\bmz} )
\\&
\mbox{By \eqref{eq:claim2:CIndB}, \axDGEI, \axMP},
\nonumber \\&
\label{eq:claim2:DistrE}
\vdash_{\datagen{}\intvE{\subst{\bmcA}{\bmv}}}
\intvL{\subst{\bmcB}{\bmx}}
\big( (\dsep(\bmx, \bmy, \bmz) \land \pos(\bmz)) \rightarrow
( \bmy|_{\bmz,\bmx=\bmcB} = \bmy|_{\bmz} ) \big)
\\&
\mbox{By \eqref{eq:claim2:DistrE}, \axDistrE$^{\rightarrow}$, \axMP},
\nonumber \\&
\nonumber 
\vdash_{\datagen{}\intvE{\subst{\bmcA}{\bmv}}}
( \intvL{\subst{\bmcB}{\bmx}} (\dsep(\bmx, \bmy, \bmz) \land \pos(\bmz)) )
\\&\hspace{8ex}
\label{eq:claim2:EqFunc}
\rightarrow
( \intvL{\subst{\bmcB}{\bmx}} \bmy|_{\bmz,\bmx=\bmcB} = \bmy|_{\bmz} ) \\&
\mbox{By \eqref{eq:claim2:EqFunc}, \axEqFunc, \axPT, \axMP},
\nonumber \\&
\nonumber 
\vdash_{\datagen{}\intvE{\subst{\bmcA}{\bmv}}}
( \intvL{\subst{\bmcB}{\bmx}} (\dsep(\bmx, \bmy, \bmz) \land \pos(\bmz)) )
\\&\hspace{8ex}
\label{eq:claim2:CondL}
\rightarrow
( \intvL{\subst{\bmcB}{\bmx}} f= \bmy|_{\bmz,\bmx=\bmcB} \leftrightarrow \intvL{\subst{\bmcB}{\bmx}} f = \bmy|_{\bmz} )
\\&
\mbox{By \eqref{eq:claim2:CondL}, \axCondL{}, \axPT, \axMP},
\nonumber \\&
\nonumber 
\vdash_{\datagen{}\intvE{\subst{\bmcA}{\bmv}}}
( \intvL{\subst{\bmcB}{\bmx}} (\dsep(\bmx, \bmy, \bmz) \land \pos(\bmz)) )
\\&\hspace{8ex}
\label{eq:claim2:XcdEL}
\rightarrow
( f= \bmy|_{\bmz,\bmx=\bmcB} \leftrightarrow \intvL{\subst{\bmcB}{\bmx}} f = \bmy|_{\bmz} )
\\&
\mbox{By \eqref{eq:claim2:XcdEL}, \axXcdEL{}, \axPT, \axMP},
\nonumber \\&
\nonumber 
\vdash_{\datagen{}\intvE{\subst{\bmcA}{\bmv}}}
( \intvL{\subst{\bmcB}{\bmx}} (\dsep(\bmx, \bmy, \bmz) \land \pos(\bmz)) )
\\&\hspace{8ex}
\label{eq:claim2:543}
\rightarrow
( f= \bmy|_{\bmz,\bmx=\bmcB} \leftrightarrow \intvE{\subst{\bmcB}{\bmx}} f = \bmy|_{\bmz} )
\\&
\mbox{By \eqref{eq:claim2:543}, \axEqB{}, \axPT, \axMP},
\nonumber \\&
\vdash_{\datagen{}\intvE{\subst{\bmcA}{\bmv}}}
( \intvL{\subst{\bmcB}{\bmx}} \psiPre )
\label{eq:claim2:DGEI}
\rightarrow
( (\intvE{\subst{\bmcB}{\bmx}} \phi_0) \leftrightarrow \phi_1 )
\\&
\mbox{By \eqref{eq:claim2:DGEI}, \axDGEI, \axMP},
\nonumber \\&
\vdashg
\intvE{\subst{\bmcA}{\bmv}} \big( ( \intvL{\subst{\bmcB}{\bmx}} \psiPre )
\rightarrow
\label{eq:claim2:559}
( ( \intvE{\subst{\bmcB}{\bmx}} \phi_0) \leftrightarrow \phi_1 ) \big)
\\&
\mbox{By \eqref{eq:claim2:559}, \axDistrE$^{\rightarrow}$, \axPT, \axMP},
\nonumber \\&
\vdashg
( \intvE{\subst{\bmcA}{\bmv}} \intvL{\subst{\bmcB}{\bmx}} \psiPre )
\rightarrow
\label{eq:claim2:SimuIE}
( (\intvE{\subst{\bmcA}{\bmv}}\intvE{\subst{\bmcB}{\bmx}} \phi_0) \leftrightarrow \intvE{\subst{\bmcA}{\bmv}} \phi_1 )
\\&
\mbox{By \eqref{eq:claim2:SimuIE}, \axSimulE, \axMP},
\nonumber \\&
\vdashg
( \intvE{\subst{\bmcA}{\bmv}} \intvL{\subst{\bmcB}{\bmx}} \psiPre )
\rightarrow
( (\intvE{\subst{\bmcA}{\bmv}, \subst{\bmcB}{\bmx}} \phi_0) \leftrightarrow \intvE{\subst{\bmcA}{\bmv}} \phi_1 )
\nonumber 
{.}
\end{align}
Therefore, Claim (2) follows.
~\\~

\item
We prove the third claim as follows.
Let $f\in\Func$,
$\psiDP \eqdef \dsep(\bmx, \bmy, \bmz) \land \pos(\bmz)$,
$\psiPre \eqdef 
 \allnanc(\bmxB, \bmx, \bmy) \land \intvE{\subst{\bmcB}{\bmxB}} \psiDP$,
and
$\psiDoB \eqdef \intvL{\subst{\bmcC}{\bmx_2}} ( \dsep(\bmx_2, \bmy, \bmz) \land \pos(\bmz) )$.
Let $\dgen{}_0 \eqdef \datagen{}\intvE{\subst{\bmcA}{\bmv}}$.
Then:
\begin{align}
&
\mbox{By \axPT{}, \axMP{}},
\nonumber \\ &
\label{eq:rule3:PT}
\vdashgp \psiPre \rightarrow \intvE{\subst{\bmcB}{\bmx_1}} \psiDP
\\ &
\mbox{By \axDsepDc{}, \axPT{}, \axMP{}},
\nonumber \\ &
\label{eq:rule3:asp1}
\vdash_{\dgen{}_0 \intvE{\subst{\bmcB}{\bmxB}}}
\psiDP \rightarrow ( \dsep(\bmxB, \bmy, \bmz) \land \pos(\bmz) )
\\ &
\label{eq:rule3:asp2}
\vdash_{\dgen{}_0 \intvE{\subst{\bmcB}{\bmxB}}}
\psiDP \rightarrow ( \dsep(\bmxC, \bmy, \bmz) \land \pos(\bmz) )
\\ &
\mbox{By \axDsepCIndB{}, \axMP{},}
\nonumber \\ &
\label{eq:rule3:679}
\vdash_{\dgen{}_0 \intvE{\subst{\bmcB}{\bmx_1}}}
( \dsep(\bmxC, \bmy, \bmz) \land \pos(\bmz) )
\rightarrow ( \bmy|_{\bmz,\bmxC=\bmcC} = \bmy|_{\bmz} )
\\ &
\mbox{By \axEqFunc{}, \axPT{}, \axMP{},}
\nonumber \\ &
\label{eq:rule3:686}
\vdash_{\dgen{}_0 \intvE{\subst{\bmcB}{\bmx_1}}}
( f_0 = \bmy|_{\bmz} )
\\ &
\mbox{By \axDGEI{}, \axMP{},}
\nonumber \\ &
\label{eq:rule3:692}
\vdashgp \intvE{\subst{\bmcB}{\bmx_1}} ( f_0 = \bmy|_{\bmz} )
\\ &
\mbox{By \eqref{eq:rule3:asp2}, \axDsepLNB{},\axPT{},\axMP{},}
\nonumber \\ &
\label{eq:rule3:intvE-dsep}
\vdash_{\dgen{}_0 \intvE{\subst{\bmcB}{\bmx_1}}}
\psiDP \rightarrow \intvL{\subst{\bmcC}{\bmx_2}} 
( \dsep(\bmx_2, \bmy, \bmz) \land \pos(\bmz) )
\\ &
\mbox{By \axNancAll{}, \axPT{}, \axMP{},}
\nonumber \\ &
\label{eq:proof:rule3:NancAll}
\vdashgp \psiPre \rightarrow \nanc(\bmx_1,\bmz)
\\ &
\mbox{By \eqref{eq:proof:rule3:NancAll}, \axNancB{}, \axPT{}, \axMP{}},
\nonumber \\ &
\label{eq:proof:rule3:nancB}
\vdashgp \psiPre \rightarrow \intvE{\subst{\bmcB}{\bmx_1}}\nanc(\bmx_1,\bmz)
\\ &
\mbox{By \axNancD{}, \axPT{}, \axMP{}},
\nonumber \\ &
\vdash_{\dgen{}_0 \intvE{\subst{\bmcB}{\bmx_1}}}
( \dsep(\bmx_1, \bmy, \bmz) \land \nanc(\bmx_1,\bmz) )
\nonumber \\ &
\label{eq:proof:rule3:nanc:x1z}
 \hspace{9ex} \rightarrow ( \nanc(\bmx_1,\bmy) \land \nanc(\bmx_1,\bmz) )
\\ &
\mbox{By \eqref{eq:proof:rule3:nanc:x1z}, \axDGEI{}, \axMP{}},
\nonumber \\ &
\vdashgp \intvE{\subst{\bmcB}{\bmx_1}}
( ( \dsep(\bmx_1, \bmy, \bmz) \land \nanc(\bmx_1,\bmz) )
\nonumber \\ &
\label{eq:proof:rule3:708}
 \hspace{9ex} \rightarrow ( \nanc(\bmx_1,\bmy) \land \nanc(\bmx_1,\bmz) ) )
\\ &
\mbox{By \eqref{eq:proof:rule3:708}, \axDistrE$^{\rightarrow}$, \axDistrE$^{\land}$, \axMP},
\nonumber \\ &
\label{eq:proof:rule3:713}
\vdashgp ( 
( \intvE{\subst{\bmcB}{\bmx_1}} \dsep(\bmx_1, \bmy, \bmz) \land \intvE{\subst{\bmcB}{\bmx_1}} \nanc(\bmx_1,\bmz) ) )
\nonumber \\ &
 \hspace{3.5ex} \rightarrow 
 ( \intvE{\subst{\bmcB}{\bmx_1}} \nanc(\bmx_1,\bmy) \land \intvE{\subst{\bmcB}{\bmx_1}} \nanc(\bmx_1,\bmz) )
\\ &
\mbox{By \eqref{eq:proof:rule3:713}, \axNancB{}, \axPT{}, \axMP{}},
\nonumber \\ &
\label{eq:proof:rule3:722}
\vdashgp ( 
( \intvE{\subst{\bmcB}{\bmx_1}} \dsep(\bmx_1, \bmy, \bmz) \land \intvE{\subst{\bmcB}{\bmx_1}} \nanc(\bmx_1,\bmz) ) )
\nonumber \\ &
 \hspace{3.5ex} \rightarrow 
 ( \nanc(\bmx_1,\bmy) \land \nanc(\bmx_1,\bmz) )
\\ &
\mbox{By \eqref{eq:proof:rule3:722}, \axNancAB{}, \axPT{}, \axMP{}},
\nonumber \\ &
\vdashgp ( 
( \intvE{\subst{\bmcB}{\bmx_1}} \dsep(\bmx_1, \bmy, \bmz) \land \intvE{\subst{\bmcB}{\bmx_1}} \nanc(\bmx_1,\bmz) ) )
\nonumber \\ &
 \hspace{3.5ex} \rightarrow 
( (f_1 = \bmy|_{\bmz}) \leftrightarrow \intvE{\subst{\bmcB}{\bmx_1}} (f_1 = \bmy|_{\bmz}) )
\label{eq:proof:rule3:nancAB}
\\ &
\mbox{By \eqref{eq:proof:rule3:nancAB}, \eqref{eq:rule3:asp1}, \axPT{}, \axMP{}},
\nonumber \\ &
\vdashgp ( (\intvE{\subst{\bmcB}{\bmx_1}} \psiDP) \land \intvE{\subst{\bmcB}{\bmx_1}} \nanc(\bmx_1,\bmz) ) 
\nonumber \\ &
\label{eq:proof:rule3:747}
\hspace{3.5ex} \rightarrow 
( (f_1 = \bmy|_{\bmz}) \leftrightarrow \intvE{\subst{\bmcB}{\bmx_1}} (f_1 = \bmy|_{\bmz}) )
\\ &
\mbox{By \eqref{eq:proof:rule3:747}, \eqref{eq:proof:rule3:nancB}, \eqref{eq:rule3:PT}, \axPT{}, \axMP{}},
\nonumber \\ &
\label{eq:rule3:754}
\vdashgp \psiPre \rightarrow 
( (f_1 = \bmy|_{\bmz}) \leftrightarrow \intvE{\subst{\bmcB}{\bmx_1}} (f_1 = \bmy|_{\bmz}) )
\\ &
\mbox{By \eqref{eq:rule3:754}, \eqref{eq:rule3:692}, \axEqB{}, \axPT{}, \axMP{}},
\nonumber \\ &
\label{eq:rule3:771}
\vdashgp \psiPre \rightarrow 
( (f_0 = \bmy|_{\bmz}) \leftrightarrow \intvE{\subst{\bmcB}{\bmx_1}} (f_0 = \bmy|_{\bmz}) )
\\[3ex] &
\mbox{By \axDoB, \axPT,\axMP},
\nonumber \\ &
\label{eq:rule3:axDoB}
\vdash_{\dgen{}_0\intvE{\subst{\bmcB}{\bmxB}}} 
\psiDoB \rightarrow
( f_2 = \bmy|_{\bmz,\bmxC=\bmcC} \leftrightarrow 
  \intvE{\subst{\bmcC}{\bmxC}} f_2 = \bmy|_{\bmz} )
\\&
\mbox{By \eqref{eq:rule3:axDoB}, \eqref{eq:rule3:intvE-dsep}, \axPT, \axMP},
\nonumber \\ &
\label{eq:rule3:ax726}
\vdash_{\dgen{}_0\intvE{\subst{\bmcB}{\bmxB}}} 
\psiDP \rightarrow
( f_2 = \bmy|_{\bmz,\bmxC=\bmcC} \leftrightarrow 
  \intvE{\subst{\bmcC}{\bmxC}} f_2 = \bmy|_{\bmz} )
\\&
\mbox{By \eqref{eq:rule3:ax726}, \eqref{eq:rule3:686}, \axEqB, \axPT, \axMP}
\nonumber \\ &
\label{eq:rule3:ax776}
\vdash_{\dgen{}_0\intvE{\subst{\bmcB}{\bmxB}}} 
\psiDP \rightarrow
( f_0 = \bmy|_{\bmz} \leftrightarrow 
  \intvE{\subst{\bmcC}{\bmxC}} f_0 = \bmy|_{\bmz} )
\\&
\mbox{By \eqref{eq:rule3:ax776}, \axDGEI{}, \axDistrE$^{\rightarrow}$, \axDistrE$^{\land}$, \axPT, \axMP},
\nonumber \\ &
\label{eq:rule3:ax}
\vdashgp
\intvE{\subst{\bmcB}{\bmxB}} \psiDP \rightarrow
( (\intvE{\subst{\bmcB}{\bmxB}} f_0 = \bmy|_{\bmz}) \leftrightarrow 
  \intvE{\subst{\bmcB}{\bmxB}} \intvE{\subst{\bmcC}{\bmxC}} f_0 = \bmy|_{\bmz} )
\\&
\mbox{By \eqref{eq:rule3:ax}, \eqref{eq:rule3:PT}, \axPT, \axMP},
\nonumber \\ &
\label{eq:rule3:792}
\vdashgp \psiPre \rightarrow
( (\intvE{\subst{\bmcB}{\bmxB}} f_0 = \bmy|_{\bmz}) \leftrightarrow 
  \intvE{\subst{\bmcB}{\bmxB}} \intvE{\subst{\bmcC}{\bmxC}} f_0 = \bmy|_{\bmz} )
\\&
\mbox{By \eqref{eq:rule3:771}, \eqref{eq:rule3:792}, \axEqB, \axPT, \axMP},
\nonumber \\ &
\label{eq:rule3:817}
\vdashgp \psiPre \rightarrow
( (f_0 = \bmy|_{\bmz}) \leftrightarrow 
  \intvE{\subst{\bmcB}{\bmxB}} \intvE{\subst{\bmcC}{\bmxC}} f_0 = \bmy|_{\bmz} )
\\&
\mbox{By \eqref{eq:rule3:817}, \axDGEI{}, \axPT{}, \axMP{}},
\nonumber \\ &
\label{eq:rule3:824}
\vdashg ( \intvE{\subst{\bmcA}{\bmv}} \psiPre )
\nonumber \\&
\hspace{3ex} \rightarrow
( ( \intvE{\subst{\bmcA}{\bmv}} f_0 = \bmy|_{\bmz}) \leftrightarrow 
  \intvE{\subst{\bmcA}{\bmv}} \intvE{\subst{\bmcB}{\bmxB}} \intvE{\subst{\bmcC}{\bmxC}} f_0 = \bmy|_{\bmz} )
\\&
\mbox{By \eqref{eq:rule3:824}, \axSimulE{}, \axMP{}},
\nonumber \\ &
\label{eq:rule3:833}
\vdashg ( \intvE{\subst{\bmcA}{\bmv}} \psiPre )
\nonumber \\&
\hspace{3ex} \rightarrow
( (\intvE{\subst{\bmcA}{\bmv}} f_0 = \bmy|_{\bmz}) \leftrightarrow 
  \intvE{\subst{\bmcA}{\bmv}, \subst{\bmcB}{\bmxB}, \subst{\bmcC}{\bmxC}} f_0 = \bmy|_{\bmz} )
\\&
\mbox{By \eqref{eq:rule3:833}, \axEqB{}, \axPT{}, \axMP{}},
\nonumber \\ &
\vdashg ( \intvE{\subst{\bmcA}{\bmv}} \psiPre )
\nonumber \\&
\nonumber \hspace{3ex} \rightarrow
( (\intvE{\subst{\bmcA}{\bmv}} \phi) \leftrightarrow 
  \intvE{\subst{\bmcA}{\bmv}, \subst{\bmcB}{\bmxB}, \subst{\bmcC}{\bmxC}} \phi )
\end{align}
Therefore, Claim (3) follows.
\myqed

\end{enumerate}
\end{proof}

}%
\end{document}